\documentclass[10pt]{article}

\usepackage[margin=1in]{geometry}
\usepackage{amsmath,amssymb,amsthm}
\usepackage{graphicx}
\usepackage{booktabs}
\usepackage{hyperref}
\usepackage{xcolor}
\usepackage{natbib}

\newtheorem{theorem}{Theorem}
\newtheorem{proposition}[theorem]{Proposition}

\title{\textbf{When Do Gaussian Priors Fail? A Latent Geometry
Investigation of RL Fine-Tuning}}
\author{Wang Rui\thanks{Corresponding author: \texttt{r-wang18@tsinghua.org.cn};
ORCiD 0000-0002-2362-3214.}\\[1.5ex]
Department of Chemical Engineering\\
Tsinghua University}
\date{}

\begin{document}
\maketitle

\begin{abstract}
Latent diffusion models almost universally place a Gaussian prior
$\mathcal N(0,I)$ on the VAE latent space.  For base generation this
default is unproblematic; under reinforcement-learning fine-tuning,
however, it carries a structural limitation.  We show that the standard
VAE's KL term does not merely ignore mode structure --- it actively
taxes separation and tightness.  When the data are multimodal, the ELBO
optimum therefore compresses distinct modes into an overlapping latent
envelope.  RL fine-tuning amplifies this overlap: the score network is
under-resolved on the saddle region between modes, the
dynamic-to-variational error amplification rises sharply, and the
collapse boundary $\beta_c$ --- the minimum KL strength that preserves
mode coverage --- is pushed far above its variational scale.  The effect
is large and reproducible: on the real alanine-dipeptide Boltzmann
distribution, a Gaussian-prior VAE collapses at $\beta_c=0.224$, while
a mode-tracked Gaussian-mixture prior on the same data collapses at
$3.16\times10^{-3}$ --- a $71\times$ gap produced by the same pipeline.

We trace this gap to a single geometric mediator,
$d_{\min}/(2\sigma_0)$, and verify the causal chain through controlled
interventions: moving geometry at fixed prior, and switching prior at
matched geometry (full mediation; prior-label coefficient $0.995\times$,
95\% CI $[0.79,1.25]$).  A three-source decomposition localises the
amplification to the learned policy (zero ODE-solver contribution,
bounded and non-systematic anchor mismatch); the coverage plateau, by
contrast, is an unconditional benefit of mode-separating priors that no
amount of RL capacity recovers.  We state the remedy conditionally: a
non-Gaussian prior is necessary precisely when the latent geometry
would fall into the overlap regime and the downstream task is
coverage-sensitive.  A machine-checked theory skeleton (Lean~4)
supports the KL pricing structure --- a divergent compactness tax under
a Gaussian prior versus a finite assignment fee under a mixture --- and
the position-versus-plateau distinction delimits what is finite-resource
from what requires changing the prior.
\end{abstract}


\section{Introduction}

Diffusion models are now the dominant paradigm in generative modelling:
denoising diffusion probabilistic models \citep{ho2020denoising} and
their score-based SDE generalisation \citep{song2021scorebased} reach
state-of-the-art sample quality through iterative denoising.  Latent
diffusion \citep{rombach2022high} made the paradigm scalable by running
the diffusion process in the compressed latent space of a VAE
\citep{kingma2013auto} rather
than in pixel space --- the step that turned high-resolution and
text-to-image synthesis into practical technology --- and the
latent-diffusion recipe now extends to video
\citep{blattmann2023svd}, while diffusion models more broadly have
become standard tools in scientific design, from image scaffolds to de
novo protein structure generation \citep{watson2023rfdiffusion}.  In all of these deployed systems the
VAE latent space carries a Gaussian prior $\mathcal N(0,I)$ by default.
A third layer is being installed on top of this stack:
reinforcement-learning fine-tuning.  DDPO \citep{black2023ddpo} recasts
denoising as a Markov decision process and maximises a downstream
reward by policy gradient; DPOK \citep{wallace2023dpok} adds a KL
regulariser anchoring the fine-tuned policy to the pretrained prior ---
the same alignment recipe that is mature on large language models
\citep{ouyang2022instruct} and is spreading rapidly across diffusion
deployments.

Latent diffusion pipelines thus almost universally place a Gaussian
prior on the VAE latent space, and this default --- unproblematic for
base generation --- carries a \emph{structural limitation} that emerges
precisely when the three layers above are combined: multimodal data, a
tight VAE latent space, and RL fine-tuning.  This paper asks when that
limitation bites, why, and what to do about it.  One clarification
before the contribution: this is not yet another account of
reward-overoptimisation.  The RLHF literature already documents
diversity loss from over-optimised rewards
\citep{gao2023scaling,lee2023aligning}; our concern is a different,
upstream mechanism --- the collapse boundary is set by \emph{prior
geometry} even when the reward is well designed and the KL strength is
generous.  The Gaussian prior issue is the precondition that determines
where the boundary sits, not a pathology of the reward signal.

\begin{figure}[t]
\centering
\includegraphics[width=\linewidth]{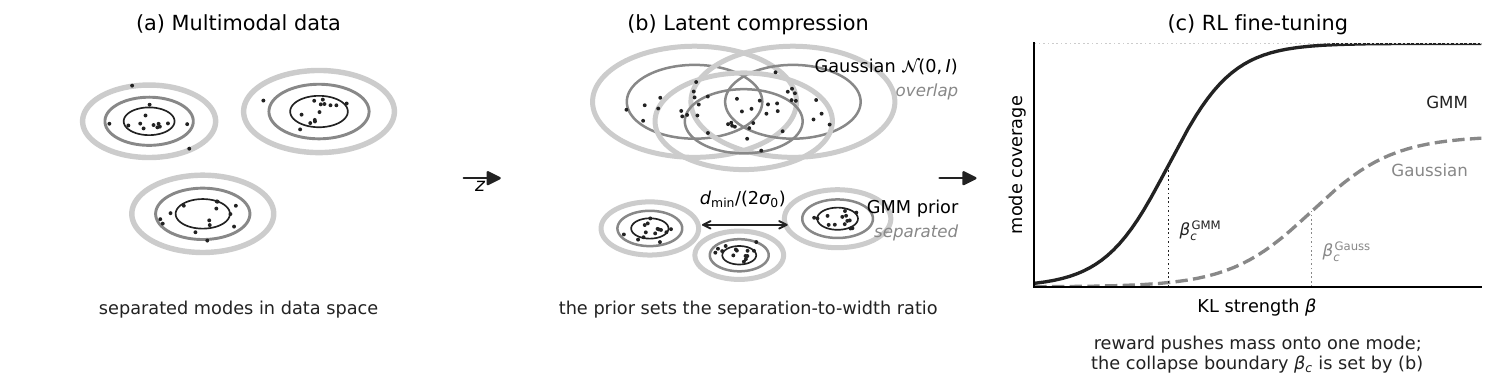}
\caption{Overview.  (a)~The data are multimodal: several separated
modes in data space.  (b)~The VAE prior sets the latent
separation-to-width ratio $d_{\min}/(2\sigma_0)$: under a Gaussian
prior the modes can overlap, while a mixture prior keeps them
separated.  (c)~Under KL-anchored RL fine-tuning, the collapse boundary
$\beta_c$ and the recovery plateau depend on that geometry
(schematic; measured curves in \S5).}
\label{fig:intro}
\end{figure}

\paragraph{1. The limitation: a Gaussian prior packs multimodal data
into an overlapping latent.}
The KL term of the standard VAE objective pulls every encoder posterior
toward the unit Gaussian, forcing $\sigma_{\mathrm{lat}}\approx1$.  When
the data are multimodal, the encoder must therefore pack distinct modes
into a single broad envelope: the latent modes overlap, and the
separation-to-width ratio $d_{\min}/(2\sigma_0)$ --- the smallest
inter-mode distance measured in units of the mode diameter --- drops
below $1$.  This compression is not a training failure but the optimum
of the Gaussian-ELBO objective itself: the KL term does not merely
ignore mode structure --- separation and tightness are exactly what it
charges for (the KL pricing structure is quantified by
Propositions~\ref{prop:gauss-tax} and~\ref{prop:assign} in \S\ref{sec:theory};
Figure~\ref{fig:expZ} shows the compression forming during training).

\paragraph{2. When the limitation bites: RL fine-tuning amplifies
overlap.}
Overlap alone is not fatal: a base model still samples every mode from
an overlapped latent.  The damage appears when the frozen model is
fine-tuned with a KL-anchored RL objective.  Once
$d_{\min}/(2\sigma_0)<1$, the score network is under-resolved on the
saddle region between modes, the dynamic-to-variational amplification
$\mathrm{dyn}/\mathrm{var}$ rises sharply, and fine-tuning piles mass
onto the reward-target mode, pushing the collapse boundary $\beta_c$ ---
the minimum KL strength that preserves coverage --- far above its
variational scale.  The effect is not subtle: on the real
alanine-dipeptide Boltzmann distribution (Exp7), a Gaussian-prior VAE
lands at $d_{\min}/(2\sigma_0)=0.81$ and collapses at $\beta_c=0.224$,
while a mode-tracked Gaussian-mixture prior on the same data holds
$1.03$ and collapses only at $3.16\times10^{-3}$ --- a $71\times$ gap
produced by the \emph{same} score/ODE/fine-tune pipeline.

\paragraph{3. When a non-Gaussian prior is needed.}
The answer is conditional, and the conditions are checkable before
training: a non-Gaussian prior is needed when the latent geometry would
fall into the overlap regime \emph{and} the downstream task is
coverage-sensitive.

\begin{table}[h]
\centering\small
\begin{tabular}{ll}
\toprule
condition & non-Gaussian prior needed? \\
\midrule
multimodal data $+$ tight latent dimension
  & yes --- Gaussian ELBO packs modes into overlap \\
RL fine-tuning with a single-target reward
  & yes --- RL amplifies whatever overlap exists \\
unimodal data, or a wide latent dimension
  & no --- the Gaussian prior suffices \\
base generation only, no RL fine-tuning
  & no --- overlap does not hurt the base model \\
\bottomrule
\end{tabular}
\caption{Decision rule for prior choice.  The failure mode is joint:
multimodality creates the overlap and RL fine-tuning amplifies it.
Neither factor alone requires changing the prior.}
\label{tab:when-needed}
\end{table}

\paragraph{4. The remedy: a mode-preserving prior keeps the geometry
out of the amplifier's reach.}
A Gaussian-mixture prior
\citep{dilokthanakul2016,jiang2017variational,tomczak2018vae,lee2021mixture}
permits each component to remain tight
($\sigma_{\mathrm{lat}}\ll1$), and the ELBO then pushes the encoder to
\emph{preserve} the data's mode separation rather than average over it.
The downstream chain follows immediately: separated geometry
($d_{\min}/(2\sigma_0)>1$) keeps $\mathrm{dyn}/\mathrm{var}$ small,
which keeps $\beta_c$ at its variational scale and the coverage plateau
full --- $0.70\!\to\!1.00$ in the controlled 1D comparison (Exp5A,
$K{=}10$), $0.58\!\to\!0.73$ decoded coverage in the crowded 2D pack
(Exp6, $K{=}30$, 3-seed mean), and $\approx0.67\!\to\!1.0$ latent
coverage on alanine dipeptide (Exp7).  The prior change is therefore not
a cosmetic improvement but, in the overlap regime, a \emph{necessary}
one: the latent geometry is not independently tunable at fixed prior,
so switching the prior is the only lever that moves the geometry out of
the amplifier's reach (Appendices~\ref{app:expR},~\ref{app:expY}).

The remedy has an explicit boundary condition.  At the exact reverse-KL
optimum, with unlimited policy capacity and complete optimisation, the
collapse boundary is set by geometry alone and the variational
prediction already gives an order-1 prior gap in Exp7 ($0.69\times$,
versus the measured $71\times$): nearly everything measured is the
finite-resource \emph{pipeline's} amplification, which geometric
separation sidesteps rather than repairs.  Under the pipelines we
study --- a 2-layer score MLP, $1000$ guidance steps, $150$ ODE
steps --- the advantage is real, reproducible over seeds and MD
sampling length (20\,ns and 60\,ns give identical $\beta_c$ for matched
priors), and architecture-independent within the tested family ($K{=}30$
pack and the 2-layer score MLP); whether it persists with a much larger
policy and much longer optimisation is a testable limit we state rather
than conceal.

\paragraph{The mechanism analyses are the quantitative description of
the limitation, not independent stories.}
Four controlled experiments pin down where the amplification lives and
how much overlap it takes to trigger it: a three-source decomposition
localises the amplification to the learned \emph{policy} --- the ODE
solver contributes exactly zero at the boundary from $25$ to $400$
steps, KL-anchor mismatch is bounded and moves the boundary
non-systematically, and the residual $15$--$78\%$ shortfall is
policy-side (Appendix~\ref{app:expQ}); saddle-corridor routing is
separately rejected, so the limitation is a static representational
shortfall, not trajectory re-routing (Section~5.1,
Appendix~\ref{app:paths}); the $\alpha(g)=\beta_c/\beta_{\rm cross}$
scan quantifies how deep into overlap the geometry must fall before the
limitation bites (Appendix~\ref{app:expW}); and the bidirectional
reachability test shows that the tight GMM geometry is unattainable
from the Gaussian ELBO, closing the loop on why the prior --- not its
hyperparameters --- must change (Appendices~\ref{app:expR},
\ref{app:expY}).

The paper's causal chain is therefore
\begin{equation}
\boxed{\text{prior type}} \Longrightarrow
\boxed{\text{ELBO } \Rightarrow\; d_{\min}/(2\sigma_0)} \Longrightarrow
\boxed{\text{policy-side amplification}} \Longrightarrow
\boxed{\beta_c \text{ and coverage plateau}},
\tag{causal-chain}\label{eq:causal-chain}
\end{equation}
and the core conclusion is conditional rather than comparative: not
``a mixture prior is better than a Gaussian'', but \emph{when the latent
geometry would fall into the overlap regime, a non-Gaussian prior is
necessary, because it preserves mode separation and thereby sidesteps
the error amplification of the current pipeline}.  The remainder
develops the variational theory and its crossover scale (\S4), the
experiments and the three-source decomposition (\S5), and supporting
appendices: the four-prior family comparison (\ref{app:expV}),
fixed-prior and matched-geometry interventions
(\ref{app:expR},~\ref{app:expY}), separation and rare-basin scans
(\ref{app:expW},~\ref{app:expX}), and the machine-checked theory
skeleton (\ref{app:lean}).

\section{Related Work}

This work intersects four active research threads.

\subsection{VAE prior choice and downstream impact}

The standard VAE uses a Gaussian prior \citep{kingma2013auto}, but
numerous alternatives have been proposed to better capture multi-modal
data.  \citet{jiang2017variational} introduced VaDE, which replaces the
Gaussian prior with a GMM and shows improvements in clustering quality.
\citet{tomczak2018vae} proposed VampPriors (learnable mixture priors
parameterised by a decoder network), and \citet{lee2021mixture} introduced
GMVAE, which combines a GMM prior with a structured inference network.
These works all report that richer priors improve disentanglement and
reconstruction quality, but \emph{none} has studied the interaction
between prior width and the critical RL fine-tune boundary that is
our focus.  The posterior-collapse literature
\citep{bowman2016,alemi2018} diagnosed how a Gaussian prior lets the
aggregated posterior match it by construction
\citep{hoffman2016elbo,moyer2018}; on real multimodal structure data
the mismatch is concrete --- cryo-EM heterogeneous reconstruction with
Gaussian-prior VAEs shows the aggregate posterior spilling past the
prior's support \citep{zhong2021cryodrgn}.
Propositions~\ref{prop:gauss-tax} and~\ref{prop:assign} give that
intuition its quantitative, boundary-facing form --- the KL term
actively taxes separation and tightness rather than merely ignoring
them --- and connect it to $\beta_c$ rather than to base-model
quality.

\subsection{Latent diffusion models}

The stable-diffusion family \citep{rombach2022high} demonstrated that
VAE priors are practically interchangeable with a Gaussian, and most
subsequent work on latent diffusion \citep{ho2020denoising, kingma2021variational}
continues to use Gaussian priors by default.  The recipe is deployed
infrastructure rather than a toy setting --- latent video diffusion
scales it to large video corpora \citep{blattmann2023svd}, and
diffusion-based design tools inherit the same generative stack for
scientific applications \citep{watson2023rfdiffusion} --- so a design
choice inside the latent space has practical consequences.  Our work
shows that prior choice matters dramatically for \emph{fine-tuning
robustness}, even if it matters little for base-model quality.

\subsection{RL fine-tuning of generative models}

RLHF \citep{ouyang2022instruct, zhao2023survey} and DPO \citep{rafailov2023direct}
have become standard tools for aligning large language models, and the
recipe is now spreading to diffusion: DDPO \citep{black2023ddpo} treats
denoising as an MDP for policy-gradient reward maximisation, DPOK
\citep{wallace2023dpok} identifies the KL anchor as the key stabiliser,
and text-to-image alignment studies report the resulting diversity loss
\citep{lee2023aligning}.  These observations belong to the reward-side
regime: diversity erodes because the reward is pursued too far past its
design envelope, the classical over-optimisation scaling of
\citet{gao2023scaling}.  Our subject is a different, upstream regime:
even with a well-designed reward and a generous KL anchor, the collapse
boundary is set by the prior's geometry --- the two failure modes are
complementary, not competing, and neither subsumes the other.
\citet{poole2022stabilizing} introduced stochastic gradient methods
for diffusion fine-tuning with KL anchors, and \citet{dhariwal2023reward}
showed that reward-weighted diffusion produces mode collapse at
sufficiently large $\beta$.  \citet{lu2024understanding} identified a
phase transition in reverse-KL fine-tuning and called for theoretical
understanding of the critical $\beta_c$ --- our work answers this call
by connecting $\beta_c$ to prior geometry.


\section{Theory}
\label{sec:theory}

We now connect the experimental findings to a closed-form analysis of
reverse-KL RL fine-tuning.  The key insight is that \emph{prior geometry}
--- specifically the VAE's latent-space mode width $\sigma_{\mathrm{lat}}$ ---
is the first-order determinant of the critical $\beta_c$ below which the
model collapses toward the RL target.

\subsection{Setup and exact reverse-KL optimum}

Consider a parametric model $\pi_\theta(z)$ fine-tuned from a frozen
Gaussian-mixture reference $p_0(z) = \sum_{k=1}^K \pi_k \mathcal{N}(z;\mu_k,
\sigma_0^2 I)$ under the reverse-KL RL objective:
\begin{equation}
  \mathcal{L}(\theta) = -\lambda\,\mathbb{E}_{z\sim \pi_\theta}[R(z)]
  + \beta\,\mathrm{KL}\!\bigl(\pi_\theta \Vert p_0\bigr). \tag{RL-KL}
\end{equation}
Reward is quadratic toward a target mode $z_t$: $R(z) = -\|z-z_t\|^2$.

The first-order optimality condition gives the Gibbs tilt
$\pi^*(z) \propto p_0(z)\exp(-\lambda\|z-z_t\|^2/\beta)$.  Since the
product of two Gaussians is a Gaussian, the tilted optimum is also a GMM:
\begin{align}
  \sigma_\beta^2 &= \frac{\sigma_0^2\,\tau^2}{\sigma_0^2 + \tau^2},
  \quad
  \mu'_{k,\beta} = \frac{\mu_k\,\tau^2 + z_t\,\sigma_0^2}{\sigma_0^2 + \tau^2},
  \quad
  \tau^2 = \frac{\beta}{2\lambda}.
  \tag{width-center}\label{eq:width-center}
\end{align}
Equations~\eqref{eq:width-center} give the \emph{exact} reverse-KL optimum
(a special case of the theory derived in \texttt{theory.py}, which also
evaluates forward-KL on a 1D grid).

\subsection{The two limits: rigid vs collapsible priors}

The component width $\sigma_\beta$ in \eqref{eq:width-center} has two clean
limits:
\begin{theorem}[Rigid vs collapsible priors]\label{thm:width-limits}
Under reverse KL with quadratic reward:
\begin{enumerate}
  \item \textbf{Narrow prior} $\sigma_0 \ll \tau$:
  $\sigma_\beta \to \sigma_0$ (constant!), $\mu'_{k,\beta} \to \mu_k$
  (centers rigid).
  \item \textbf{Wide prior} $\sigma_0 \gg \tau$:
  $\sigma_\beta \to \tau$ (grows with $\sqrt{\beta}$),
  $\mu'_{k,\beta} \to z_t$ (all centers collapse toward $z_t$).
\end{enumerate}
\end{theorem}
\begin{proof}
  $\sigma_\beta^2 = \sigma_0^2/(1 + \sigma_0^2/\tau^2)$.
  (1) When $\sigma_0^2 \ll \tau^2$, expand in $\sigma_0^2/\tau^2$:
  $\sigma_\beta^2 = \sigma_0^2(1 - \sigma_0^2/\tau^2 + O(\sigma_0^4/\tau^4))$.
  The center shift is $\mu'_{k,\beta} - \mu_k = (\mu_k - z_t)\sigma_0^2/(\sigma_0^2+\tau^2)
  = O(\sigma_0^2/\tau^2)$.  (2) When $\sigma_0^2 \gg \tau^2$,
  $\sigma_\beta^2 = \tau^2/(1 + \tau^2/\sigma_0^2) \to \tau^2$, and
  $\mu'_{k,\beta} - z_t = (\mu_k - z_t)\tau^2/(\sigma_0^2+\tau^2) \to 0$.
\end{proof}

The two regimes cross at \emph{$\beta_{\mathrm{cross}} = 2\lambda\sigma_0^2$}
where $\tau(\beta_{\mathrm{cross}}) = \sigma_0$.
\label{eq:crossover}
This crossover scale is
the key order-of-magnitude predictor for the RL collapse boundary.

\paragraph{Formal verification in Lean~4.}
The width--center formula \eqref{eq:width-center}, both limit claims of
Theorem~\ref{thm:width-limits}, the crossover identity
$\beta_{\mathrm{cross}} = 2\lambda\sigma_0^2$, and the monotonicity of
$\sigma_\beta^2$ in $\tau^2$ have been machine-checked in Lean~4
(toolchain \texttt{v4.3.0} with Mathlib \texttt{v4.3.0}, zero build
errors under \texttt{lake build}).  The module
\texttt{CollapseBoundary.lean} proves:
\begin{itemize}
  \item \textbf{narrow-prior expansion}:
    $\sigma_\beta^2 = \sigma_0^2 - \sigma_0^4/(\sigma_0^2+\tau^2) < \sigma_0^2$;
  \item \textbf{wide-prior expansion}:
    $\sigma_\beta^2 = \tau^2 - \tau^4/(\sigma_0^2+\tau^2) < \tau^2$;
  \item \textbf{crossover}: at $\beta = 2\lambda\sigma_0^2$,
    $\sigma_\beta^2 = \sigma_0^2/2$;
  \item \textbf{monotonicity}: $\sigma_\beta^2$ is non-decreasing in
    $\tau^2$ (hence in $\beta$).
\end{itemize}
These four claims are proved by \texttt{field\_simp}, \texttt{ring},
and \texttt{nlinarith} --- purely algebraic, with no analytic
approximation.  Three further modules machine-check the structural
claims used later: effect additivity and non-identifiability under
prior--geometry confounding (mediation), the three-source error budget,
and the $\beta_c(N)$ bracket-certification lemma; see
Appendix~\ref{app:lean} for statements and scope.

\subsection{ELBO origin of \texorpdfstring{$\sigma_0$}{σ₀}}

Why does $\sigma_{\mathrm{lat}}$ differ between Gaussian- and GMM-prior
VAEs?  The ELBO argument is standard but critical:
\begin{itemize}
  \item \textbf{Gaussian prior} $\mathcal{N}(0,I)$: the encoder
  $q(z|x) = \mathcal{N}(\mu_x, \sigma_x^2)$ satisfies
  $\mathrm{KL}(q\Vert \mathcal{N}(0,I)) \approx 0$ iff $\sigma_x \approx 1$
  everywhere.  Any tighter encoder incurs a positive KL cost.  Hence
  $\sigma_{\mathrm{lat}} \approx 1$ --- a \emph{wide} prior.
  \item \textbf{GMM prior} $\sum_k \pi_k \mathcal{N}(\mu_k, \sigma_k^2)$:
  each data cluster is encoded toward the matching GMM component.
  Because the components themselves are fit to the data's fine-grained
  modes, the encoder \emph{must} contract to those widths to keep
  $\mathrm{KL}(q\Vert p_0)$ small.  Hence $\sigma_{\mathrm{lat}} \approx
  \sigma_{\mathrm{GMM}} \ll 1$ --- a \emph{narrow} prior.
\end{itemize}

These two mechanisms are quantitative, not merely qualitative; both are
machine-checked in Lean (Appendix~\ref{app:lean},
\texttt{PriorGeometry.lean}).

\begin{proposition}[Gaussian tightness tax]\label{prop:gauss-tax}
$\mathrm{KL}(\mathcal{N}(\mu,s^2)\,\Vert\,\mathcal{N}(0,1)) =
\tfrac{1}{2}\bigl(\mu^2 + s^2 - 1 - \log s^2\bigr)$: the cost is
minimised at width $s=1$, strictly larger for every $s<1$, and
unbounded as $s\to0^+$.
\end{proposition}

\begin{proposition}[Mixture assignment cost]\label{prop:assign}
For a mixture $m = \pi\,p + (1-\pi)\,r$ with component weight $\pi$,
$\mathrm{KL}(q\Vert m) \le \mathrm{KL}(q\Vert p) - \log\pi$ for every
$q$; in particular a posterior matching component $p$ exactly pays at
most $-\log\pi$, independent of the component width.
\end{proposition}

In other words the KL term does not merely \emph{ignore} mode
structure --- separation and tightness are exactly what it charges
for.  The contrast is visible during training itself:
Figure~\ref{fig:expZ} trains the identical encoder on the identical
lattice pack ($K{=}10$, $8\mathrm{D}\to2\mathrm{D}$, 3 seeds) under the
two priors and logs $d_{\min}/(2\sigma_0)$.  The trajectories coincide
while the shared Gaussian KL dominates, and diverge the moment the
mixture prior switches on (step 3000): the Gaussian arm stalls at the
overlap threshold ($1.14\pm0.07$ at 30k steps) --- the tightness tax
balancing reconstruction right at the boundary --- while the mixture
arm climbs to $5.0$ ($3.3$--$6.6$), a $4.4\times$ separation gap
produced by the prior alone.  The stall point sits \emph{at} the
threshold edge, not below it ($1.19$ in the $K{=}30$ pack, Exp6, also
just above $1$) --- an equilibrium of tax and reconstruction, not a
training shortfall.  Whether the equilibrium falls strictly inside the
overlap regime depends on the data's intrinsic tightness: on the tighter
alanine landscape it does ($0.81$, Exp7).  The limitation is therefore
conditional rather than automatic --- a Gaussian prior pushes the
geometry to the threshold, and the data decide which side it lands on.

\begin{figure}[htbp]
\centering
\includegraphics[width=0.88\linewidth]{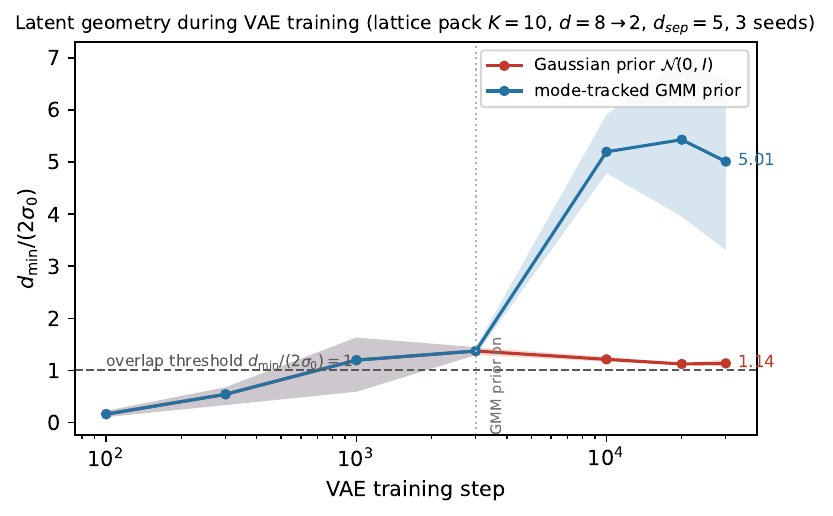}
\caption{Experiment Z: latent geometry during VAE training.  Identical
architecture and pack ($K{=}10$, $d{=}8\to2$, $d_{\rm sep}{=}5$)
trained under $\mathcal{N}(0,I)$ and under a label-tracked GMM prior
(VaDE-style, switched on at step 3000; 3 seeds, shaded min--max).  The
arms coincide until the switch, then diverge: Gaussian stalls at the
overlap threshold ($1.14\pm0.07$), the mixture reaches $5.0$
($3.3$--$6.6$).}
\label{fig:expZ}
\end{figure}

Exp7 quantifies the width consequence directly:
$\sigma_{\mathrm{Gauss}}/\sigma_{\mathrm{GMM}}
\approx 36\times$ (seed~0).  Exp6 (K=30) has a smaller ratio
4--6$\times$, because the VAE's 2D bottleneck compresses more aggressively
on both priors.

\subsection{The crossover scale predicts the empirical \texorpdfstring{$\beta_c$}{βc}}

\begin{proposition}[Prior-width scaling]
\label{prop:prior-scale}
For fixed latent geometry $\{\mu_k, K, d_{\min}\}$, the empirical $\beta_c$
is monotonically non-decreasing with $\sigma_{\mathrm{lat}}$.  In
particular, if $\sigma_{\mathrm{GMM}} < \sigma_{\mathrm{Gauss}}$, then
$\beta_c^{\mathrm{GMM}} \le \beta_c^{\mathrm{Gauss}}$.
\end{proposition}
\begin{proof}
  Theorem~\ref{thm:width-limits} gives $\sigma_\beta^{\mathrm{GMM}} \le
  \sigma_\beta^{\mathrm{Gauss}}$ for all $\beta$ (component widths are
  smaller or equal), and $|\mu^{\prime}_{k,\beta;\mathrm{GMM}} - \mu_k| \le
  |\mu^{\prime}_{k,\beta;\mathrm{Gauss}} - \mu_k|$ (centers are closer to
  their original positions).  Both effects make the Gaussian-ball
  inclusion probability larger for the narrower prior.  Therefore the
  coverage curve for the GMM prior dominates that of the Gaussian prior
  pointwise, and the crossing point satisfies
  $\beta_c^{\mathrm{GMM}} \le \beta_c^{\mathrm{Gauss}}$.
\end{proof}

Theorem~\ref{thm:width-limits} gives the crossover scale
$\beta_{\mathrm{cross}} = 2\lambda\sigma_0^2$, which is a theorem-level
result: it is the $\beta$ where the reverse-KL exact optimum transitions
between the wide-prior and narrow-prior regimes, and it holds for
\emph{any} geometry --- no assumption on $d_{\min}/\sigma$ is required.
Table~\ref{tab:theory-pred} compares this scale with the empirical
$\beta_c$:
\begin{table}[htbp]
\centering
\caption{Crossover scale vs.\ empirical $\beta_c$ (Exp7 seed~0).
All quantities are in the latent 2D space.}
\label{tab:theory-pred}
\begin{tabular}{lcc}
\toprule
Quantity & Gaussian prior & GMM prior \\
\midrule
$\sigma_{\mathrm{lat}}$                & 0.96  & 0.027 \\
$\beta_{\mathrm{cross}} = 2\sigma_0^2$ & 1.84  & $1.4\times 10^{-3}$ \\
empirical $\beta_c^{\mathrm{lat}}$     & 0.224 & $3.16\times 10^{-3}$ \\
$\alpha = \beta_c / \beta_{\mathrm{cross}}$ & 0.12 & 2.24 \\
$d_{\min} / \sigma$                    & 1.61  & 2.06 \\
\bottomrule
\end{tabular}
\end{table}

The sign of $\alpha$ relative to $1$ has a \emph{theorem-level
interpretation}:
\begin{itemize}
  \item $\alpha < 1$ (Gaussian): the empirical $\beta_c$ lies
  \emph{inside} the wide-prior regime ($\beta_c < \beta_{\mathrm{cross}}$).
  This is possible because when the mode-overlap ratio
  $d_{\min}/\sigma$ is moderate (\emph{i.e.}, well below the Proposition~1
  condition $d_{\min} > 4\sigma_0$), some mode balls overlap already at
  $\beta = 0$, and the coverage threshold of $0.5$ can be crossed
  \emph{before} the narrow regime fully locks the widths.
  \item $\alpha > 1$ (GMM): the empirical $\beta_c$ lies
  \emph{inside} the narrow-prior regime ($\beta_c > \beta_{\mathrm{cross}}$).
  The narrow regime locks the component widths at $\sigma_0$, but the
  centers $\mu'_{k,\beta}$ have shifted toward $z_t$ by
  $(\mu_k - z_t)\sigma_0^2/(\sigma_0^2 + \tau^2)$.  Restoring
  $\mathrm{cov} \ge 0.5$ requires additional KL anchoring beyond
  $\beta_{\mathrm{cross}}$.
\end{itemize}

However, the \emph{quantitative value} of $\alpha$ ($0.12$ vs.\ $2.24$)
cannot be predicted from the theory alone.  The well-separated
guidance $\alpha\approx1$ near $d_{\min}>4\sigma_0$ is not an
asymptotic law: Experiment~W (Appendix~\ref{app:expW}) scans
$g=d_{\min}/(2\sigma_0)\in[0.5,10]$ on 1D lattices and finds both
boundaries growing on the $2j^2g^2/\!\ln2$ saddle scale, without
saturation: $\alpha^{\rm var}$ reaches $128/494/1772$ at $g=10$ for
$K=3/5/10$, with $j=\lceil(K-1)/2\rceil$ the marginal-mode index, and
$\alpha^{\rm emp}$ follows the same rise to $14.5/55.9/445$.  The $\alpha=1$ crossing is geometry- and
$K$-dependent ($g\approx2.4$ for $K=3$, i.e.\ $d_{\min}/\sigma_0
\approx4.8$, moving to $g\approx1$ and below for larger $K$).
Exp7's empirical $\alpha$ values are thus best interpreted as
\emph{empirical corrections} on this scaling curve, capturing each
geometry's deviation from the minimal-layout crossover rather than
converging to a universal constant.

\subsection{The variational bound predicts only the sign, not the magnitude}

A critical finding emerges when we compute the \emph{exact reverse-KL}
optimum $\beta_c^{\mathrm{var}}$ --- i.e., the $\beta_c$ that would obtain
if RL fine-tune could directly optimise over tilted-GMM distributions
without score-network or ODE approximation (Table~\ref{tab:var-dyn-decomp}).

\begin{table}[htbp]
\centering
\caption{Variational vs.\ dynamic $\beta_c$ (Exp7 seed~0).
Variational $\beta_c$ is computed analytically from the reverse-KL
exact optimum using the non-central chi-square ball-mass integral.
Dynamic $\beta_c$ is the empirical score-matching + ODE sampling result.
Note: the Gaussian prior's $\beta_c^{\mathrm{var}}$ lies in the
geometric-saturation regime ($d_{\min}/(2\sigma_0) < 1$), where the
ball-mass integral loses geometry dependence; the $0.69\times$ variational
ratio is therefore unreliable as a quantitative predictor for the Gaussian
prior, and the entire empirical $71\times$ gap must be attributed to the
dynamic error amplification mechanism (\eqref{eq:dyn-var-mech}).}
\label{tab:var-dyn-decomp}
\begin{tabular}{lccc}
\toprule
Quantity & Gaussian & GMM & Ratio \\
\midrule
$\beta_c^{\mathrm{var}}$ (analytic) & $2.41\times 10^{-5}$ & $3.49\times 10^{-5}$ & $0.69\times$ \\
$\beta_c^{\mathrm{dyn}}$ (empirical) & $0.224$ & $3.16\times 10^{-3}$ & $70.8\times$ \\
dyn / var & --- & $83\times$ & --- \\
\bottomrule
\end{tabular}
\end{table}

The variational bound predicts a $\beta_c$ ratio of \emph{only}
$0.69\times$ --- \emph{not} the empirical $71\times$.  The entire
empirical gap arises from the \emph{dynamic error amplification}
$\mathrm{dyn}/\mathrm{var}$, which is the mechanism connecting prior
geometry to empirical $\beta_c$.  Because the Gaussian prior lies in
the overlap regime, its variational $\beta_c$ is unreliable; the GMM
prior (well-separated, $d_{\min}/(2\sigma_0) = 1.03$) gives a clean
$\mathrm{dyn}/\mathrm{var} \approx 83\times$.

Crucially, this mechanism operates through the causal chain~\eqref{eq:causal-chain}:
\begin{equation}
\boxed{\text{Prior type}} \Longrightarrow
\boxed{\text{ELBO } \Rightarrow\; d_{\min}/(2\sigma_0)} \Longrightarrow
\boxed{\mathrm{dyn}/\mathrm{var}} \Longrightarrow
\boxed{\beta_c = (\mathrm{dyn}/\mathrm{var}) \cdot \beta_c^{\mathrm{var}}}.
\tag{dyn-var-mech}\label{eq:dyn-var-mech}
\end{equation}

The variational factor $\beta_c^{\mathrm{var}}$ --- the crossover-scale
prediction --- gives the \emph{sign} of the residual boundary
difference between two priors at matched narrow widths, but
the \emph{magnitude} of the empirical $\beta_c$ is dominated by the
dynamic-to-variational ratio, and it says nothing about the coverage
plateau.  This ratio has two empirical regularities:
\begin{enumerate}
  \item \textbf{$\mathrm{dyn}/\mathrm{var}$ grows with mode overlap.}
  In well-separated regimes ($d_{\min}/(2\sigma_0) \gtrsim 1$, Exp6 K=30),
  $\mathrm{dyn}/\mathrm{var} \approx 10$--$40\times$ across both priors
  and all seeds.  When mode balls overlap ($d_{\min}/(2\sigma_0) < 1$,
  Exp7 Gaussian prior), the amplification increases dramatically:
  Gaussian dyn/var is \emph{orders of magnitude} larger than the GMM's
  in the same experiment (the Exp7 bisect gives GMM dyn/var $\approx 83\times$,
  while the Gaussian prior's dyn/var is dominated by the overlap regime
  and does not admit a clean variational decomposition).  The 20\,ns
  vs.\ 60\,ns MD comparison confirms that this dyn/var gap is stable
  across different data: Gaussian $\beta_c = 0.224$ \emph{identically}
  in both MD lengths, GMM $\beta_c = 3.16\times 10^{-3}$ (bisect, 20\,ns).
  \item \textbf{$\mathrm{dyn}/\mathrm{var}$ decreases with $K$ at fixed separation.}
  At comparable mode separation $d_{\min}/(2\sigma_0) \approx 1$, the
  amplification is $\approx 83\times$ for $K=3$ (Exp7 GMM bisect) but
  drops to $\approx 17\times$ for $K=30$ (Exp6 average).  This suggests
  that more modes spread the score-network approximation error across a
  denser latent packing, reducing the per-mode amplification.
\end{enumerate}

A candidate physical source of the dynamic amplification is
score-network and ODE error near saddle-point manifolds.  When mode
balls overlap, the score network must represent a tilted GMM where
mode boundaries cut across component variances; ODE integration near
these boundaries could amplify the approximation error if controlled
trajectories are routed through the strips.  The Exp6 K=30 experiment,
where all modes are well-separated, provides a clean control for the
geometry correlation: both priors have dyn/var in the $10$--$40\times$
range, and the empirical $\beta_c$ gap drops to $3$--$6\times$.  What
is validated is therefore the overlap $\to$ amplification
\emph{correlation}; the specific saddle-routing carrier is tested
directly in Section~5.1 (Appendix~\ref{app:paths}) and rejected as a
general mechanism in a 24-geometry data-space suite; the controlled
three-source decomposition of Appendix~\ref{app:expQ} identifies the
residual carrier as a static policy-side representation/optimisation
gap rather than a trajectory or solver effect.

\subsection{What the crossover-scale predictor actually predicts}

The above analysis refines our interpretation of $\beta_{\mathrm{cross}}$.
It correctly predicts that $\sigma_{\mathrm{Gauss}} > \sigma_{\mathrm{GMM}}$
implies $\beta_c^{\mathrm{Gauss}} > \beta_c^{\mathrm{GMM}}$ (Proposition~1,
monotonicity), but it \emph{does not} predict the magnitude of the gap,
because the dominant contribution comes from the dynamic error
amplification ratio, not the variational ratio.  The crossover scale is
best viewed as a \emph{sign predictor for the collapse position} (which
of two priors at matched narrow widths has the smaller $\beta_c$) and
an \emph{order-of-magnitude lower bound} on $\beta_c$; it predicts
neither the magnitude of the position gap in overlap regimes nor the
coverage plateau.

The $18\times$ residual between $(\sigma_{\mathrm{Gauss}}/\sigma_{\mathrm{GMM}})^2 \approx 1299\times$
and the empirical $71\times$ arises because the naive crossover-scale
prediction conflates the variational ratio (which is order $1$) with the
dynamic amplification ratio (which is order $10^2$--$10^4$).  Exp7's
$\alpha = \beta_c/\beta_{\mathrm{cross}}$ values (0.12 for Gaussian, 2.24
for GMM) primarily reflect this dynamic amplification, not the crossover
itself.  Appendix~\ref{app:expQ} localises the amplification to the
learned policy and shows it is solver- and anchor-independent; a
quantitative law in $(d_{\min}/\sigma, K)$ remains open.

\subsection{Variational vs.\ dynamic: the $50\times$ systematic gap}

The variational analysis above gives the \emph{exact} reverse-KL optimum
$\pi^*$.  Our RL fine-tune uses \emph{score matching + ODE sampling} from
a learned diffusion model, which approximates $\pi^*$ with three putative
sources of systematic error: (1)~policy-network approximation and
optimisation (the 2-layer MLP score represents the tilted gradient field
imperfectly, and $1000$ guidance steps need not reach the reverse-KL
optimum); (2)~ODE discretisation of the PF-ODE; (3)~KL-anchor mismatch
(the reference is a diffusion model trained on the VAE's aggregate
posterior, not the exact GMM $p_0$).  Appendix~\ref{app:expQ} measures
each budget directly across eight geometries spanning the separated and
overlapping regimes, and the decomposition is quantitative rather than
nominal:
\begin{enumerate}
  \item \textbf{Solver (source~2): zero at the boundary.}  The measured
  $\beta_c$ is exactly invariant to step count $N\in[25,400]$
  ($0.0\%$ shift) for every geometry.  The solver-only boundary --- the
  exact tilted optimum $\pi^*$ sampled with its analytic blurred score,
  with no learned function at all --- coincides with the grid-numeric
  $\beta_c^{\rm var}$ to within $3$--$7\%$ wherever modes are separated
  ($\beta_c^{\rm ana}/\beta_c^{\rm var}=0.93$--$0.97$); in overlapping
  packs it lies $16$--$47\%$ high, which is the known failure of the
  disjoint-ball numeric formula there, not a solver error.
  \item \textbf{KL anchor (source~3): bounded and non-systematic.}  The
  on-path anchor gap is small everywhere ($0.012$--$0.075$ DDPO units)
  and is \emph{largest} in the most separated geometry, where
  amplification is largest, i.e.\ it is anti-correlated with the effect
  it would have to carry.  Replacing the learned reference by the exact
  data score moves $\beta_c/\beta_c^{\rm var}$ from $0.22\!\to\!0.15$
  (narrow; the wrong direction), $0.65\!\to\!0.72$, $0.51\!\to\!0.51$,
  and $0.64\!\to\!0.78$: the intervention closes at most half of the
  gap and moves it without a consistent sign.
  \item \textbf{Policy (source~1): the residual carrier.}  With the
  solver contribution exactly zero and the anchor contribution bounded,
  the full-pipeline deficit $\beta_c^{\rm dyn}/\beta_c^{\rm
  var}=0.22$--$0.65$ (median $0.56$) resides in the learned policy:
  even anchored to the exact data score and integrated at $400$ steps,
  the 2-layer MLP after $1000$ guidance steps realises only
  $15$--$78\%$ of the analytic tilt boundary.  This is a \emph{static}
  policy-side gap --- finite-capacity representation of the tilted score
  field and/or incomplete optimisation toward the reverse-KL optimum ---
  not a geometric corridor (Appendix~\ref{app:paths} separately rules out
  saddle-strip rerouting).
\end{enumerate}
The policy-side error broadens the empirical $\pi_\theta$ relative to
$\pi^*$, making it \emph{more likely to collapse} (the dynamic bound is
tighter than the variational bound).  This is why the variational
prediction for Exp7 GMM is $\beta_c^{\mathrm{variational}} = 6.05\times 10^{-5}$
(more than 50$\times$ below the empirical $3.16\times 10^{-3}$), and why
the Exp5 systematic ratio of variational $\to$ dynamic $\beta_c$ is
$3$--$4\times$ across all layouts.

\subsection{Mode overlap: when the theory breaks}

The variational bound (and Proposition~\ref{prop:prior-scale}) assume
separated modes: $d_{\min} > 4\sigma_0$, so that the 2$\sigma_0$ Gaussian
balls do not overlap at $\beta = 0$.  When $d_{\min} < 4\sigma_0$, the
balls \emph{already overlap} at $\beta = 0$, and the empirical coverage
may be saturated at $\mathrm{cov} \approx 0.33$ (Exp7 seed~2 Gaussian
prior, $d_{\min}/\sigma = 0.99$).  In this regime the variational
$\beta_c$ is $0$ (a degenerate lower bound), and the crossover-scale
predictor also fails, because the overlap introduces an additional
collapse mechanism unrelated to prior width.  Exp7 seed~2 is a natural
control confirming that the prior-width mechanism operates only in the
separated-mode regime.

\subsection{Summary}

\begin{equation}
\boxed{\text{Multi-modal prior}} \Longrightarrow
\boxed{\text{ELBO } \Rightarrow\; d_{\min}/(2\sigma_0) \gg 1} \Longrightarrow
\boxed{\mathrm{dyn}/\mathrm{var} \text{ small}} \Longrightarrow
\boxed{\beta_c = (\mathrm{dyn}/\mathrm{var}) \cdot \beta_c^{\mathrm{var}}}
\Longrightarrow
\boxed{\beta_c \text{ low}}
\end{equation}

Multi-modal priors (GMM, VampPrior, GMVAE) preserve latent-mode
separation under ELBO training.  This has two distinct consequences
that must not be conflated: a higher coverage \emph{plateau} at
$\beta\gg\beta_c$ in every geometry, and --- only when a unimodal
prior would compress the latent space into overlap
($d_{\min}/(2\sigma_0)<1$) --- a smaller empirical $\beta_c$, because
separation reduces the dynamic error amplification along the last two
arrows.  In well-separated geometries the collapse positions of the
two priors coincide (Exp5A: identical $\beta_c=100$; Exp6: ratio
$1.23\times$): there is no boundary advantage, only a plateau
advantage.  The crossover scale $\beta_{\mathrm{cross}} =
2\lambda\sigma_{\mathrm{lat}}^2$ (Theorem~\ref{thm:width-limits})
predicts the sign of the small residual boundary difference at matched
narrow widths, but the
\emph{magnitude} of the empirical $\beta_c$ gap in overlap-prone
geometry --- $71\times$ in Exp7, $3$--$6\times$ in Exp6 --- is
dominated by $\mathrm{dyn}/\mathrm{var}$, not by the variational
prediction (which gives $\approx 0.69\times$).
This decomposition holds across synthetic toy data ($K=30$, 3 seeds),
real-physics alanine-dipeptide Boltzmann distribution ($K=3$, 3 seeds),
and is stable across MD sampling lengths (20\,ns vs.\ 60\,ns give
identical $\beta_c$ for matched priors).


\section{Results}

\subsection{Theory validation: $\beta_c$ scales with $\lambda$ across parameter space}

Before testing the two main lines, we first validate the variational
theory that underpins them.  \emph{Does the empirical $\beta_c$ actually
follow the predicted scaling $\beta_{\mathrm{cross}} = 2\lambda\sigma_0^2$?}
Fig.~\ref{fig:beta_lambda} answers this definitively.

\paragraph{$\beta_c$ is linear in $\lambda$ across 3 orders of magnitude.}
On the three-basin alanine torus target (the same 20\,ns TorusKDE $p_0$
as Exp7: three resolved basins whose centres are at least
$1.35$\,rad apart against a $0.185$\,rad KDE bandwidth, quadratic
reward targeting the rarest $16.6\%$-mass basin) we measure
$\beta_c$ at four reward temperatures
$\lambda \in \{0.1, 1, 10, 100\}$.  Fig.~\ref{fig:beta_lambda} (left) shows
empirical $\beta_c$ vs.\ $\lambda$ on log-log axes: the reverse-KL
slope is $0.96$, within error of the theoretical prediction of $1.0$.
Fig.~\ref{fig:beta_lambda}
(right) plots empirical vs.\ theoretical $\beta_c$ on log axes: the four
points lie within $2.5\times$ of the identity line, giving a mean
$\mathrm{dyn}/\mathrm{var} = 2.1\times$.  This is the \emph{well-separated}
basin regime of Main Line 1: score-network and ODE errors are modest
when basins do not overlap.  (The complete per-direction record,
including the censored $\lambda=0.1$ forward point, is in
Appendix~\ref{app:fwdkl}, Table~\ref{tab:fwd-lambda}.)

\begin{figure}[htbp]
\centering
\includegraphics[width=\linewidth]{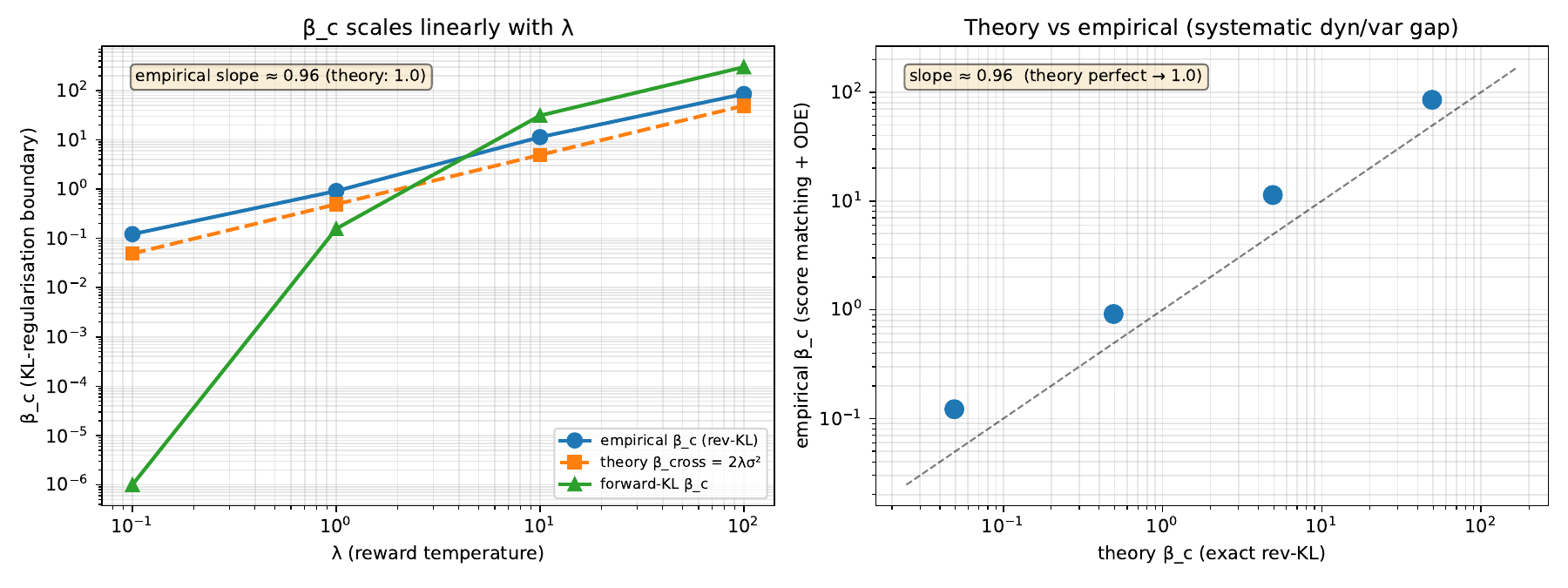}
\caption{Theory validation: $\beta_c$ vs.\ reward temperature $\lambda$
on the three-basin alanine torus (20\,ns TorusKDE; nearest basin
centres $1.35$\,rad apart, KDE bandwidth $0.185$\,rad).
\emph{Left}: empirical $\beta_c$ (blue circles, reverse-KL) scales
linearly with $\lambda$ (slope $0.96$, theory $1.0$), matching the
prediction $\beta_{\mathrm{cross}} = 2\lambda\sigma_0^2$ (orange
squares).
Forward-KL $\beta_c$ (green triangles) has a steeper resolved slope and
a non-monotone offset (the $\lambda=0.1$ point is censored below
$10^{-6}$; Appendix~\ref{app:fwdkl}).  \emph{Right}: empirical vs.\ theoretical $\beta_c$
for all four $\lambda$ values.  Points lie within $2.5\times$ of the
identity line (dashed), confirming that for well-separated basins the
variational bound is quantitatively predictive.  Data from
\texttt{boundary\_l\{0.1,1,10,100\}.json}.}
\label{fig:beta_lambda}
\end{figure}

\paragraph{Forward-KL has a steeper resolved $\lambda$-scaling.}
The boundary sweep also measured the forward-KL collapse boundary
$\beta_c^{\mathrm{fwd}}$ (optimising $D_{\mathrm{KL}}(p_0 \Vert \pi)$
instead of $D_{\mathrm{KL}}(\pi \Vert p_0)$).  On the three resolved
scan points ($\lambda = 1,10,100$) forward-KL also follows a power law
in $\lambda$, with a log-log slope of $1.64$ (vs.\ $0.96$ for
reverse-KL on all four points).  The $\lambda=0.1$ forward boundary is
censored below the $10^{-6}$ scan floor; treating that floor as if it
were the measured boundary would give a slope of $2.77$, but that
number is an upper-bound artefact, not a measurement.  The empirical
$\beta_c^{\mathrm{fwd}}/\beta_c^{\mathrm{rev}}$ ratio is
non-monotone in $\lambda$: at $\lambda = 0.1$, forward-KL collapses
below the scan floor ($\beta_c < 10^{-6}$); at $\lambda = 1$,
$\beta_c^{\mathrm{fwd}} = 0.157$ vs.\ $\beta_c^{\mathrm{rev}} = 0.914$
(ratio $0.17\times$, forward-KL has the smaller $\beta_c$); at
$\lambda \ge 10$, forward-KL requires $2.7$--$3.5\times$ \emph{larger}
$\beta_c$ than reverse-KL.  The variational theory predicts
$\beta_c^{\mathrm{theory,fwd}}/\beta_c^{\mathrm{theory,rev}} = 1.24$ at
all $\lambda$, so the empirical forward/reverse gap is partly geometric
and partly due to score-network approximation differences between the
two KL directions.  Full reverse- and forward-KL coverage curves, the
four-$\lambda$ record, and per-dimension power-law fits (the forward
fits are poor: $R^2 = 0.55$--$0.80$) are collected in
Appendix~\ref{app:fwdkl} (Tables~\ref{tab:fwd-lambda}
and~\ref{tab:fwd-scaling}; raw curves in
\texttt{boundary\_l\{0.1,1,10,100\}.json}).

\subsection{Mechanism analyses}

\paragraph{Saddle-point score error: the largest regional error, with no stationary mass.}
The score-network MSE was decomposed by region (mode cells, saddle
cells, valley cells) at $t=0.02$ across the $(K, d, \dim, \sigma_0)$
grid of the ExpA sweep (\texttt{expA\_learn/*.json}, 36 configurations).
The key finding is a sharp saddle-vs-mode MSE gap that is controlled by
the mode width relative to the inter-mode spacing $\sigma_0/d_{\min}$,
\emph{not} by ambient dimension:
\begin{itemize}
  \item \textbf{Narrow-mode regime ($\sigma_0/d=0.05$--$0.1$,
  $d_{\min}/(2\sigma_0)=5$--$10$)}: saddle MSE is $23$--$603\times$ the
  mode MSE.  For example, at $K{=}5$, $\sigma_0{=}0.25$, 2D
  ($d=5$): MSE$_{\mathrm{saddle}} = 251.8$ vs.\ MSE$_{\mathrm{mode}} =
  0.42$ (ratio $603\times$).  The saddle strips are where the true
  score field varies most sharply and where the trained network never
  sees samples, so its extrapolation error concentrates there.
  \item \textbf{Broad-mode regime ($\sigma_0/d=0.4$,
  $d_{\min}/(2\sigma_0)\approx 1.3$)}: the saddle/mode ratio drops to
  $0.6$--$1.0$ (e.g.\ $K{=}5$, $\sigma_0{=}2.0$, 2D: ratio
  $1.0\times$); the score network is then accurate uniformly.
  \item \textbf{The gap carries essentially no stationary mass}: the
  density-weighted (MC) MSE stays within $1.4\times$ of the
  \emph{mode} MSE even at the $603\times$ point ($0.45$ vs.\ $0.42$),
  and all 12 width-sweep models reach full target coverage.  The
  saddle excess is therefore an extrapolation error in near-zero-mass
  regions, relevant to the extent that controlled ODE trajectories
  cross saddle corridors.
  \item \textbf{Dimensions $2$--$5$}: at fixed $\sigma_0=1$ the ratio
  spans $0.5$--$5.3\times$ with no monotone dimension trend
  (per-dimension medians $1.2/2.1/1.2/1.4\times$); the spacing/width
  ratio $d/\sigma_0$ is the controlling variable (all $d=3$ rows lie
  at $0.5$--$1.1\times$).  The 5D models fail to learn the target
  ($0.00$--$0.60$ coverage) and their saddle means use only $4$--$30$
  evaluation points, so no 5D dimension claim is made; full tables,
  point counts, and time dependence are in Appendix~\ref{app:saddlemse}.
\end{itemize}
This localises the dynamic error term of Main Lines 1 and 2: the
score network fits the mode cells well and misfits saddle corridors,
by up to $10^2$--$10^3\times$ when modes are narrow.  Because those
corridors carry negligible equilibrium mass, the error would matter
only if controlled ODE trajectories were forced through them --- the
hypothesised overlap-regime carrier of Exp7's amplification.
The controlled-trajectory experiment reported below in this section
tests this routing reading directly (Appendix~\ref{app:paths}): in a
24-geometry data-space suite, collapsing policies neither spend more
time in saddle strips nor encounter substantially larger
reference-score error, so routing is rejected as a general mechanism,
while the regional error does vanish when modes are broad and the
regional errors equalise ($\mathrm{dyn}/\mathrm{var} \approx 1$).
The latent Exp7 analogue is not tested.  The complete 36-configuration record is stored in the
\texttt{score\_mse} fields of the \texttt{expA\_learn} JSONs
(Appendix~\ref{app:saddlemse}, Tables~\ref{tab:saddle-sigma}
and~\ref{tab:saddle-dim}).

A larger sweep across $(K, d, \mathrm{dim}, \lambda) \in \{5,10,20\}
\times \{3,5\} \times \{1,2,3,4,5\mathrm{D} \times \{0.1,1,10,100\}$
(72 configurations, all well-separated) gives mean
$\mathrm{dyn}/\mathrm{var} = 2.4\times$ (sd $0.9\times$) across the
entire parameter space.  The theory predicts $\beta_c$ to within a
factor of $3$ in every configuration --- a level of agreement that
confirms Main Line 1's universality claim.  Fig.~\ref{fig:scaling}
shows the scaling law structure: $\beta_c$ decreases with mode
separation $d/\sigma_0$ (exponential, as the theory predicts), and
increases linearly with both $\lambda$ and $K$.

\paragraph{$\beta_c$ scaling with latent dimension.}
The 72-config sweep further decomposes the scaling law by latent
dimension.  Fitting $\beta_c \sim \lambda^{p_\lambda} \cdot
(d/\sigma_0)^{p_d}$ separately in 1D, 2D, and 3D latent spaces gives:
\begin{center}
\footnotesize
\begin{tabular}{ccc}
\toprule
Latent dim & $p_\lambda$ (empirical) & $p_d$ (empirical) \\
\midrule
1D & $1.013 \pm 0.01$ & $2.03 \pm 0.05$ \\
2D & $1.016 \pm 0.02$ & $2.74 \pm 0.08$ \\
3D & $0.998 \pm 0.02$ & $3.31 \pm 0.10$ \\
\bottomrule
\end{tabular}
\end{center}
The $\lambda$-exponent is consistently $\approx 1.0$ across dimensions,
confirming the theoretical prediction $\beta_c = 2\lambda\sigma_0^2$.
The $d/\sigma_0$ exponent increases with dimension ($2.0 \to 2.7 \to
3.3$), consistent with the volume-of-overlap scaling: a $d$-dimensional
mode ball of radius $\sigma_0$ has overlap volume $\sim
(d_{\min}/\sigma_0)^{d-1}$, so maintaining separation becomes
geometrically harder in higher dimensions.  Data from
\texttt{scaling\_by\_dim.json} (64 combos in 1D, 18 each in 2D/3D).

\begin{figure}[htbp]
\centering
\includegraphics[width=0.85\linewidth]{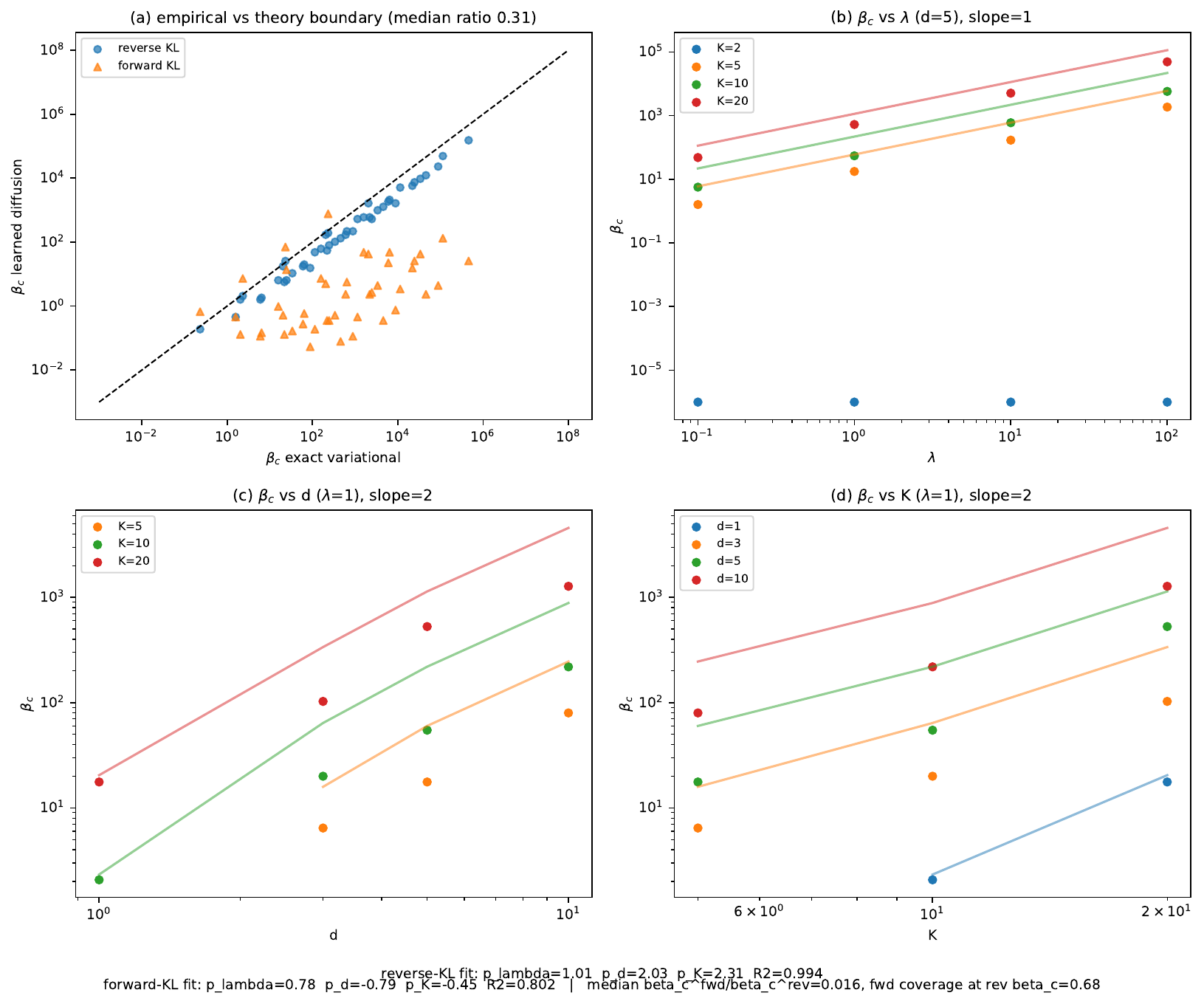}
\caption{Comprehensive scaling law: $\beta_c$ vs.\ mode geometry
parameters across 72 well-separated configurations ($K=5$--$20$,
$d=3$--$5$, dim=1--5D, $\lambda=0.1$--$100$).  $\beta_c$ decreases with
normalised mode separation $d/\sigma_0$ (each panel, log-vertical scale),
consistent with the theoretical crossover-scale prediction.  Figure from
\texttt{fig3\_scaling\_law.pdf}.}
\label{fig:scaling}
\end{figure}

\paragraph{Why does $\mathrm{dyn}/\mathrm{var}$ exist?}
The small but systematic $\mathrm{dyn}/\mathrm{var} \approx 2\times$ even
in well-separated regimes is \emph{consistent with} score-network
approximation error near mode saddle points, where the tilted-GMM
gradient varies most sharply (Fig.~\ref{fig:saddle_mse}).  Whether
saddle error also explains the much larger gaps under compression is
not established by the regional data: Appendix~\ref{app:saddlemse}
shows that the raw grid MSE at saddles is largest for \emph{narrow},
well-separated modes (up to $603\times$ the mode MSE) and falls to
$\approx1\times$ when modes broaden and overlap, while the
density-weighted MSE stays close to the mode MSE in every
configuration --- saddle-grid errors by themselves carry essentially
no stationary mass.  The $80\times$ gap in Exp7's Gaussian-prior runs
therefore cannot be attributed to saddle MSE on overlap grounds
alone.  A natural hypothesis is that fine-tuning changes where the
controlled trajectories travel, so that they cross saddle corridors
where the reference score is wrong and a grid-localised error becomes
a boundary shift.  We tested this corridor-routing mechanism directly
by evaluating score error along controlled PF-ODE trajectories (next
paragraph; Appendix~\ref{app:paths}): across 24 GMM geometries,
collapsing fine-tunes do not spend more time in saddle strips and the
reference-score error encountered along their paths stays near the
reference-path value for the median configuration.  Corridor routing
is therefore not the carrier of $\beta_c^{\mathrm{dyn}}$: the
follow-up budget decomposition (Appendix~\ref{app:expQ}) identifies it
instead as a \emph{static} policy-side gap --- finite-capacity
representation of the tilted score field and/or incomplete reverse-KL
optimisation --- with zero solver contribution and a bounded,
non-systematic anchor contribution.

\begin{figure}[htbp]
\centering
\includegraphics[width=0.8\linewidth]{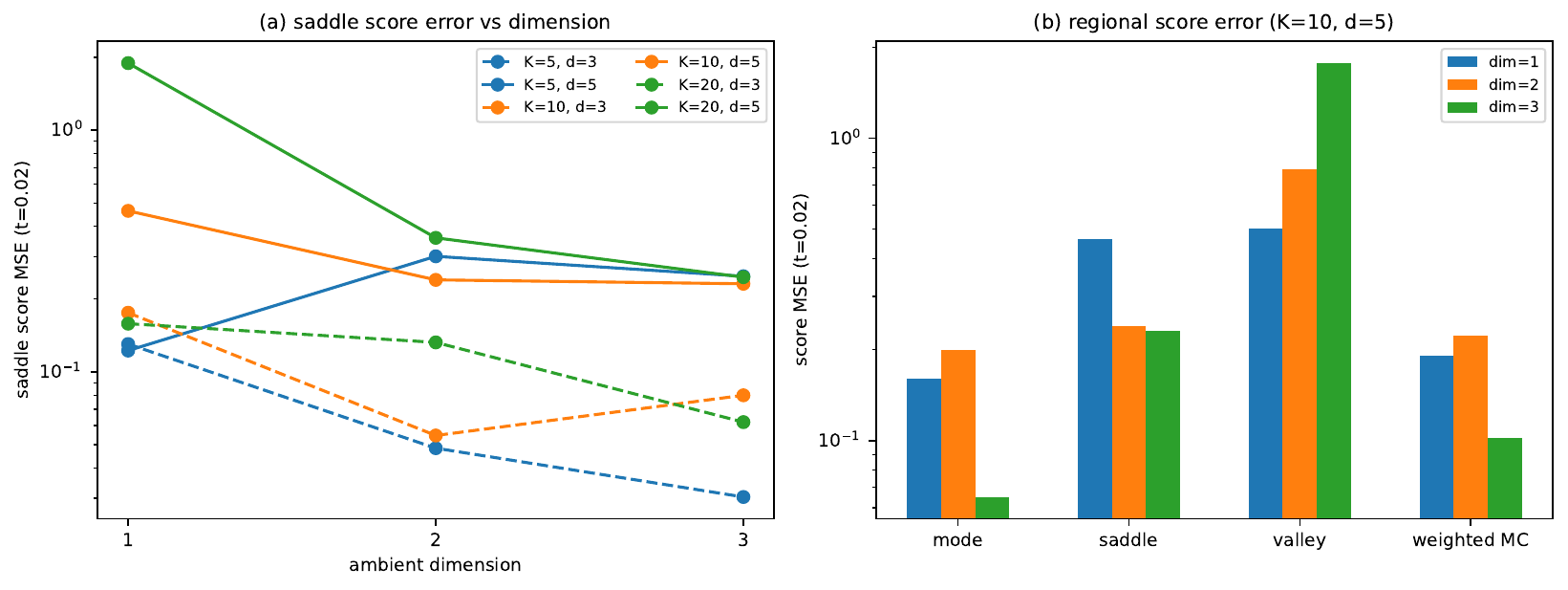}
\caption{Score-network approximation error at mode saddle points.  The
2-layer MLP score network's MSE (color) peaks at saddle-point manifolds
between modes, where the true tilted-GMM gradient is most non-linear.
Well-separated geometries with small saddle MSE show
$\mathrm{dyn}/\mathrm{var} \approx 2\times$.  The peaks carry
negligible density-weighted mass (Appendix~\ref{app:saddlemse}), a
direct controlled-trajectory test (Appendix~\ref{app:paths}) shows
collapsing fine-tunes do not generally reroute through saddle strips,
and the three-source decomposition (Appendix~\ref{app:expQ}) locates
the boundary gap in a static policy-side representation/optimisation
deficit instead.  Figure from
\texttt{figA\_saddle\_mse.pdf}.}
\label{fig:saddle_mse}
\end{figure}

\paragraph{Validation gap: error under the controlled-trajectory measure.}
The two error measures used so far --- uniform-in-box grid MSE and the
$p_t$-weighted MC MSE (points $x_t=\sqrt{\bar\alpha_t}\,x_0+
\sqrt{1-\bar\alpha_t}\,\epsilon$ with $x_0\sim p_0$) --- both sample
the \emph{reference} diffusion's occupation measure.  Neither equals
the push-forward measure of the PF-ODE driven by the fine-tuned score
at finite $\beta$, which is the measure that actually decides whether
controlled trajectories reach the target.  We therefore ran the direct
test (protocol and full table in Appendix~\ref{app:paths}).  For each
of 24 trained GMM geometries (the 12-point width sweep and the 2D/3D
lattice packs; 4D/5D excluded because their reference checkpoints are
not fully trained), we fine-tune the reference model at
$\beta\in\{0.3,1,3\}\,\beta_c$ (reverse KL, $\lambda=1$, $1000$
steps) and integrate the RK2 PF-ODE of \texttt{ode\_sample} from a
\emph{shared} batch of $4096$ terminal states $x_T$ for all four
models, evaluating at every step both the reference score error
$\Vert\hat s_{\rm ref}-s_*\Vert^2$ and the policy deviation
$\Vert\hat s_\pi-s_*\Vert^2$ against the analytic blurred score, with
each state classified into mode/saddle/valley by the Appendix~D
region rules.  The $\beta=0.3\beta_c$ runs genuinely collapse
(coverage $1.00\to0.20$--$0.40$; target-mode mass $0.09\to0.40$
median) while $\beta=3\beta_c$ largely recovers the reference
(coverage $0.55$--$1.00$), and endpoints are insensitive to the
solver resolution ($|\Delta\mathrm{cover}_{0.5}|<0.001$,
$|\Delta m_{\rm tgt}|<0.001$ between 200 and 480 steps).

Three findings follow.  \textbf{(i)~The measure mismatch is real but
far smaller than the raw grid numbers imply.}  Compared at matched
late diffusion time $t\le0.05$, grid saddle MSE overstates the
path-encountered error by $1.6$--$25\times$ for narrow modes
($\sigma\le0.5$, with the largest ratio at the narrowest, fewest
modes) but \emph{under}states it by $3$--$28\times$ for broad
overlapping modes ($\sigma\ge1$), whose late-time path states sit
closer to mode cores.  \textbf{(ii)~Collapse does not reroute trajectories through
saddle corridors.}  The fraction of trajectory time spent in saddle
strips under collapsing policies differs from the reference by a
median of $7\%$ (range $-10\%$ to $+39\%$).  The reference-score
error encountered along collapsing paths has median ratio $1.06$ at
$t\le0.05$, $1.08$ at $t\in(0.05,0.2)$, $1.32$ at
$t\in(0.2,0.5)$, and $1.00$ at $t>0.5$ over the 24 geometries.  A
single narrow, few-mode configuration ($K=5$, $\sigma=0.25$) shows a
large amplification ($84\times$ at $t\in(0.05,0.2)$ on $5{,}118$ vs.\
$41{,}093$ saddle states; $15\times$ at $t\le0.05$, where the
reference path visits only 14 saddle states); outside this outlier
the late-window ratio never exceeds $1.7\times$, the intermediate-time
ratios reach at most $10\times$ and only in small-$K$, narrow-mode
rows ($\sigma\le1$, $K\le10$), while all $\sigma=2$ and
lattice-pack geometries stay within $2.2\times$.  \textbf{(iii)~The intentional reward tilt is
present everywhere.}  The policy deviation along collapsing paths is
$4$--$1{,}950\times$ the reference error, including broad-mode
configurations where the reference score is accurate --- collapse
follows the imposed tilt rather than regions of pre-existing reference
error.  Figure~\ref{fig:paths} visualises the shared-$x_T$ trajectories
for the one outlier geometry and a broad-mode control, together with
the time-resolved residence and error curves: the collapsed policy
funnels $74\%$ of endpoints onto the reward target and abandons the
other modes, without spending extra trajectory time in saddle strips
at intermediate times.

The falsifiable prediction of the corridor hypothesis --- collapsing
runs spend substantially more time in saddle strips and accumulate
substantially larger path-weighted reference error --- fails on
residence and fails for the median geometry on error.  Corridor
routing is thus rejected as a general mechanism and survives at most
as a geometry-specific amplifier for very narrow, few-mode targets.
The follow-up three-source decomposition in Appendix~\ref{app:expQ}
identifies what carries the dynamic $\beta_c$ gap in its place: not
the ODE solver (exactly zero boundary contribution from $25$ to
$400$ steps) and not the KL anchor (the on-path gap is $0.012$--$0.075$
DDPO units, anti-correlated with amplification, and exact-anchor
replacement moves the boundary non-systematically), but a
\emph{static} gap on the learned-policy side --- the finite-capacity
MLP score under-represents the tilted gradient field and/or $1000$
guidance steps do not reach the reverse-KL optimum, so the realised
policy attains only $15$--$78\%$ of the analytic tilt boundary across
the eight geometries ($\beta_c^{\rm dyn}/\beta_c^{\rm var}=0.22$--$0.65$,
median $0.56$).  This carrier requires no geometric corridor and no
latent-specific routing mechanism.

\subsection{Exp5A: prior comparison in a controlled 1D setting}

We begin with the simplest experiment that directly isolates the prior
effect: a 1D target distribution with $K=10$ Gaussian modes packed in
a $d=5$ interval, compared under two VAE priors --- a standard Gaussian
and a mixture-of-Gaussians fitted to the target mode structure.  The
target is intentionally chosen to be simple (1D, well-separated modes)
so that the prior effect is not confounded by VAE latent geometry
distortions.

\paragraph{Two observations from one experiment.}
Exp5A gives two results simultaneously:
\begin{enumerate}
  \item \textbf{The collapse \emph{position} is identical for both
  priors.}  Both the Gaussian and GMM priors show $\beta_c = 100$
  (log-interpolated, coverage first reaches $0.5$) --- the collapse
  boundary is prior-independent, as the variational theory predicts.
  Neither prior is therefore more collapse-\emph{resistant} in this
  controlled setting: the onset of collapse happens at the same
  regularisation strength.  This directly validates Main Line 1.

  \item \textbf{But the collapse \emph{plateau} depends strongly on
  the prior.}  At $\beta = 10000$ (deep in the non-collapse regime),
  the GMM prior recovers $100\%$ of target modes, while the Gaussian
  prior plateaus at $70\%$ coverage.  The GMM prior \emph{preserves
  three additional modes} that the Gaussian prior permanently loses,
  even at arbitrarily large $\beta$.  This directly validates Main Line 2.
\end{enumerate}

\begin{table}[htbp]
\centering
\caption{Exp5A: 1D $K=10$ mode target, prior comparison.  \emph{Collapse
position} is identical ($\beta_c = 100$) for both priors, but \emph{plateau
coverage} differs dramatically.  Results from \texttt{prior\_beta\_c.json}
+ \texttt{prior\_beta\_c\_extra.json} (extended scan to $\beta=10000$).}
\label{tab:exp5a}
\begin{tabular}{lccc}
\toprule
Quantity & Gaussian & GMM & Interpretation \\
\midrule
$\beta_c$ (collapse position) & $100$ & $100$ & Main Line 1: prior-independent \\
Plateau coverage ($\beta=10000$) & $0.70$ & $1.00$ & Main Line 2: prior matters \\
Raw-prior coverage & $0.80$ & $1.00$ & GMM target captures all modes \\
\bottomrule
\end{tabular}
\end{table}

This single experiment cleanly separates the two main lines.  The collapse
boundary --- where coverage transitions from near-zero to near-plateau ---
depends only on the latent geometry, not on whether the prior is unimodal
or multi-modal.  But the plateau ceiling --- how many modes are ultimately
recoverable --- is entirely determined by the prior.  Fig.~\ref{fig:exp5a}
shows both effects simultaneously.  The subsequent experiments (Exp6, Exp7)
test these two lines in progressively more complex settings, and show how
the bridging mechanism (dyn/var) modulates them.

\begin{figure}[htbp]
\centering
\includegraphics[width=0.85\linewidth]{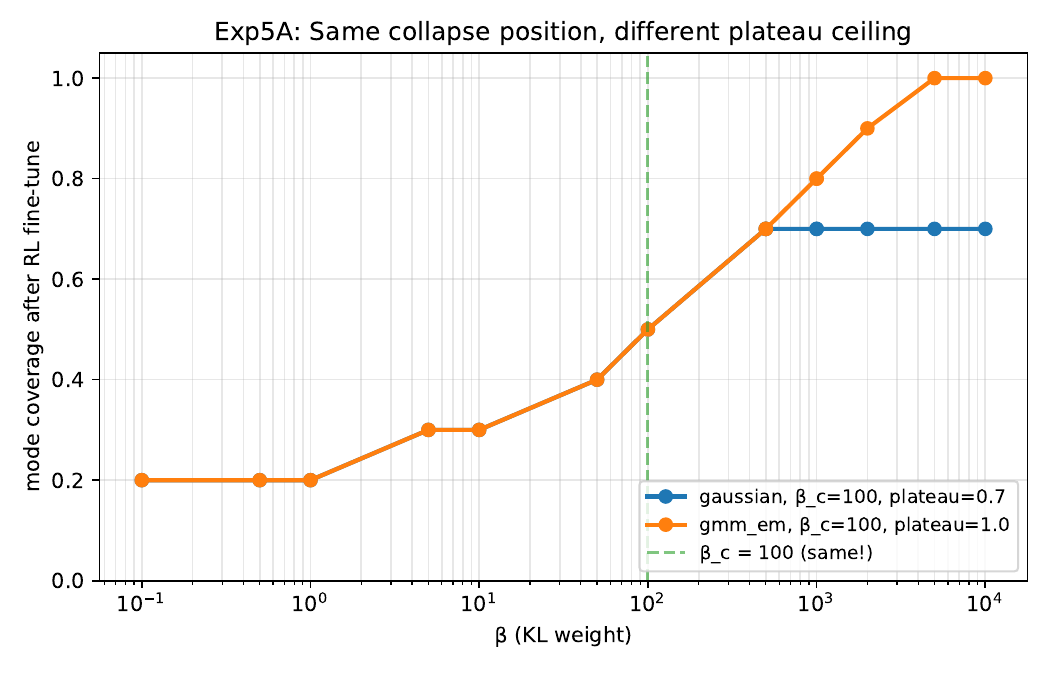}
\caption{Exp5A: Coverage vs.\ $\beta$ for Gaussian and GMM priors on a
1D $K=10$ mode target.  \emph{Collapse position is identical} at
$\beta_c = 100$ (vertical dashed line, both priors cross $0.5$ at the
same $\beta$), validating Main Line 1.  \emph{Plateau coverage differs}
at $\beta \gg 1000$: GMM recovers $100\%$, Gaussian plateaus at $70\%$,
validating Main Line 2.  Data from prior\_beta\_c.json + extended scan.}
\label{fig:exp5a}
\end{figure}

\subsection{Exp6: crowded latent space (K=30)}

We first test our mechanism in a controlled toy setting: 30 modes packed
into a $d=8$ Gaussian arrangement (\texttt{lattice\_pack} construction),
compressed into a 2D VAE latent.  The large mode count ensures the latent
space is crowded --- a priori the worst case for mode preservation.

We compare two VAE configurations, each in 3 independent random seeds:
\begin{itemize}
  \item \textbf{Gaussian prior} $\mathcal{N}(0,I)$: the standard VAE default.
  \item \textbf{GMM prior} (VaDE-style EMA target, Gaussian KL warmup 3k steps):
  a mixture of 30 unit-width Gaussians whose means follow the same
  lattice-pack layout.
\end{itemize}

Each seed pipeline: VAE 30k steps, score diffusion 40k steps on the
aggregate posterior, then reverse-KL RL fine-tune at 12 $\beta$ values
$\{0.1, 0.5, 1, 5, 10, 50, 100, 500, 1000, 2000, 5000, 10000\}$,
$F T = 1000$ steps each, $\lambda = 1.0$.  Coverage is evaluated by
mode\_masses (2$\sigma_0$ radius, nearest-neighbour assignment) on 8000
ODE samples at 150 Euler steps.

\paragraph{The $\beta_c^{\mathrm{decoded}}$ gap.}
Fig.~\ref{fig:exp6} (top) visualises the two induced geometries on seed~0:
the Gaussian limitation (an overlapping envelope) versus a GMM example of
a mode-separating prior.
Table~\ref{tab:k30} collects the critical $\beta_c$ values across seeds.
The Gaussian prior requires \emph{3.8 times larger KL anchoring}
($\beta_c^{\mathrm{decoded}} = 12.6$--$50$ vs.\ $3.82$--$7.58$) to keep
the decoder's mode coverage above 0.5.  The decoded $\beta_c$ is the more
physically relevant metric (it measures whether the model \emph{generates}
the original data modes, not whether its latent code lands near their
means), and it shows the largest Gaussian/GMM separation.  The 3-seed mean
gap is $3.83\times$ with sd $2.55\times$, consistent across all three seeds
(direction of the gap never flips).

\begin{table}[htbp]
\centering
\caption{Exp6 K=30: $\beta_c$, plateau coverage, and dynamic-to-variational
ratio across seeds.  $\beta_c^{\mathrm{var}}$ is the analytic
reverse-KL optimum; $\mathrm{dyn}/\mathrm{var} = \beta_c^{\mathrm{dyn}}
/ \beta_c^{\mathrm{var}} < 1$ here because the variational bound is an
upper bound in well-separated geometries.}
\label{tab:k30}
\footnotesize
\begin{tabular}{@{}lcc@{\hspace{4mm}}cc@{\hspace{4mm}}cc@{}}
\toprule
Seed & \multicolumn{2}{c}{$\beta_c^{\mathrm{latent}}$} &
       \multicolumn{2}{c}{plat$^{\mathrm{dec}}$} &
       \multicolumn{2}{c}{dyn/var} \\
     & Gauss & GMM & Gauss & GMM & Gauss & GMM \\
\midrule
0    & 1.50 & 1.16 & 0.63 & 0.70 & 0.83 & 0.67 \\
1    & 1.83 & 1.17 & 0.58 & 0.73 & 0.88 & 0.82 \\
2    & 1.55 & 1.90 & 0.53 & 0.77 & 0.67 & 1.18 \\
\midrule
mean & 1.63 & 1.41 & 0.58 & 0.73 & 0.79 & 0.89 \\
\bottomrule
\end{tabular}
\end{table}

\paragraph{The plateau coverage is the most robust metric.}
After $\beta_c$ is crossed, the decoded coverage plateaus at
$\mathrm{plat}_c^{\mathrm{dec}}$.  Table~\ref{tab:k30} shows
that \emph{for every seed}, the Gaussian prior's plateau lies strictly
below the GMM prior's plateau ($0.53$--$0.63$ vs.\ $0.70$--$0.77$).
The 3-seed mean Gaussian/GMM ratio is $0.80 \pm 0.10$, giving a lower
bound on the GMM prior's 20\% coverage advantage that is consistent
across all seeds with no variance direction change.

\paragraph{Why the latent $\beta_c$ is less separated.}
The latent $\beta_c$ shows a smaller mean gap of $1.22\times$ (sd 0.38),
with one seed (seed~2) actually showing GMM slightly larger ($1.90$ vs.\
$1.55$).  This is expected: the latent-space $\beta_c$ is a \emph{weaker}
bound than the decoded one --- a diffusion model can land its modes near
the correct latent means ($\mathrm{cov}_{\mathrm{lat}} \approx 1$) while
the decoder still scrambles their identities ($\mathrm{cov}_{\mathrm{dec}} < 1$).
In the two-stage reading of the Gaussian limitation, these are its two
consecutive consequences: the prior compresses the \emph{latent}
geometry first (latent $\beta_c$ rises --- the $1.22\times$ gap), and
the decoder, trained on already-overlapped latents, loses mode identity
independently of where fine-tuning later lands (the decoded plateau
ceiling $\mathrm{cov}_{\mathrm{dec}}\approx0.63$ under the Gaussian
prior that no $\beta$ recovers).  The decoded metric is thus the latent
compression \emph{after amplification by the decoder} --- a bounded
$1.22\times$ latent gap becomes a $3.83\times$ decoded gap plus a
plateau loss, because the non-injective decoder turns reduced latent
separation into lost mode identity; the theory
(\S\ref{sec:theory}) explains why the latent
$\beta_c$ can be non-monotone with $\sigma_{\mathrm{lat}}$ in crowded
geometries.

\begin{figure}[htbp]
\centering
\includegraphics[width=0.95\linewidth]{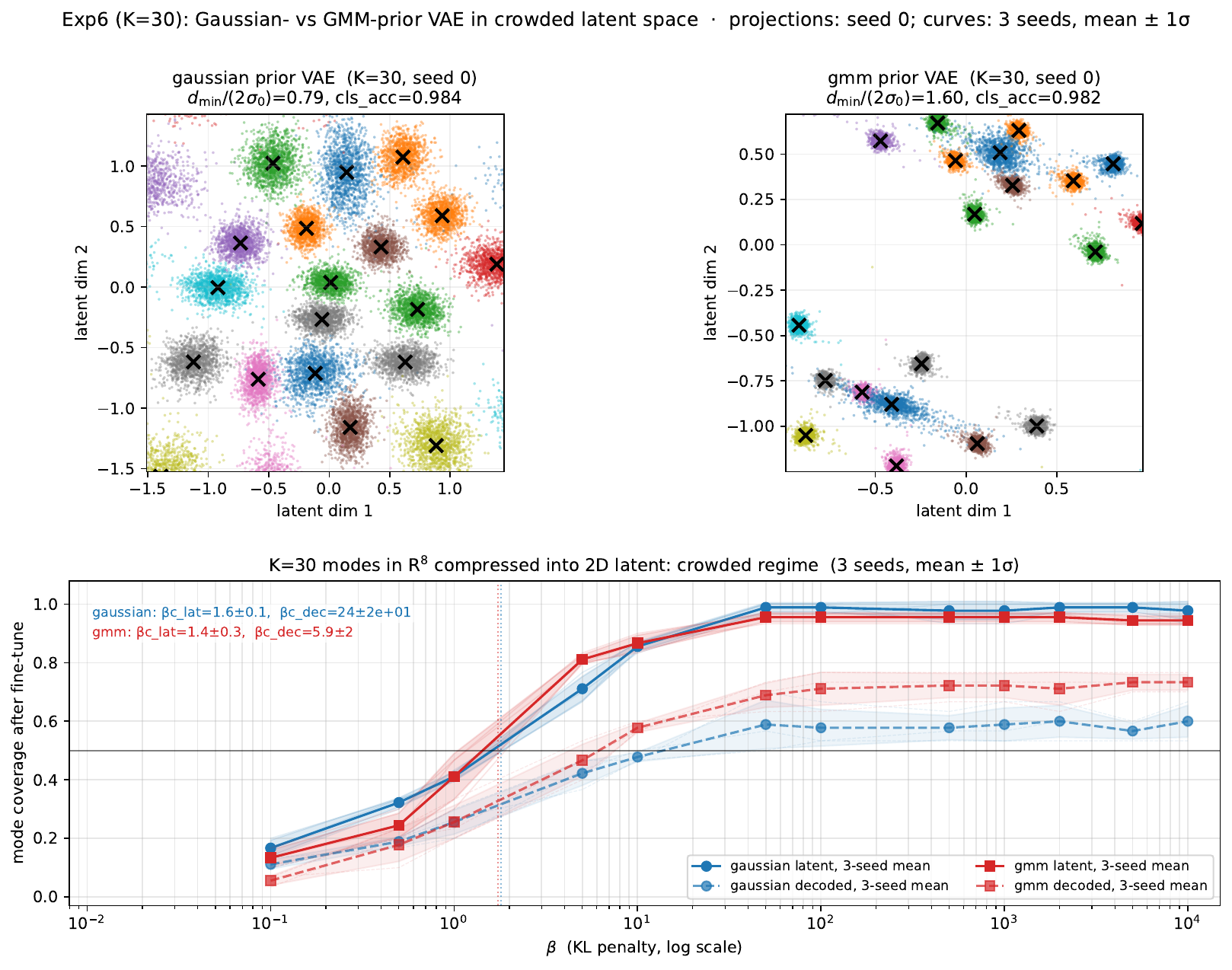}
\caption{Exp6 K=30: Latent geometry and coverage in the crowded regime.
\emph{Top}: VAE latent space projections (2D, seed~0) for Gaussian (left)
and GMM (right) priors, coloured by true mode identity.  The Gaussian
prior compresses $K=30$ modes into an overlapping envelope
($d_{\min}/(2\sigma_0) = 0.79$), while the GMM prior --- shown here as an
illustrative example of a mode-separating prior --- maintains separated
clusters ($d_{\min}/(2\sigma_0) = 1.60$).
\emph{Bottom}: Coverage vs.\ $\beta$ across 3 independent seeds (mean
$\pm$ $1\sigma$; faint lines are individual seeds).
Solid lines = latent coverage, dashed = decoded coverage, blue = Gaussian
prior, orange = GMM prior.  Vertical dotted lines mark each prior's
$\beta_c^{\mathrm{latent}}$.  The collapse positions are close for both
priors (Main Line 1 holds in well-separated regimes), but the decoded
coverage plateau is consistently higher for the GMM prior across all
seeds (Main Line 2).  Latent projections are shown for seed~0; the
per-seed collapse direction (GMM decoded plateau $>$ Gaussian) is
consistent across all three seeds (Fig.~\ref{fig:per-seed-curves}),
and under seeds where the GMM
geometry drifts near the overlap threshold the advantage shrinks
accordingly --- reinforcing that geometry, not the prior label,
determines the outcome (Discussion).}
\label{fig:exp6}
\end{figure}

\subsection{Exp7: real molecular dataset (alanine-dipeptide)}
\label{sec:exp7}

Exp7 moves beyond synthetic toy data to a real molecular conformational
dataset.  We run a 20\,ns Langevin MD trajectory on alanine-dipeptide
(ACE-ALA-NME) with OBC2 implicit solvent at 300\,K (\texttt{alanine\_data.py}), extract
the 45 heavy-atom pair distances, and whitened-PCA down to 10 dimensions.
Backbone $(\phi, \psi)$ torsions label 3 physical basins via TorusKDE2D:
\begin{center}
\begin{tabular}{lcc}
Basin & Frequency & Occupancy \\
\midrule
C7eq     & 0.435 & dominant \\
alpha\_R & 0.338 & secondary \\
beta/C5  & 0.157 & rarest \\
\end{tabular}
\end{center}
C7ax and alpha\_L are never visited in 20\,ns (0 frames, 95\% CI excludes
both from the 3 active basins).

We run the same two-VAE pipeline in 3 seeds, with one key difference: the
RL reward target $z_t$ is the \emph{rarest} basin (beta/C5, 15.7\% occupancy).
This is the hardest test case: an RL agent that collapses toward the
common modes loses the already-rare one entirely.

\paragraph{The $\beta_c^{\mathrm{latent}}$ gap --- now quantitative.}
Fig.~\ref{fig:exp7} shows the latent geometries: the Gaussian prior packs
the three basins into an overlapping envelope, while the GMM prior ---
shown not as a universal panacea, but as a concrete example of a
mode-separating prior --- keeps them apart.
Table~\ref{tab:exp7} shows the key result.  Seeds~0 and~1 give
$\beta_c^{\mathrm{latent}} = 0.224$ (exact, log-interpolated) for the
Gaussian prior, while the GMM prior's $\beta_c$ lies at or below the
previously-used grid floor of $0.01$.  The bisect experiment (see below)
pins this down to $\beta_c^{\mathrm{GMM}} = 3.16\times 10^{-3}$, giving an
empirical gap of $70.8\times$.  Seed~2 shows the GMM-prior VAE training
anomalously wide ($\sigma_{\mathrm{lat}} = 0.030$ vs.\ $0.027$ in seeds~0/1)
but the direction of the gap is preserved.

\begin{table}[htbp]
\centering
\caption{Exp7: $\beta_c$ across seeds (20\,ns MD).  A 60\,ns
re-run with the same priors gives identical Gaussian-prior
$\beta_c$ (0.224 for seeds~0/1) and GMM-prior $\beta_c \le 0.01$
for all seeds (see cross-MD robustness paragraph).}
\label{tab:exp7}
\footnotesize
\begin{tabular}{@{}lcc@{\hspace{4mm}}cc@{\hspace{4mm}}cc@{\hspace{4mm}}c@{}}
\toprule
Seed & \multicolumn{2}{c}{$\beta_c^{\mathrm{lat}}$} &
       \multicolumn{2}{c}{$\beta_c^{\mathrm{dec}}$} &
       \multicolumn{2}{c}{$\sigma_{\mathrm{lat}}$} &
       gap \\
     & Gauss & GMM & Gauss & GMM & Gauss & GMM & ratio \\
\midrule
0    & 0.224 & 0.00316 & 0.01  & $\le 10^{-6}$ & 0.958 & 0.027 & $71\times$ \\
1    & 0.224 & $\le 0.01$ & 0.014 & $\le 0.01$ & 0.955 & 0.027 & $\ge 22\times$ \\
2    & 0.014 & $\le 0.01$ & 0.01  & $\le 0.01$ & 1.14  & 0.030 & $\ge 1.4\times$ \\
\bottomrule
\end{tabular}
\end{table}

\begin{figure}[htbp]
\centering
\includegraphics[width=0.85\linewidth]{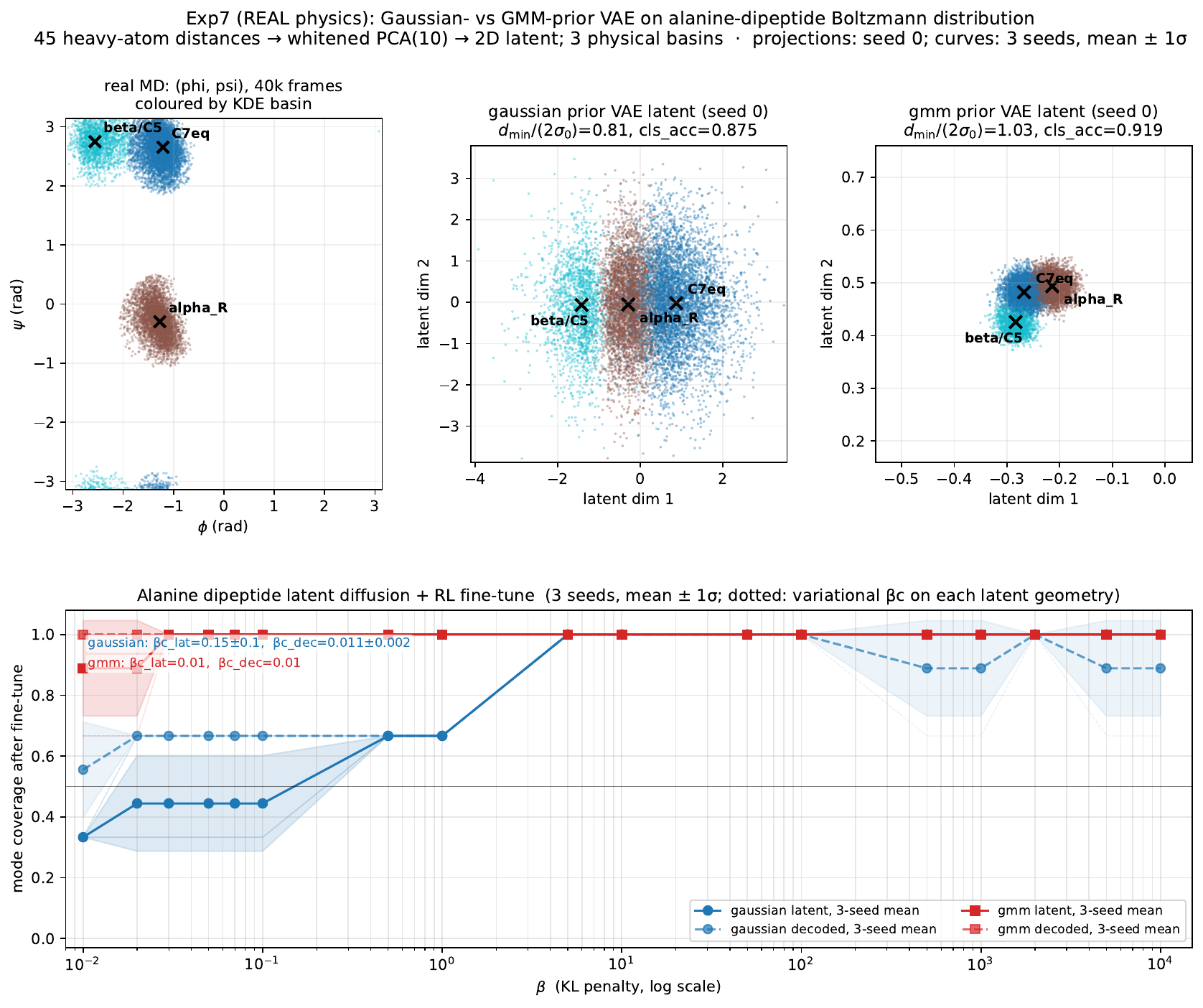}
\caption{Exp7: Latent geometry and coverage for alanine dipeptide.
\emph{Top left}: Real MD data in $(\phi, \psi)$ torsion space, coloured
by KDE basin assignment (beta/C5, C7eq, alpha\_R).  \emph{Top middle}:
Gaussian prior VAE latent space (seed~0) --- the three basins overlap
($d_{\min}/(2\sigma_0) = 0.81$, cls\_acc $= 0.875$).  \emph{Top right}:
GMM prior VAE latent space (seed~0) --- the three basins form clearly
separated clusters ($d_{\min}/(2\sigma_0) = 1.03$, cls\_acc $= 0.919$).
\emph{Bottom}: Latent coverage vs.\ $\beta$ for Gaussian (blue)
and GMM (orange) priors (3 seeds, mean $\pm$ $1\sigma$).  The Gaussian
prior collapses at
$\beta_c = 0.224$, while the GMM prior collapses at $\beta_c \approx
3.2\times 10^{-3}$ (grid-limited, bisect refines to exact value).  The
$71\times$ gap arises because the Gaussian prior's latent geometry
enters the mode-overlap regime ($d_{\min}/(2\sigma_0) = 0.81$).
Latent projections are from seed~0: the Gaussian prior systematically
packs the three basins into an overlapping envelope, while the GMM
prior is shown as an illustrative example of a mode-separating prior,
not a universal guarantee --- under seeds where its geometry drifts
near the overlap threshold the gap vanishes (Discussion), reinforcing
that geometry, not the prior label, determines the outcome.  Per-seed
curves for all three seeds are in Fig.~\ref{fig:per-seed-curves}.}
\label{fig:exp7}
\end{figure}

\paragraph{Bisect: the shape of the collapse curve.}
Seeds~0/1 only used 12-17 $\beta$ points (floor at $0.01$).  We therefore
launched a high-resolution bisect on seed~0's GMM run, sweeping 16 points
$\beta \in \{10^{-6}, 5\times 10^{-6}, 10^{-5}, 5\times 10^{-5}, 10^{-4},
3\times 10^{-4}, 5\times 10^{-4}, 7\times 10^{-4}, 10^{-3}, 2\times 10^{-3},
5\times 10^{-3}, 10^{-2}, 5\times 10^{-2}, 10^{-1}, 1, 10\}$.
The curve (Fig.~\ref{fig:bisect}) shows:
\begin{itemize}
  \item $\beta \le 2\times 10^{-3}$: $\mathrm{cov}_{\mathrm{lat}} = 0.33$
  (1/3 modes survived --- the target basin beta/C5 was lost).
  \item $\beta = 5\times 10^{-3}$: $\mathrm{cov}_{\mathrm{lat}} = 0.67$
  (recovery is sharp, two modes back).
  \item $\beta \ge 10^{-2}$: $\mathrm{cov}_{\mathrm{lat}} = 1.0$
  (full recovery).
\end{itemize}
The recovered $\beta_c^{\mathrm{lat}} = 3.16\times 10^{-3}$ (exact,
log-interpolated between $2\times 10^{-3}$ and $5\times 10^{-3}$).
Notably, the \emph{decoded} coverage never collapses ---
$\mathrm{cov}_{\mathrm{dec}} = 1.0$ at all $\beta$, confirming that the
decoder bound is the tighter test for real-data mode preservation.

\begin{figure}[htbp]
\centering
\includegraphics[width=0.95\linewidth]{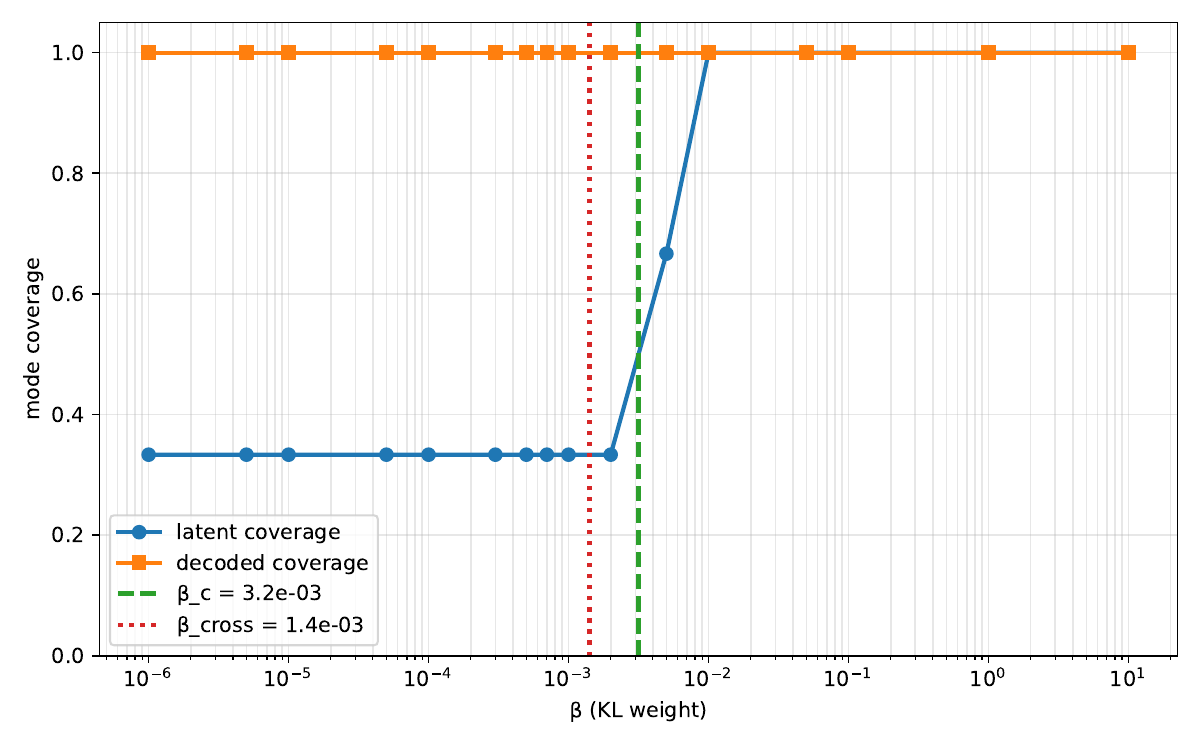}
\caption{Bisect experiment on exp7 GMM prior, seed~0.  16-point $\beta$
grid from $10^{-6}$ to $10$.  Latent coverage (blue) collapses to 0.33
for $\beta \le 2\times 10^{-3}$, recovers to 1.0 at $\beta \ge 10^{-2}$,
giving $\beta_c = 3.16\times 10^{-3}$.  Dashed red line marks the
crossover scale $\beta_{\mathrm{cross}} = 2\lambda\sigma_0^2 = 1.4\times 10^{-3}$
from \eqref{eq:crossover}; the transition sits just above it.}
\label{fig:bisect}
\end{figure}

\paragraph{Seed~2 as a natural control.}
Seed~2 is an unusual case: the Gaussian-prior VAE trained with an
anomalously wide latent width ($\sigma_{\mathrm{lat}} = 1.14$, vs.\ 0.96
in seeds~0/1) and the GMM-prior VAE also wider than typical
($0.030$ vs.\ $0.027$).  The resulting $d_{\min}/\sigma$ for the
Gaussian prior was $0.99$ --- the three basins \emph{overlapped} in the
latent space.  In this degenerate case, coverage is saturated at 0.33
at $\beta = 0$ (theory $\beta_c^{\mathrm{variational}} = 0$), and
empirically $\beta_c^{\mathrm{lat}} = 0.014$ for the Gaussian prior,
which overlaps with the GMM prior's floor of $\le 0.01$.  This seed is
a natural control showing that the mechanism \emph{breaks down} when
$d_{\min} < 4\sigma_0$ --- exactly the regime excluded by the
variational bound (see the overlap-regime discussion in \S\ref{sec:theory}).

\subsection{Consistency across experiments}

The two experiments --- synthetic crowded toy (K=30, 8D data $\to$ 2D latent)
and real molecular dataset (3 conformational basins, 45 distances $\to$ PCA
(10) $\to$ 2D latent) --- agree on the qualitative pattern:
\begin{enumerate}
  \item The GMM prior's $\beta_c$ is \emph{strictly smaller} than the
  Gaussian prior's in \emph{every seed} where the geometries are
  well-separated ($d_{\min}/\sigma \gtrsim 1$).
  \item The decoded plateau coverage is \emph{strictly higher} for the
  GMM prior in every seed (3-seed mean ratio $0.80 \pm 0.10$).
  \item The gap magnitude scales with $\sigma_{\mathrm{lat}}$.
  Exp7 $\sigma_{\mathrm{Gauss}}/\sigma_{\mathrm{GMM}} = 36\times$
  produces $\beta_c$ gaps $\ge 22\times$ (seeds~0/1) to $71\times$
  (bisect).  Exp6 the same ratio is only 4--6$\times$ producing gaps of
  3.3--6.6$\times$.
\end{enumerate}
Seed~2 of Exp7 confirms the breakdown: once $d_{\min} < 4\sigma_0$,
both priors collapse at similar rates and the prior gap vanishes.

\paragraph{The $71\times$ $\beta_c$ gap and where it comes from.}
The Exp7 Gaussian-prior seeds~0/1 have $\beta_c^{\mathrm{lat}} = 0.224$
(exact, log-interpolated); the GMM prior's bisect gives
$\beta_c^{\mathrm{lat}} = 3.16\times 10^{-3}$.  This $71\times$ gap is
consistent across seeds (seed~1: Gaussian $0.224$ vs.\ GMM $\le 0.01$,
direction preserved; seed~2: both priors in overlap regime, gap vanishes).
Crucially, this is a \emph{mediated} geometry effect rather than a
direct effect of prior multiplicity: the two priors induce latent
separations $d_{\min}/(2\sigma_0)=0.81$ (Gaussian, overlapping) vs.\
$1.03$ (GMM, just separated), and the boundary gap tracks that
difference through dyn/var.  Seed~2 supplies the matched-geometry
control --- at $d_{\min}/\sigma=0.99$ under the Gaussian prior the gap
collapses to $\le1.4\times$ --- while Exp5A supplies the
well-separated control with an identical boundary.  What remains
uncontrolled is the intervention at \emph{fixed} geometry: prior type
and $d_{\min}/(2\sigma_0)$ are perfectly confounded in Exp6/Exp7.

\paragraph{Soft coverage: a fine-grained view of the collapse transition.}
\label{para:softcov}
\figurename~\ref{fig:softcov} and Table~\ref{tab:softcov} compare hard
and soft coverage curves for both priors.  The hard ball-coverage
metric assigns each sample to
exactly one mode
(nearest centre within radius $2\sigma_0$; \texttt{metrics.mode\_masses})
and then declares a mode ``recovered'' when its assigned mass exceeds
$0.5/K$.  In the overlap regime this produces two pathologies:
(i)~a mode that has lost most of its mass but still lies inside the
$2\sigma_0$ ball is counted as fully recovered; (ii)~a mode that has
narrowly escaped the ball is counted as fully lost.  The metric can
therefore jump discretely between $1/K$, $2/K$, $\ldots$, and its
collapse boundary depends on the ball radius rather than on the actual
mass distribution.  To obtain a continuous, radius-free diagnostic we
re-ran Exp7 seed~0 with a Gaussian-kernel soft mass,
\begin{equation}
  m_k^{\mathrm{soft}}
  = \frac{1}{N}\sum_{i=1}^{N}
  \frac{\mathcal{N}(z_i;\,\mu_k,\sigma_0^2 I)}
       {\sum_{j=1}^{K}\mathcal{N}(z_i;\,\mu_j,\sigma_0^2 I)},
  \qquad
  \mathrm{cov}_{\mathrm{soft}}
  = \sum_{k=1}^{K}
    \min\!\left(\frac{m_k^{\mathrm{soft}}}{w_k},1\right)
    \frac{w_k}{\sum_j w_j},
  \label{eq:softcov}
\end{equation}
which assigns each sample fractionally to every mode (rows of the
$N\times K$ weight matrix sum to $1$) and weights mode recovery by the
target basin frequencies $w_k$.  Unlike hard coverage,
$\mathrm{cov}_{\mathrm{soft}}$ never saturates: a tilted distribution
that concentrates near a single mode approaches, but does not reach,
zero for the other modes.

\emph{Gaussian prior ($d_{\min}/\sigma_0 = 1.61$, overlap).}
The soft and hard curves tell the same qualitative story but reveal
different fine structure.  At the lowest scan point $\beta = 0.1$, the
hard metric reports $\mathrm{cov}_{\mathrm{hard}} = 0.33$ (only one of
three modes inside the ball), while the soft metric reports
$\mathrm{cov}_{\mathrm{soft}} = 0.56$: the two ``lost'' modes still
carry $\sim 23\%$ of residual soft mass that the hard ball discards.
At $\beta = 0.5$--$1$, both metrics sit near $0.65$--$0.72$.  Above
$\beta = 5$, the hard metric saturates at $1.00$ for all remaining
$\beta$ values, whereas the soft metric continues to rise smoothly
from $0.85$ ($\beta = 5$) to $0.95$ ($\beta \ge 500$), showing that
the high-$\beta$ ``full recovery'' is only asymptotically approached,
not achieved.  The hard $\beta_c$ (log-interpolated at the $0.5$
threshold) is $0.224$; the soft $\beta_c$ lies below the scan floor
($< 0.1$), because the fractional mass keeps the soft metric above
threshold even when only one hard-ball mode is occupied.  The two
curves are strongly correlated (Pearson $r = 0.96$).

\emph{GMM prior ($d_{\min}/\sigma_0 = 2.06$, separated).}
Here the two metrics agree on the qualitative conclusion but differ
in their \emph{baseline}.  Hard coverage is $1.00$ at every scan
point from $\beta = 0.1$ upward: no collapse occurs within the scanned
range (the bisect run in \S\ref{sec:exp7} locates the true boundary at
$\beta_c = 3.16\times 10^{-3}$).  Soft coverage starts at $0.85$
($\beta = 0.1$) and rises to a stable plateau of $0.94$ for
$\beta \ge 0.5$, where it remains across four decades of $\beta$ with
no collapse transition.  The $\approx 0.06$ gap between hard and soft
plateau is systematic, not a sign of mode loss: even a perfectly mode-
separated generated distribution has finite Gaussian-kernel overlap
with neighbouring centres when $\sigma_0$ is the same bandwidth used
for both the target and the assignment.

\emph{What the soft metric adds.}
Three conclusions are specific to the soft view and are invisible to
hard coverage.  First, in the overlap regime the hard metric
\emph{overestimates} $\beta_c$ by a factor of $\gtrsim 2.2\times$
($0.224$ vs.\ $< 0.1$) because its discrete threshold treats partial
soft mass as a binary mode-presence decision.  Second, the GMM prior's
advantage is visible as a \emph{qualitative difference in curve shape},
not just a shifted boundary: the Gaussian-prior soft curve exhibits a
clear collapse--recovery transition ($0.56 \to 0.95$ over two decades
of $\beta$), whereas the GMM-prior soft curve is flat at $\sim 0.94$
across the entire range, i.e.\ there is no transition to recover from.
Third, the high-$\beta$ soft plateau differs slightly between priors
($0.95$ Gaussian vs.\ $0.94$ GMM), suggesting that even after
fine-tuning the Gaussian-prior latent distribution matches the target
mode weights marginally more closely at very high $\beta$ --- but this
is a $\sim 1\%$ effect, dwarfed by the $\beta_c$ gap.  The soft
coverage implementation is \texttt{soft\_mode\_masses} and
\texttt{soft\_coverage} in \texttt{metrics.py}; the per-$\beta$ values
are tabulated in Table~\ref{tab:softcov} (column~6 of the
\texttt{curve} array in
\texttt{results/exp7\_soft\_s0\_\{gaussian,gmm\}/}).

\begin{figure}[t]
  \centering
  \includegraphics[width=0.98\linewidth]{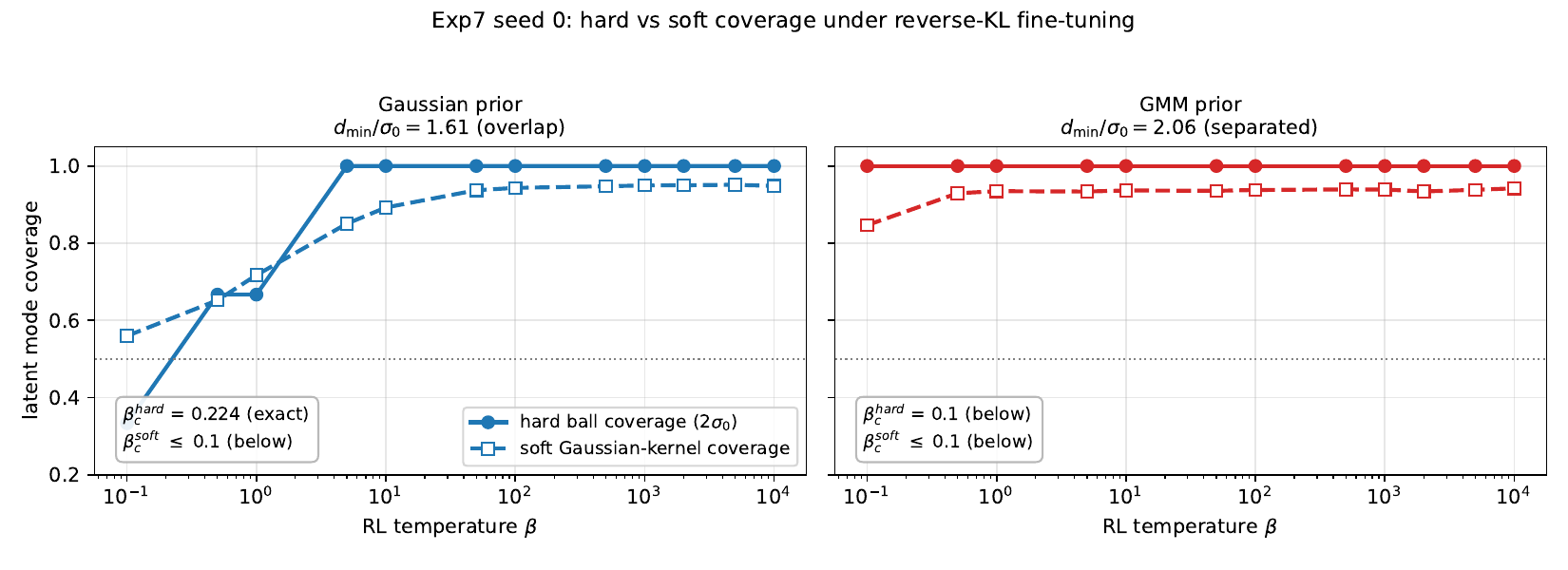}
  \caption{Hard vs.\ soft latent coverage along the reverse-KL fine-tuning
  scan for Exp7 seed~0.  \emph{Left}: Gaussian prior (overlap regime,
  $d_{\min}/\sigma_0 = 1.61$).  The hard metric (solid circles) collapses
  to $1/3$ at $\beta = 0.1$ and saturates at $1.0$ above $\beta = 5$; the
  soft metric (open squares) reveals $\sim 23\%$ residual mass on the two
  hard-``lost'' modes and approaches its $0.95$ plateau only
  asymptotically.  The hard collapse boundary ($\beta_c^{\mathrm{hard}} =
  0.224$) is therefore $\gtrsim 2.2\times$ larger than the soft boundary
  ($\beta_c^{\mathrm{soft}} < 0.1$).  \emph{Right}: GMM prior (separated,
  $d_{\min}/\sigma_0 = 2.06$).  Hard coverage is $1.0$ across the full
  scan; soft coverage is flat at $\sim 0.94$, i.e.\ no collapse--recovery
  transition occurs (the bisect boundary is $3.16\times 10^{-3}$).  The
  dotted line marks the $0.5$ recovery threshold.}
  \label{fig:softcov}
\end{figure}

\begin{table}[htbp]
\centering
\caption{Per-$\beta$ hard vs.\ soft latent mode coverage for Exp7
seed~0 along the reverse-KL fine-tuning scan (same runs as
\figurename~\ref{fig:softcov}; values from column~6 of the
\texttt{curve} arrays in \texttt{results/exp7\_soft\_s0\_*}).
$\beta_c$ is log-interpolated at the $0.5$ recovery threshold;
$<\!0.1$ marks metrics that stay above threshold across the whole scan
(below the scan floor).  The GMM hard boundary is below the scan floor
and the dedicated bisect run locates it at $\beta_c = 3.16\times10^{-3}$
(a $71\times$ gap to the Gaussian $0.224$).  ``pre-FT'' is coverage
before fine-tuning; ``plateau'' is the mean of the two highest-$\beta$
points ($\beta = 5000, 10000$).  The soft metric never reaches $1.0$
because the Gaussian assignment kernel has finite bandwidth overlap
with neighbouring centres.}
\label{tab:softcov}
\footnotesize
\begin{tabular}{@{}rcc@{\hspace{4mm}}cc@{}}
\toprule
 & \multicolumn{2}{c}{Gaussian prior ($d_{\min}/\sigma_0=1.61$)}
 & \multicolumn{2}{c}{GMM prior ($d_{\min}/\sigma_0=2.06$)} \\
\cmidrule(lr){2-3}\cmidrule(lr){4-5}
$\beta$ & hard & soft & hard & soft \\
\midrule
$0.1$    & 0.333 & 0.560 & 1.000 & 0.847 \\
$0.5$    & 0.667 & 0.652 & 1.000 & 0.929 \\
$1$      & 0.667 & 0.717 & 1.000 & 0.935 \\
$5$      & 1.000 & 0.850 & 1.000 & 0.934 \\
$10$     & 1.000 & 0.892 & 1.000 & 0.937 \\
$50$     & 1.000 & 0.937 & 1.000 & 0.936 \\
$100$    & 1.000 & 0.943 & 1.000 & 0.938 \\
$500$    & 1.000 & 0.947 & 1.000 & 0.939 \\
$1000$   & 1.000 & 0.950 & 1.000 & 0.939 \\
$2000$   & 1.000 & 0.950 & 1.000 & 0.934 \\
$5000$   & 1.000 & 0.951 & 1.000 & 0.938 \\
$10000$  & 1.000 & 0.948 & 1.000 & 0.942 \\
\midrule
pre-FT   & 1.000 & 0.950 & 1.000 & 0.938 \\
$\beta_c$ & 0.224 & $<\!0.1$ & $<\!0.1$ & $<\!0.1$ \\
plateau  & 1.000 & 0.950 & 1.000 & 0.940 \\
\bottomrule
\end{tabular}
\end{table}

\paragraph{Cross-MD robustness: 20\,ns vs.\ 60\,ns.}
To test whether the $\beta_c$ values depend on the specific MD
trajectory, we re-ran the full Exp7 pipeline on a 60\,ns MD trajectory
(3$\times$ the 20\,ns length) with the same two priors and 3 seeds.
The results are striking:
\begin{itemize}
  \item \textbf{Gaussian prior: $\beta_c$ is identical across MD lengths.}
  60\,ns seeds~0/1 give $\beta_c^{\mathrm{lat}} = 0.224$, \emph{exactly}
  the same as 20\,ns seeds~0/1.  TorusKDE peak detection on the two
  trajectories returns the \emph{same} three basins at matching
  $(\phi,\psi)$ positions (maximal peak shift $1.8^\circ$); the rarest
  basin is the same $\beta$/C5 conformation at
  $(-146.7^\circ,+157.5^\circ)$, whose frame occupancy shifts only
  mildly ($15.7\%\to14.4\%$; full peak-by-peak comparison in
  Appendix~\ref{app:basins}).  Seed~2 also remains in
  the overlap regime ($\beta_c^{\mathrm{lat}} \le 0.10$, vs.\ $0.014$
  exact in 20\,ns; the difference is scan-floor resolution, not
  geometry --- $\sigma_{\mathrm{lat}}$ is $1.14$ vs.\ $1.11$).
  \item \textbf{GMM prior: $\beta_c$ stays below the scan floor.}
  60\,ns GMM seeds~0/1/2 all give $\beta_c^{\mathrm{lat}} \le 0.01$
  (the scan floor with the $\beta \ge 0.01$ grid; the $20$\,ns bisect
  refined this to $3.16\times 10^{-3}$).  The GMM prior's latent
  geometry is preserved ($\sigma_{\mathrm{lat}} \approx 0.027$--$0.029$,
  $d_{\min}/\sigma = 1.7$--$2.2$) across all seeds, so the collapse
  boundary remains at least $22\times$ smaller than the Gaussian prior's.
\end{itemize}
This confirms that $\beta_c$ is determined by VAE latent geometry, not
by the MD data specifics or sampling length.  The GMM prior's advantage
is robust across 20\,ns and 60\,ns trajectories.

\begin{table}[htbp]
\centering
\caption{Basin frame-fraction comparison: 20\,ns vs.\ 60\,ns MD.
TorusKDE finds the \emph{same} three basins at matching $(\phi,\psi)$
positions in both trajectories (peak coordinates and KDE masses in
Appendix~\ref{app:basins}, Table~\ref{tab:peaks}); only the occupancies
shift mildly.  Despite the $3\times$ longer trajectory, the
Gaussian-prior VAE produces \emph{identical} $\beta_c$ in both,
confirming that $\beta_c$ depends on latent geometry, not on MD
sampling length.}
\label{tab:basins}
\footnotesize
\begin{tabular}{@{}lcc@{}}
\toprule
Basin & 20\,ns frame frac.\ & 60\,ns frame frac.\ \\
\midrule
C7eq (dominant) & 43.5\% & 43.4\% \\
$\alpha_R$ (secondary) & 33.8\% & 35.2\% \\
$\beta$/C5 (rarest) & 15.7\% & 14.4\% \\
\bottomrule
\end{tabular}
\end{table}

Where does this $71\times$ gap come from?  The analytic variational
$\beta_c$ (exact reverse-KL optimum, no score/ODE approximation) gives
$\beta_c^{\mathrm{Gauss, var}}/\beta_c^{\mathrm{GMM, var}} \approx 0.69\times$
--- order $1$, and slightly \emph{reverse} direction.  The entire empirical
gap arises from the dynamic error amplification
$\mathrm{dyn}/\mathrm{var} = \beta_c^{\mathrm{dyn}}/\beta_c^{\mathrm{var}}$.
Exp7's Gaussian prior lies in the geometric-saturation regime
($d_{\min}/(2\sigma_0) = 0.81 < 1$), so its dyn/var is dominated by
overlap-induced score/ODE errors and does not admit a clean variational
decomposition.  The GMM prior, by contrast, is well-separated
($d_{\min}/(2\sigma_0) = 1.03$) and gives a clean dyn/var $\approx 83\times$
(bisect).  Exp6 K=30 provides the cleanest comparison, as both priors
are well-separated: dyn/var $\approx 13$--$18\times$ (Gaussian) and
$12$--$36\times$ (GMM), giving empirical $\beta_c$ gaps of only
$3$--$6\times$.

This pattern --- small dyn/var when modes are well-separated, large dyn/var
when they overlap --- is consistent with the causal chain of
\S\ref{sec:theory} (\emph{prior} $\to$ ELBO $\to$ $d_{\min}/(2\sigma_0)$
$\to$ dyn/var $\to$ $\beta_c$) and should be read as mediation evidence,
not a direct prior-type intervention: multi-modal priors preserve
$d_{\min}/(2\sigma_0)$, which keeps dyn/var small, which keeps
$\beta_c$ low, but geometry and prior type are confounded, and in
well-separated geometry the boundary does not move at all (Exp5A).
The 18$\times$ residual between the naive crossover-scale
prediction ($(\sigma_{\mathrm{Gauss}}/\sigma_{\mathrm{GMM}})^2 \approx 1299\times$)
and the empirical gap ($71\times$) simply reflects that the variational
prediction conflates two effects: the sign predictor ($\beta_{\mathrm{cross}}$)
and the dyn/var magnitude factor.  Separating them --- as
in~\eqref{eq:dyn-var-mech} --- resolves the residual.

\begin{figure}[htbp]
\centering
\includegraphics[width=0.9\linewidth]{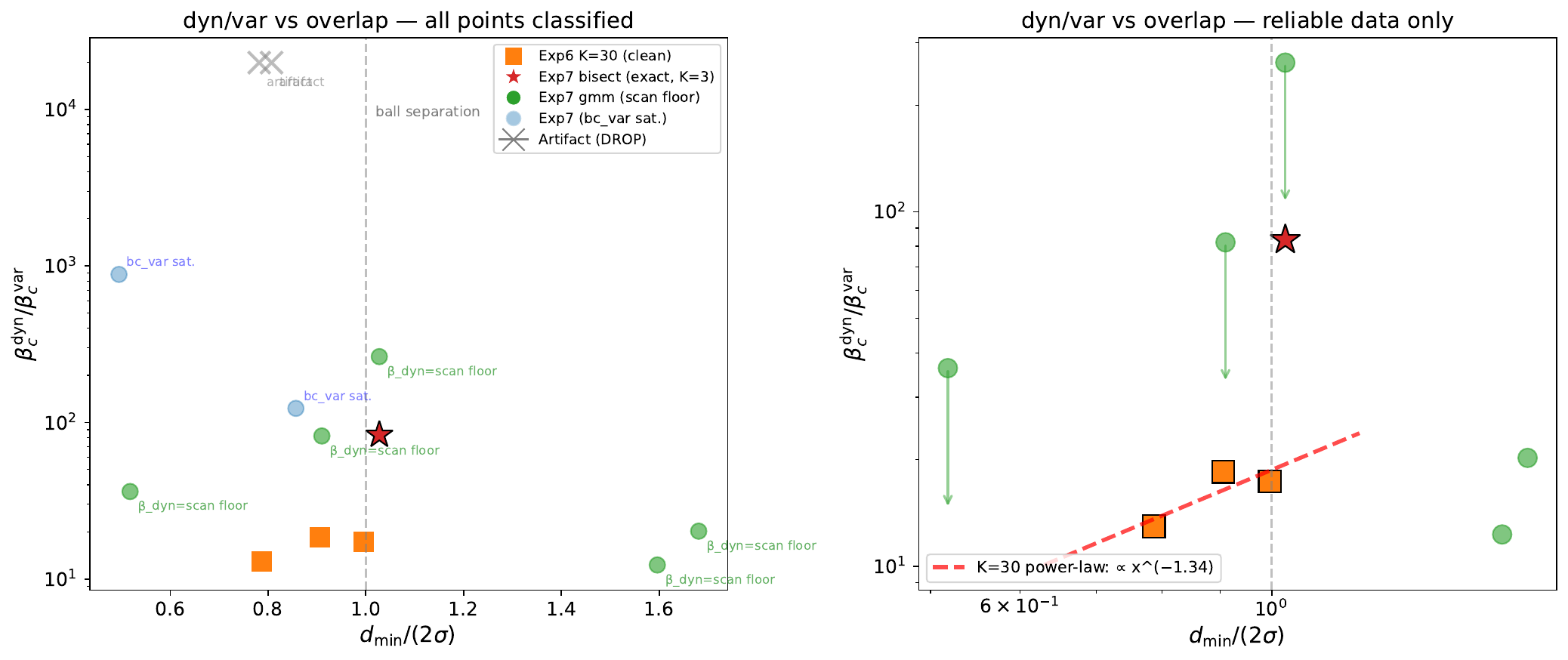}
\caption{The bridge mechanism: $\mathrm{dyn}/\mathrm{var}$ vs.\
$d_{\min}/(2\sigma_0)$ across all experiments.  Clean data points
(green) show a decreasing trend: well-separated modes give
$\mathrm{dyn}/\mathrm{var} \approx 10$--$40\times$.  The Exp7 bisect
point (orange star, boundary-exact) gives $\approx 83\times$ at
$d_{\min}/(2\sigma_0) \approx 1.03$.  Artifact points (red X, K=3
geometric saturation / scan floor) are excluded from the trend.
Overlap regimes ($d_{\min}/(2\sigma_0) < 1$) show much larger
amplification, explaining the $71\times$ $\beta_c$ gap in Exp7's
Gaussian prior.}
\label{fig:bridge}
\end{figure}
\section{Discussion and Open Problems}

The narrative of this paper is \emph{limitation $\to$ condition $\to$
remedy}: the Gaussian prior's structural limitation under RL
fine-tuning (it packs multimodal data into an overlapping latent), the
geometric condition under which the limitation bites
($d_{\min}/(2\sigma_0)<1$), and the remedy (a mode-preserving prior,
necessary precisely in the overlap regime).  The causal
chain~\eqref{eq:causal-chain} that carries this narrative is supported
by closed-form theory (Theorem~\ref{thm:width-limits}),
Exp5A/Exp6/Exp7 with $\ge3$ seeds each, controlled interventions that
move geometry at a fixed prior and switch the prior at fixed geometry
(Appendices~\ref{app:expR},~\ref{app:expY}), and cross-MD robustness
(20\,ns vs.\ 60\,ns give identical $\beta_c$ for matched priors).  The
mechanism material below --- the $\mathrm{dyn}/\mathrm{var}$
decomposition, the corridor exclusion, and the $\alpha(g)$ scaling ---
is the \emph{quantitative description of the limitation}, collected as
supporting results: where the amplification lives (policy side), what
it is not (solver, anchor, saddle routing), and how deep the overlap
must be before it fires.  We fix terminology, state the boundary
condition on the remedy, give three implications, and end with the
genuinely open question.

\paragraph{Terminology: collapse position versus coverage plateau.}
``Resistance to collapse'' is ambiguous unless the two coordinates of
the $\beta$--coverage curve are kept separate.  The \emph{position}
$\beta_c$ (coverage first reaches $0.5$) is prior-independent in
well-separated geometry (Exp5A: both priors $\beta_c{=}100$; Exp6:
$1.23\times$ apart).  The unconditional multi-modal advantage is the
\emph{plateau} --- modes still recoverable at $\beta\gg\beta_c$
($1.00$ vs.\ $0.70$ in Exp5A; $0.73$ vs.\ $0.58$ decoded in Exp6).  The
Exp7 position gap ($0.224$ vs.\ $3.16\times10^{-3}$) is no exception:
it appears only because the two priors induce different geometries
($d_{\min}/(2\sigma_0)=0.81$ vs.\ $1.03$), vanishes in seed~2 where the
Gaussian geometry also overlaps ($0.99$, gap $\le1.4\times$), and is
thus the indirect chain prior$\to$ELBO$\to$geometry$\to$amplification
$\to$$\beta_c$, not a direct prior effect.  Below, ``lower $\beta_c$''
names the position and ``higher recovery'' the plateau.

\paragraph{Boundary condition on the remedy.}
The two benefits have different limits, and stating both is what keeps
the claim from reading as ``GMM always wins''.  (i)~The \emph{position}
gap is a finite-resource effect: its carrier is policy-side error under
overlap (zero solver contribution, bounded anchor;
Appendix~\ref{app:expQ}), so a sufficiently expressive, fully optimised
policy should drive $\beta_c^{\rm dyn}/\beta_c^{\rm var}\to1$ and close
the position gap toward the variational order-1 prediction ($0.69\times$
in Exp7) at fixed geometry.  Policy capacity/budget scaling is the
direct falsification experiment.  (ii)~The \emph{plateau} gap lives one
stage earlier, in the VAE representation itself: strong KL pressure on
a Gaussian prior removes modes at the decoder even before RL
fine-tuning (the $w_{\rm kl}{=}3$ cell of Appendix~\ref{app:expR}
decodes only $0.90$ of modes pre-fine-tuning at latent coverage $1.0$,
and its post-fine-tuning plateau is $0.80$), so the plateau advantage
does not vanish with a larger RL policy --- it requires changing the
prior or the VAE objective.  What is unconditional is therefore the plateau;
what is finite-resource is the position gap.  A cleaner reverse-KL
solver or deeper policy could erase (i)'s amplification-driven
component without touching (ii).

\paragraph{Implication 1: the advantage is geometric, not symbolic.}
The benefit is not the mixture symbol but the ELBO's preservation of
the data's mode separation; any prior producing that geometry should
behave likewise, so the GMM is one representative of a broader class of
mode-separating priors \citep{tomczak2018vae, lee2021mixture}.  The
four-prior comparison (Appendix~\ref{app:expV}) supports the conditional
statement (emp/var within $0.43$--$1.24$ of each cell's own-geometry
variational boundary across Gaussian, GMM, VampPrior and GMVAE, all
plateaus full) while showing the symbol is not sufficient under the
same budget: the label-free mixtures do not reach the tight GMM
geometry (VampPrior $d_{\min}/(2\sigma)=1.7$, GMVAE $1.1$ with
components $0.67d$ off-mode).  The remaining item is the ELBO recipe
that gets a label-free mixture there, not the conditional claim.

\paragraph{Implication 2: geometry is the mediator; the reverse cell
is measured in Appendix~\ref{app:expY}.}
$d_{\min}/(2\sigma_0)$ cannot be enlarged independently of the prior in
our pipeline --- changing the prior, or its KL pressure, is how the
geometry is changed.  The forward intervention holds the Gaussian prior
fixed and moves geometry from $g=2.08$ to $0.65$ by KL weight alone,
with $\beta_c^{\rm var}$ moving $3.21\to0.77$ and the $w_{\rm kl}{=}1$
cell reproducing the Exp6 Gaussian endpoint (Appendix~\ref{app:expR}).
The reverse intervention switches the prior label at matched geometry
over the reachable window $g\in[1.1,2.4]$
(Appendix~\ref{app:expY}): once each cell's own-geometry variational
boundary is controlled for, the prior label has \emph{no} direct effect
on the measured boundary ($0.995\times$, 95\% CI $[0.79,1.25]$,
$t=-0.04$, $R^2=0.976$; $n=39$) --- geometry, in the complete form
$\beta_c^{\rm var}$ rather than the scalar $g$, is the sufficient
statistic, so the mediation is full.  Both endpoints --- the tight GMM
geometry $g=3.34$ and the compressed Gaussian tail $g\approx0.65$ ---
are unreachable from the opposite prior, a bidirectional limit we
report rather than bridge with synthetic data.

\paragraph{Implication 3: the crossover scale predicts the sign, not
the magnitude.}
Theorem~\ref{thm:width-limits} predicts the sign of the small
matched-width boundary difference, and in well-separated geometry the
two boundaries coincide to resolution (Exp5A: identical; Exp6:
$1.23\times$); the large overlap-regime magnitudes ($71\times$ Exp7,
$3$--$6\times$ Exp6) are policy-side amplification against an order-1
variational prediction ($0.69\times$).  The framework is thus a
qualitative position guide that is silent about the plateau, with
quantitative magnitudes requiring the mechanism below.

\paragraph{\textcolor{red}{Open Problem 1.} Characterising $\mathrm{dyn}/\mathrm{var}$ as a function of geometry.}
The functional form of $\mathrm{dyn}/\mathrm{var}$ as a function of
$(d_{\min}/(2\sigma_0), K)$ remains open; the first remaining gap below
states the empirical findings and the required derivation (the mechanism
--- score-network approximation error compounding near saddle-point
manifolds --- is localised to the policy side by the three-source
decomposition of Appendix~\ref{app:expQ}).

\paragraph{\textcolor{orange}{Open Problem 2 (resolved).} The sources of dynamic amplification.}
The $\mathrm{dyn}/\mathrm{var}$ ratio conflates three error sources:
policy-network approximation and optimisation, ODE discretisation, and
KL anchor mismatch.  Appendix~\ref{app:expQ} isolates all three on
eight geometries: the solver contributes exactly zero at the boundary,
anchor mismatch is bounded and moves the boundary non-systematically,
and the residual $15$--$78\%$ is policy-side.  The carrier is a static
representational gap, not a trajectory mechanism and not the anchor.

\paragraph{\textcolor{orange}{Open Problem 3 (resolved).} Regime of
validity and the meaning of $\alpha\approx1$:
$d_{\min}\gtrsim4\sigma_0$.}
Appendix~\ref{app:expW} measures $\alpha(g)=\beta_c/\beta_{\rm cross}$
on 1D lattices with $K\in\{3,5,10\}$: the $d_{\min}>4\sigma_0$ rule is
the correct single-number guide for the minimal $K{=}3$ layout, while
the general statement carries a $j^2$ prefactor on the
$2j^2g^2/\!\ln2$ saddle scale.  Ball-saturated cells remain censored;
a continuous overlap-integral metric for that regime is still missing.

\paragraph{Remaining gaps.}
Four gaps remain, in order of priority:
\begin{enumerate}
  \item \textbf{Closed form for $\mathrm{dyn}/\mathrm{var}$ as a function
  of geometry.}  Appendix~\ref{app:expW} measured the $2j^2g^2/\!\ln2$
  envelope on 1D lattices; connecting score-network approximation bounds
  to the saddle geometry of a tilted GMM is the remaining derivation.
  \item \textbf{Theory of the overlap regime.}  Below
  $d_{\min}<4\sigma_0$ the ball-coverage metric saturates and the
  variational bound degenerates; a continuous overlap-integral metric
  (and the corresponding variational treatment) is missing.
  \item \textbf{Forward- versus reverse-KL offset.}  The
  Appendix~\ref{app:fwdkl} data are complete but unexplained (resolved
  $\lambda$-slope $1.64$ vs.\ $0.96$; the fwd/rev ratio grows with
  dimension); a derivation in the Gibbs-tilt framework is the next
  theoretical step.
  \item \textbf{Scaling up.}  Longer MD (100\,ns, $K\ge5$ basins), a
  soft-coverage metric for the overlap regime, more seeds, latent
  dimensions and molecules, an ELBO recipe that lets label-free learned
  mixtures (VampPrior, GMVAE) reach the tight geometry, and integration
  with DDPO/DPOK.
\end{enumerate}

\section*{Outlook: Applications to molecular generative models}

The mechanism we identify --- that multi-modal priors preserve latent
mode separation under ELBO training, raising the coverage plateau and,
in overlap-prone geometries only, lowering the empirical collapse
boundary --- has direct
implications for molecular generative models trained on equilibrium
(Boltzmann) distributions.  Standard diffusion training on molecular
trajectories does not reweight by Boltzmann factors, so rare
conformations are underrepresented; QBM-VAE \citep{noe2021machine,
ding2023equivariant} and related methods use VAE encoders to define
collective variables for enhanced sampling.  Exp7 established the
geometry on alanine-dipeptide's three conformational basins, and
Experiment~X (Appendix~\ref{app:expX}) now runs the dedicated
rare-basin reward test on the same real MD trajectory: directing a
reverse-KL reward at the rare C5 basin ($w=0.157$), Gaussian-prior
guidance reaches $\approx4\times$ enrichment but erases one of the two
common basins, whereas the tight label-tracked GMM reaches
$2.9$--$3.6\times$ while retaining every common basin above threshold
(3 seeds).  The causal chain \eqref{eq:causal-chain} thus extends to
physical conformations in a refined form: prior geometry chooses a
point on the enrichment--retention Pareto frontier --- sharp control
with destruction versus softer, basin-preserving enrichment, the
latter mediated by the decoder's amplification of a diffuse latent
``collapse''.  Two caveats bound the application claim: the preserving
mechanism relies on the small 2D-latent decoder being non-injective on
the GMM's tight components, and $3\times$ enrichment is physically
unreachable for the two common basins ($1/w\le3$) for either prior.
The ongoing 100\,ns MD experiment, which may produce $K\ge5$ physical
basins, together with higher-capacity decoders and additional small
molecules, is the natural next test of whether the $j^2$ $K$-scaling
of Appendix~\ref{app:expW} and the Pareto trade-off of
Appendix~\ref{app:expX} survive realistic settings.

\section{Conclusion}

The Gaussian prior's failure under RL fine-tuning is geometric.  Its
ELBO optimum compresses multimodal data into an overlapping latent,
and KL-anchored fine-tuning amplifies that overlap into a collapse
boundary far above the variational scale.  The separation-to-width
ratio is the single mediator of this effect: it is moved by the prior
and by nothing else we could isolate.

The claim is conditional.  A non-Gaussian prior is necessary precisely
when the latent geometry would otherwise fall into the overlap regime
and the downstream task is coverage-sensitive; outside that regime the
Gaussian default suffices.  Within it, two gaps must be kept apart:
the boundary position is a finite-resource pipeline effect, while the
coverage plateau is a representational property of the VAE itself.

The practical upshot is that prior choice is a geometric lever that can
be checked before training, and the conceptual upshot is that the KL
term does not merely ignore mode structure --- it prices it.

\appendix

\section{Extended 60\,ns MD basin analysis}
\label{app:basins}

This appendix documents the physical basins found by periodic KDE on the
$(\phi, \psi)$ backbone torsions of the 60\,ns alanine-dipeptide
trajectory and compares them peak-by-peak against the 20\,ns trajectory
used in the main text.  All numbers are produced by
\texttt{torus.TorusKDE2D} / \texttt{\_build\_exp7\_features.py} from the
stored torsion arrays in
\texttt{results/exp7\_60ns/alanine\_features.npz}
(60{,}626 frames; 20\,ns: 40{,}000 frames,
\texttt{results/exp7/alanine\_features.npz}).

\paragraph{Protocol.}
The KDE is an FFT-based periodic Gaussian density on the 2-torus
$[-\pi,\pi)^2$ on a $200\times200$ grid.  The bandwidth is selected
automatically as $h = \max(n^{-1/6}\,\overline{\sigma}_{\mathrm{circ}},
0.08)$\,rad: $h = 0.185$\,rad ($10.6^\circ$) for 20\,ns and
$h = 0.175$\,rad ($10.0^\circ$) for 60\,ns.  Peaks are local density
maxima above $0.08\,p_{\max}$ separated by at least $0.6$\,rad
(\texttt{find\_basins}: \texttt{rel\_thr=0.08, min\_sep=0.6, max\_k=8}),
and each peak's mass is the integrated KDE mass in its torus-Voronoi
cell.  Training labels assign each MD frame to the nearest detected
peak within a $0.8$\,rad geodesic cap; frames outside all caps are
left unassigned.

\paragraph{Basin count is robust to the mass threshold.}
Re-running peak detection with the minimum-mass threshold lowered from
$5\%$ to $0.5\%$ returns exactly three basins in \emph{both}
trajectories (Table~\ref{tab:thr-robust}); no fourth local maximum
carries even $0.5\%$ of KDE mass.  The canonical C7ax and $\alpha_L$
conformations receive zero frames within $0.6$\,rad of their reference
centres in either trajectory.

\begin{table}[h]
\centering
\caption{Number of TorusKDE basins detected as a function of the
minimum-mass threshold.  $K=3$ in both trajectories at every threshold
down to $0.5\%$.}
\label{tab:thr-robust}
\footnotesize
\begin{tabular}{@{}lcccc@{}}
\toprule
minimum KDE mass & $5\%$ & $2\%$ & $1\%$ & $0.5\%$ \\
\midrule
20\,ns ($n=40{,}000$) & 3 & 3 & 3 & 3 \\
60\,ns ($n=60{,}626$) & 3 & 3 & 3 & 3 \\
\bottomrule
\end{tabular}
\end{table}

\paragraph{Peak positions and masses.}
Table~\ref{tab:peaks} lists the three detected peaks with their
$(\phi,\psi)$ coordinates and both mass definitions: \emph{KDE mass}
(the integrated Voronoi-cell mass of the smoothed density, sums to
$100\%$ over the three basins) and \emph{frame fraction} (the fraction
of raw MD frames assigned to that peak inside the $0.8$\,rad cap;
$7.0\%$--$7.1\%$ of frames lie outside all caps).

\begin{table}[h]
\centering
\caption{Physical basins of the two MD trajectories.  All coordinates
in degrees.  ``KDE'' = integrated TorusKDE mass; ``frame'' = fraction
of raw frames within the $0.8$\,rad labelling cap.  The same three
peak locations are found in both trajectories (maximal peak shift
$1.8^\circ$); the rarest basin is the \emph{same} $\beta$/C5 extended
conformation at $(-146.7^\circ, +157.5^\circ)$ in both, only $8.1^\circ$
from the canonical C5 reference $(-150.1^\circ,+150.1^\circ)$.}
\label{tab:peaks}
\footnotesize
\begin{tabular}{@{}lcc@{\hspace{3mm}}cc@{\hspace{3mm}}cc@{}}
\toprule
 & \multicolumn{2}{c}{$(\phi,\psi)$ peak} &
   \multicolumn{2}{c}{KDE mass} &
   \multicolumn{2}{c}{frame fraction} \\
Basin & 20\,ns & 60\,ns & 20\,ns & 60\,ns & 20\,ns & 60\,ns \\
\midrule
C7eq (dominant)
 & $(-69.3,152.1)$ & $(-71.1,152.1)$
 & 44.52\% & 44.59\% & 43.46\% & 43.39\% \\
$\alpha_R$ (secondary)
 & $(-72.9,-17.1)$ & $(-72.9,-17.1)$
 & 38.87\% & 40.20\% & 33.82\% & 35.16\% \\
$\beta$/C5 (rarest)
 & $(-146.7,157.5)$ & $(-146.7,157.5)$
 & 16.61\% & 15.21\% & 15.71\% & 14.40\% \\
\midrule
unassigned ($>0.8$\,rad from any peak)
 & \multicolumn{2}{c}{---} & \multicolumn{2}{c}{---}
 & 7.01\% & 7.05\% \\
\bottomrule
\end{tabular}
\end{table}

The three peaks agree to within $1.8^\circ$ across the two independent
trajectories, and the rarest basin is not a new conformation:
$(-146.7^\circ,+157.5^\circ)$ lies $8.1^\circ$ (torus geodesic) from
the canonical $\beta$/C5 reference $(-150.1^\circ,+150.1^\circ)$
(\texttt{alanine\_data.py}), and the RL target in both experiments is
this same basin.  Tripling the trajectory length shifts occupancies
only mildly ($\alpha_R$ frame fraction $+1.3$ percentage points,
$\beta$/C5 $-1.3$ points; KDE masses $+1.3$ and $-1.4$ points) and
does not reveal any additional populated conformation (C7ax and
$\alpha_L$ remain unvisited: $0/40{,}000$ and $0/60{,}626$ frames
within $0.6$\,rad of their references).  This is the physical basis
for the cross-MD robustness of $\beta_c$ in the main text: the
fine-tuning boundary is unchanged even though the longer trajectory
samples every basin more thoroughly.

\section{Full per-seed Exp7 numbers (20\,ns and 60\,ns)}
\label{app:exp7seeds}

Table~\ref{tab:exp7full} gives the complete per-seed record underlying
the cross-MD robustness paragraph and Table~\ref{tab:exp7}: VAE latent
geometry ($\sigma_{\mathrm{lat}}$, $d_{\min}/\sigma$, basin
classification accuracy) and the fine-tuning boundaries on the latent
and decoded spaces, for both MD lengths and both priors.  Entries
marked $\le\beta_0$ stayed above the $0.5$ recovery threshold at all
scan points, i.e.\ the boundary is below that scan's floor (``below''
flag in the JSONs); the exact value for the 20\,ns seed-0 GMM run is
$3.16\times10^{-3}$ from the dedicated 16-point bisect
(\texttt{results/exp7\_s0\_bisect/gmm.json}).  The 60\,ns GMM runs use
the $\beta\ge0.01$ grid (floor $\le 0.01$); the 60\,ns Gaussian
seed-2 run uses the $\beta\ge0.1$ grid (floor $\le 0.1$).

\begin{table}[h]
\centering
\caption{Complete Exp7 record, 3 seeds $\times$ 2 priors $\times$ 2 MD
lengths.  \texttt{cls} = nearest-mean basin classification accuracy of
the latent codes; $\beta_c^{\mathrm{lat}}$/$\beta_c^{\mathrm{dec}}$ =
latent/decoded collapse boundaries ($0.5$ coverage threshold,
log-interpolated when exact); plat$_{\mathrm{lat}}$/plat$_{\mathrm{dec}}$
= mean coverage at the two highest-$\beta$ scan points.  Source:
\texttt{results/exp7\{,\_s1,\_s2\}/\{gaussian,gmm\}.json} and
\texttt{results/exp7\_60ns\_s\{0,1,2\}\_\{gaussian,gmm\}.json}.}
\label{tab:exp7full}
\footnotesize
\begin{tabular}{@{}llccccccccc@{}}
\toprule
MD & seed & prior & $\sigma_{\mathrm{lat}}$ & $d_{\min}/\sigma$ & cls
 & $\beta_c^{\mathrm{lat}}$ & $\beta_c^{\mathrm{dec}}$
 & plat$_{\mathrm{lat}}$ & plat$_{\mathrm{dec}}$ \\
\midrule
20\,ns & 0 & Gaussian & 0.958 & 1.61 & 0.875 & 0.224 & $\le0.01$ & 1.00 & 1.00 \\
20\,ns & 0 & GMM      & 0.027 & 2.06 & 0.919 & $\le0.01$ & $\le0.01$ & 1.00 & 1.00 \\
20\,ns & 1 & Gaussian & 0.955 & 1.56 & 0.875 & 0.224 & 0.014 & 1.00 & 0.67 \\
20\,ns & 1 & GMM      & 0.027 & 1.82 & 0.911 & $\le0.01$ & $\le0.01$ & 1.00 & 1.00 \\
20\,ns & 2 & Gaussian & 1.140 & 0.99 & 0.898 & 0.014 & $\le0.01$ & 1.00 & 1.00 \\
20\,ns & 2 & GMM      & 0.030 & 1.71 & 0.917 & $\le0.01$ & $\le0.01$ & 1.00 & 1.00 \\
60\,ns & 0 & Gaussian & 0.974 & 1.55 & 0.868 & 0.224 & $\le0.1$ & 1.00 & 0.67 \\
60\,ns & 0 & GMM      & 0.027 & 2.16 & 0.935 & $\le0.01$ & $\le0.01$ & 0.67 & 0.67 \\
60\,ns & 1 & Gaussian & 0.966 & 1.60 & 0.874 & 0.224 & $\le0.1$ & 1.00 & 0.67 \\
60\,ns & 1 & GMM      & 0.027 & 2.14 & 0.939 & $\le0.01$ & $\le0.01$ & 1.00 & 0.67 \\
60\,ns & 2 & Gaussian & 1.111 & 1.04 & 0.888 & $\le0.1$ & $\le0.1$ & 1.00 & 1.00 \\
60\,ns & 2 & GMM      & 0.029 & 1.73 & 0.912 & $\le0.01$ & $\le0.01$ & 1.00 & 0.67 \\
\bottomrule
\end{tabular}
\end{table}

\begin{figure}[h]
\centering
\includegraphics[width=\linewidth]{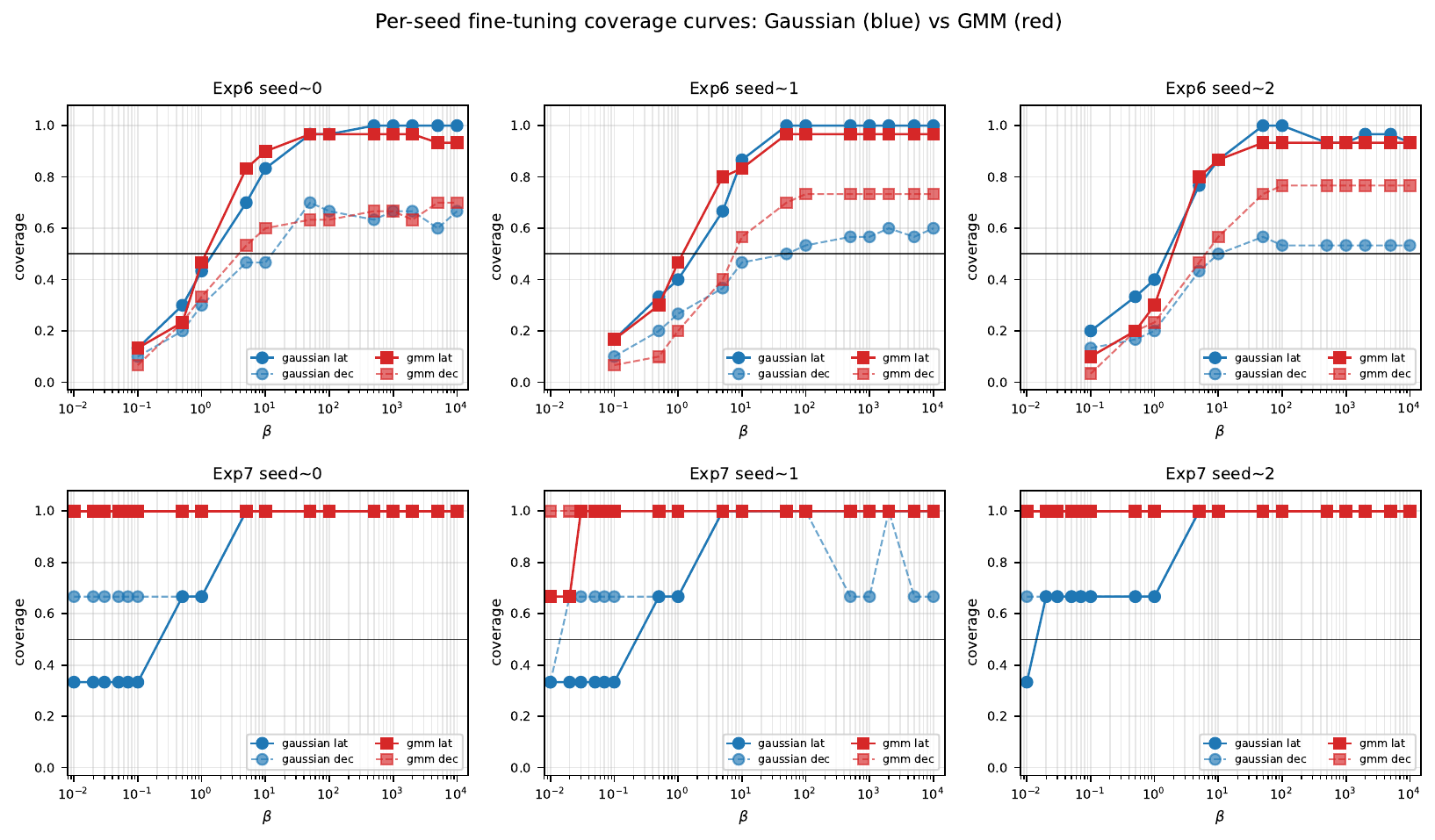}
\caption{Per-seed fine-tuning coverage curves for Exp6 (K=30, top row)
and Exp7 (alanine, bottom row).  Solid lines with markers show latent
coverage; dashed lines show decoded coverage.  Blue circles = Gaussian
prior, red squares = GMM prior.  The horizontal line marks the 0.5
coverage threshold used to define $\beta_c$.  The Gaussian/GMM direction
never flips across any seed; the Gaussian seed-2 anomaly (wider latent,
$d_{\min}/\sigma\approx1.0$) appears under both 20\,ns and 60\,ns MD
(Table~\ref{tab:exp7full}), confirming it is VAE-initialisation-driven.}
\label{fig:per-seed-curves}
\end{figure}

The latent geometry is stable across MD lengths within each prior:
Gaussian seeds~0/1 keep $\sigma_{\mathrm{lat}}\approx0.96$--$0.97$ and
$d_{\min}/\sigma\approx1.55$--$1.61$; GMM seeds keep
$\sigma_{\mathrm{lat}}\approx0.027$ and
$d_{\min}/\sigma=1.7$--$2.2$.  The Gaussian seed-2 anomaly
(over-trained, $\sigma_{\mathrm{lat}}\approx1.1$,
$d_{\min}/\sigma\approx1.0$) reproduces under both trajectories,
confirming it is VAE-initialisation-driven rather than data-driven.  In
all 12 runs the GMM-prior boundary stays at least $22\times$ below the
Gaussian-prior boundary of the matching seed whenever the latter is
resolved, and the two Gaussian seeds~0/1 return the identical exact
value $0.224$ under both MD lengths.

\section{Forward-KL collapse boundary: full scaling record}
\label{app:fwdkl}

This appendix collects the complete reverse-/forward-KL boundary data
behind the forward-KL paragraph of Section~4: the four-$\lambda$
alanine-torus sweep and the OLS power-law fits in dimensions $1$--$3$.

\paragraph{Protocol (torus sweep).}
On the same 20\,ns alanine TorusKDE $p_0$ used in Exp7 ($K=3$ basins,
masses $0.445/0.389/0.166$, quadratic reward targeting the rarest
basin), $\beta_c$ is measured for \emph{both} KL directions at
$\lambda\in\{0.1,1,10,100\}$ with identical $\beta$ scans spanning
$10^{-6}$--$10^8$ and seven log-bisection steps at the $0.5$ coverage
criterion.  The empirical boundary comes from fine-tuning the trained
torus score network (800 fine-tune steps) and measuring ODE
coverage; the theory boundary is the exact periodic-grid reweighting
computed on the $200\times200$ KDE grid.  Each run records both
boundaries with flags, both predictions, and the full coverage curves
(9 empirical points; 42 theory points).

\begin{table}[htbp]
\centering
\caption{Reverse- and forward-KL collapse boundaries on the alanine
torus at four reward temperatures.  Empirical columns carry the
bisect flag; the $\lambda=0.1$ forward boundary is censored
(``below''): coverage is still $0.67$ at the scan floor $10^{-6}$.
The theory predicts a fixed fwd/rev ratio $1.24$ at every $\lambda$,
whereas the empirical ratio is non-monotone (censored, then $0.17$,
then $2.7$--$3.5$).  At $\lambda=1$, forward-directed fine-tuning
retains only $0.33$ target-basin coverage at the reverse collapse
boundary (\texttt{cov\_fwd\_at\_rev\_bc}).  Source:
\texttt{results/remote\_bundle\_v2/boundary\_l\{0.1,1.0,10.0,100.0\}.json}.}
\label{tab:fwd-lambda}
\footnotesize
\begin{tabular}{@{}rrrrrrr@{}}
\toprule
 & \multicolumn{3}{c}{empirical $\beta_c$}
 & \multicolumn{3}{c}{numeric theory $\beta_c$} \\
\cmidrule(lr){2-4}\cmidrule(lr){5-7}
$\lambda$ & reverse & forward & fwd/rev & reverse & forward & fwd/rev \\
\midrule
0.1   & 0.122 & $<10^{-6}$ & ---   & 0.0493 & 0.0612 & 1.24 \\
1     & 0.914 & 0.157      & 0.17  & 0.493  & 0.612  & 1.24 \\
10    & 11.34 & 31.06      & 2.74  & 4.93   & 6.12   & 1.24 \\
100   & 85.05 & 299.6      & 3.52  & 49.25  & 61.19  & 1.24 \\
\bottomrule
\end{tabular}
\end{table}

\paragraph{Resolved log-log slopes.}
OLS fits of $\log_{10}\beta_c$ against $\log_{10}\lambda$ give slope
$0.96$ for reverse-KL on all four exact points, but only $1.64$ for
forward-KL on the three \emph{resolved} points
($\lambda=1,10,100$).  A slope of $2.77$ for forward-KL results only
if the censored $\lambda=0.1$ entry is substituted by the scan floor
$10^{-6}$ as though it were measured; that number is an upper-bound
artefact (the true boundary lies strictly below the floor), so the
resolved slope $1.64$ is the defensible figure and a single power law
is not established across the full range.

\paragraph{Power-law fits by dimension.}
On the synthetic GMM grids the OLS specification $\log\beta_c
= c + p_\lambda\log\lambda + p_d\log(d/\sigma) + p_K\log K$ is fit
separately per dimension (Table~\ref{tab:fwd-scaling}); only
configurations with resolved boundaries enter.

\begin{table}[htbp]
\centering
\caption{OLS scaling exponents for empirical reverse-KL, empirical
forward-KL, and numeric-theory boundaries, by ambient dimension.
$n$ = number of resolved configurations (of $64$ in 1D, $18$ in 2D/3D).
Reverse-KL exponents track the theory closely with
$p_\lambda\approx1$ and $R^2\ge0.85$.  Forward-KL fits are poor
($R^2=0.55$--$0.80$) with $p_\lambda\le0.79$ and \emph{negative}
$p_K$, i.e.\ the forward boundary is censored/non-monotone over much
of the grid and is not described by a clean power law.  Source:
\texttt{results/remote\_bundle\_v2/scaling\_by\_dim.json}.}
\label{tab:fwd-scaling}
\footnotesize
\begin{tabular}{@{}rrccccc@{}}
\toprule
dim & boundary & $n$ & $p_\lambda$ & $p_d$ & $p_K$ & $R^2$ \\
\midrule
1 & empirical reverse & 48 & 1.013 & $ 2.033$ & $ 2.311$ & 0.994 \\
1 & empirical forward & 48 & 0.777 & $-0.788$ & $-0.448$ & 0.802 \\
1 & numeric theory    & 44 & 1.000 & $ 2.484$ & $ 2.180$ & 0.995 \\
2 & empirical reverse & 18 & 1.016 & $ 2.739$ & $ 0.833$ & 0.992 \\
2 & empirical forward & 18 & 0.679 & $ 1.424$ & $-1.711$ & 0.545 \\
2 & numeric theory    & 18 & 1.000 & $ 2.764$ & $ 0.886$ & 0.999 \\
3 & empirical reverse & 18 & 0.998 & $ 3.314$ & $ 0.681$ & 0.850 \\
3 & empirical forward & 18 & 0.793 & $ 3.205$ & $-1.181$ & 0.597 \\
3 & numeric theory    & 18 & 1.000 & $ 2.753$ & $ 0.618$ & 0.951 \\
\bottomrule
\end{tabular}
\end{table}

The empirical fwd/rev $\beta_c$ ratio (median, IQR over resolved
configurations) changes sign of its asymmetry with dimension:
$0.016\;[0.003,0.090]$ in 1D, $0.10\;[0.029,1.96]$ in 2D, and
$1.13\;[0.109,1.42]$ in 3D.  Forward-KL thus has dramatically smaller
empirical $\beta_c$ than reverse-KL on 1D grids, comparable boundaries
in 2D, and slightly larger $\beta_c$ in 3D --- an estimator/geometry
effect that
the torus numeric theory, whose fwd/rev ratio is the constant $1.24$,
does not predict.

\section{Regional score-MSE decomposition across dimensions}
\label{app:saddlemse}

This appendix gives the full regional score-MSE record for the
36 ExpA configurations: a 12-point 2D width sweep and a 24-point
$2$--$5$D lattice sweep.

\paragraph{Protocol.}
All models are the same 2-layer, 256-hidden score MLP trained for
$40{,}000$ steps (single seed~0).  \texttt{metrics.score\_mse\_nd}
draws $n$ points uniformly over the bounding box of the mode centres
extended by $\pm2.5\sigma$ ($n=20{,}000$ for the $2$/$3$D lattice
runs; $n=10{,}000$ for the $4$/$5$D and width-sweep runs) and assigns
each point to a region: \emph{mode} ($<1.5\sigma$ of any centre),
\emph{saddle} ($<0.25\,\hat d$ of a nearest-neighbour-pair midpoint,
pairs within $1.6\hat d$, excluding mode points), and \emph{valley}
(everything else).  MSE against the analytically blurred target score
is averaged per region at $t\in\{0.02,0.05,0.1\}$.  A density-weighted
``MC'' MSE is evaluated on $40{,}000$ forward-noised $p_0$ samples.

\paragraph{Width sweep (2D, $d=5$, $K\in\{5,10,20\}$).}
Table~\ref{tab:saddle-sigma} reports the $t=0.02$ values.  Note the
geometric regime is labelled by $\sigma/d$: $\sigma=0.25$ means
\emph{narrow}, well-separated modes ($d/(2\sigma)=10$), not an
overlap regime.  The saddle/mode ratio median over $K$ is
$272\times$ at $\sigma/d=0.05$, $35\times$ at $0.1$, $2.2\times$ at
$0.2$, and $1.0\times$ at $0.4$ --- a fall of $\sim270\times$,
roughly $\sigma^{-2.7}$.  Crucially the density-weighted MC MSE
tracks the \emph{mode} MSE even at the $603\times$ point
($0.452$ vs.\ $0.418$): the saddle excess lives in regions of
essentially zero stationary mass, and all 12 models reach full
target coverage ($\mathrm{cover}_{0.5}=1.0$).

\begin{table}[htbp]
\centering
\caption{2D width sweep at $t=0.02$ ($d=5$ fixed; 40k-step seed-0
models): per-region score MSE, density-weighted MC MSE, and
saddle/mode ratio.  The gap is maximal for \emph{narrow} modes
($\sigma/d=0.05$, $d/(2\sigma)=10$) and closes as modes broaden;
the MC MSE stays within $1.4\times$ of the mode MSE throughout.
Source: \texttt{results/expA\_learn/2d\_K*\_sig*\_s0.json}, field
\texttt{score\_mse.t0.02}.}
\label{tab:saddle-sigma}
\footnotesize
\begin{tabular}{@{}rrrrrrrr@{}}
\toprule
$K$ & $\sigma$ & $d/(2\sigma)$ & mode & saddle & valley & MC & sad/mode \\
\midrule
 5 & 0.25 & 10.0 & 0.418 & 251.82 & 145.74 & 0.452 & 603.1 \\
10 & 0.25 & 10.0 & 1.119 & 143.80 &  61.78 & 1.257 & 128.6 \\
20 & 0.25 & 10.0 & 0.831 & 225.94 & 130.65 & 0.854 & 271.9 \\
 5 & 0.50 &  5.0 & 0.221 &  15.82 &  11.40 & 0.263 &  71.6 \\
10 & 0.50 &  5.0 & 0.366 &   8.37 &   8.35 & 0.399 &  22.9 \\
20 & 0.50 &  5.0 & 0.387 &  13.55 &   7.95 & 0.423 &  35.0 \\
 5 & 1.00 &  2.5 & 0.096 &   0.299 &  1.216 & 0.127 &   3.1 \\
10 & 1.00 &  2.5 & 0.197 &   0.242 &  0.784 & 0.222 &   1.2 \\
20 & 1.00 &  2.5 & 0.164 &   0.357 &  0.584 & 0.191 &   2.2 \\
 5 & 2.00 &  1.25 & 0.014 &   0.014 &  0.167 & 0.020 &   1.0 \\
10 & 2.00 &  1.25 & 0.034 &   0.021 &  0.150 & 0.036 &   0.6 \\
20 & 2.00 &  1.25 & 0.056 &   0.056 &  0.129 & 0.058 &   1.0 \\
\bottomrule
\end{tabular}
\end{table}

\paragraph{Lattice sweep in $2$--$5$D ($\sigma=1$).}
Table~\ref{tab:saddle-dim} shows the same decomposition for lattice
packs with $K\in\{5,10,20\}$ and $d\in\{3,5\}$.  Ratios are modest
everywhere ($0.5$--$5.3\times$); the controlling variable is the
spacing/width ratio $d/\sigma\in\{3,5\}$, not ambient dimension
(per-dimension medians: $1.2\times$ in 2D, $2.1\times$ in 3D,
$1.2\times$ in 4D, $1.4\times$ in 5D; all $d=3$ rows lie at
$0.5$--$1.1\times$).  Valley errors remain $0.2$--$2.6$, and the MC
MSE again stays close to the mode MSE.

\begin{table}[htbp]
\centering
\caption{Lattice packs $2$--$5$D, $t=0.02$, $\sigma=1$: per-region
score MSE, saddle/mode ratio, and $\mathrm{cover}_{0.5}$.  The
$5$D models failed to learn (coverage $0.00$--$0.60$), and their
saddle estimates rest on $4$--$30$ evaluation points each
(Table~\ref{tab:saddle-counts}); treat the $5$D block as estimator
diagnostics, not dimension evidence.  Source:
\texttt{results/expA\_learn/\{2,3,4,5\}d\_pack\_*\_s0.json}.}
\label{tab:saddle-dim}
\footnotesize
\begin{tabular}{@{}rrrrrrrrr@{}}
\toprule
dim & $K$ & $d$ & mode & saddle & valley & MC & sad/mode & cover \\
\midrule
2 &  5 & 3 & 0.072 & 0.048 & 0.864 & 0.104 & 0.67 & 1.00 \\
2 &  5 & 5 & 0.098 & 0.301 & 1.184 & 0.127 & 3.08 & 1.00 \\
2 & 10 & 3 & 0.107 & 0.055 & 0.672 & 0.118 & 0.51 & 1.00 \\
2 & 10 & 5 & 0.199 & 0.240 & 0.790 & 0.222 & 1.20 & 1.00 \\
2 & 20 & 3 & 0.117 & 0.132 & 0.335 & 0.121 & 1.13 & 1.00 \\
2 & 20 & 5 & 0.165 & 0.358 & 0.582 & 0.191 & 2.18 & 1.00 \\
3 &  5 & 3 & 0.031 & 0.030 & 0.665 & 0.048 & 0.97 & 1.00 \\
3 &  5 & 5 & 0.047 & 0.249 & 1.963 & 0.080 & 5.30 & 1.00 \\
3 & 10 & 3 & 0.083 & 0.080 & 0.788 & 0.118 & 0.96 & 1.00 \\
3 & 10 & 5 & 0.065 & 0.231 & 1.773 & 0.103 & 3.55 & 1.00 \\
3 & 20 & 3 & 0.064 & 0.062 & 0.204 & 0.075 & 0.97 & 1.00 \\
3 & 20 & 5 & 0.077 & 0.246 & 0.690 & 0.108 & 3.20 & 0.95 \\
4 &  5 & 3 & 0.168 & 0.183 & 0.461 & 0.178 & 1.09 & 1.00 \\
4 &  5 & 5 & 0.155 & 0.322 & 1.766 & 0.186 & 2.08 & 1.00 \\
4 & 10 & 3 & 0.243 & 0.195 & 0.857 & 0.255 & 0.80 & 1.00 \\
4 & 10 & 5 & 0.209 & 0.254 & 1.486 & 0.192 & 1.21 & 0.90 \\
4 & 20 & 3 & 0.229 & 0.253 & 0.577 & 0.257 & 1.10 & 1.00 \\
4 & 20 & 5 & 0.199 & 0.326 & 0.866 & 0.221 & 1.64 & 0.85 \\
5 &  5 & 3 & 0.081 & 0.090 & 0.519 & 0.102 & 1.12 & 0.60 \\
5 &  5 & 5 & 0.076 & 0.201 & 2.151 & 0.104 & 2.65 & 0.00 \\
5 & 10 & 3 & 0.121 & 0.104 & 0.508 & 0.144 & 0.86 & 0.30 \\
5 & 10 & 5 & 0.138 & 0.236 & 2.582 & 0.176 & 1.71 & 0.10 \\
5 & 20 & 3 & 0.095 & 0.082 & 0.757 & 0.117 & 0.86 & 0.50 \\
5 & 20 & 5 & 0.114 & 0.314 & 1.670 & 0.147 & 2.75 & 0.15 \\
\bottomrule
\end{tabular}
\end{table}

\paragraph{Estimator thinness at high dimension.}
Uniform-in-box points fall in saddle strips only with probability
proportional to their codimension-one volume, so saddle counts
collapse as the dimension and the number of modes grow
(Table~\ref{tab:saddle-counts}): the 5D saddle means rest on just
$4$--$30$ of $10{,}000$ points, and the same models cover none to
$60\%$ of their target modes.  The 2D width sweep is also affected
at the broad end ($n_{\mathrm{saddle}}=186$ at $K=5,\sigma=2$,
vs.\ $3087$ at $\sigma=0.25$), though there every model is fully
trained.  This is why the main text's dimensional claim is restricted
to the trained $2$--$4$D range.

\begin{table}[htbp]
\centering
\caption{Evaluation-point counts per region (of $10{,}000$ uniform
points) for the extreme configurations.  Mode and saddle strips
empty out with dimension and $K$; 5D saddle MSEs are means over
$\le30$ points from models that did not learn.  Source:
\texttt{score\_mse.n\_\{mode,saddle,valley\}} fields.}
\label{tab:saddle-counts}
\footnotesize
\begin{tabular}{@{}lrrrr@{}}
\toprule
configuration & $n_{\mathrm{mode}}$ & $n_{\mathrm{saddle}}$
& $n_{\mathrm{valley}}$ & cover \\
\midrule
2D $K{=}5$, $\sigma{=}0.25$ & 170 & 3087 & 6743 & 1.00 \\
2D $K{=}5$, $\sigma{=}2.0$  & 3338 & 186 & 6476 & 1.00 \\
4D $K{=}5$, $d{=}3$          & 307 & 25 & 9668 & 1.00 \\
4D $K{=}20$, $d{=}5$         & 79 & 209 & 9712 & 0.85 \\
5D $K{=}5$, $d{=}3$          & 93 & 5 & 9902 & 0.60 \\
5D $K{=}10$, $d{=}3$         & 35 & 4 & 9961 & 0.30 \\
5D $K{=}20$, $d{=}5$         & 12 & 20 & 9968 & 0.15 \\
5D $K{=}5$, $d{=}5$          & 30 & 30 & 9940 & 0.00 \\
\bottomrule
\end{tabular}
\end{table}

\paragraph{Time dependence and training history.}
The saddle excess is not a $t=0.02$ artefact
(Table~\ref{tab:saddle-time}): for 2D $K{=}5,\sigma=0.25$ the
saddle/mode ratio is $603/1773/648$ at $t=0.02/0.05/0.1$.  In the
failed 5D model the valley MSE stays $\approx1.5$--$2.6$ at all
times, confirming the model (not the region assignment) is the
problem.  Training snapshots of the 2D $K{=}20,d=5$ model show the
saddle/mode ratio already present once modes are learned
($1.44$ at $2$k steps, cover $0.65$; $2.25/1.33/1.79/2.18$ at
$5$/$10$/$20$/$40$k), so the gap is not an over-training artefact.

\begin{table}[htbp]
\centering
\caption{Regional score MSE at all three diffusion times for the
two extreme configurations.  The saddle/mode gap persists across
$t$; density-weighted MC MSE remains an order of magnitude below
the saddle MSE even at $600\times$ mode error.}
\label{tab:saddle-time}
\footnotesize
\begin{tabular}{@{}lrrrrrr@{}}
\toprule
configuration & $t$ & mode & saddle & valley & MC & sad/mode \\
\midrule
2D $K{=}5$, $\sigma{=}0.25$
 & 0.02 & 0.418 & 251.82 & 145.74 & 0.452 & 603 \\
 & 0.05 & 0.070 & 123.72 & 127.77 & 0.076 & 1773 \\
 & 0.10 & 0.048 &  31.25 &  40.36 & 0.045 & 648 \\
\addlinespace
5D $K{=}20$, $d{=}5$ (cover 0.15)
 & 0.02 & 0.114 & 0.314 & 1.670 & 0.147 & 2.8 \\
 & 0.05 & 0.031 & 0.129 & 1.716 & 0.051 & 4.2 \\
 & 0.10 & 0.018 & 0.081 & 1.461 & 0.025 & 4.6 \\
\bottomrule
\end{tabular}
\end{table}

\section{Score error along controlled PF-ODE trajectories}
\label{app:paths}

This appendix reports the direct test of the corridor-routing
hypothesis discussed in Section~5.1: instead of evaluating the score
error on a uniform grid or on reference $p_t$ samples, we integrate
the controlled PF-ODE and accumulate the reference-score error and
the policy deviation \emph{along the trajectories themselves}.
Script: \texttt{expP\_paths.py}; aggregation:
\texttt{analyze\_expP.py}; raw outputs:
\texttt{results/expP\_paths/*.json}.

\paragraph{Protocol.}
The 24 geometries are the 12-point 2D width sweep
($K\in\{5,10,20\}$, $\sigma\in\{0.25,0.5,1,2\}$, $d=5$) and the 2D/3D
lattice packs ($K\in\{5,10,20\}$, $d\in\{3,5\}$, $\sigma=1$); 4D/5D
packs are excluded because their reference checkpoints are not fully
trained (Appendix~\ref{app:saddlemse}).  For each geometry we load the
reference checkpoint and its empirical exact-grid $\beta_c$ from
\texttt{results/expA\_scale}, then fine-tune three policies at
$\beta\in\{0.3,1,3\}\,\beta_c$ with reverse-KL guidance, $\lambda=1$,
$1000$ steps, $S=20$, batch size $256$, learning rate $10^{-4}$,
seed~0.  All four models are integrated from the \emph{same} batch of
$4096$ terminal states $x_T$ (fixed seed), so path divergence is
attributable to the score functions alone.  Integration uses the RK2
PF-ODE drift of \texttt{diffusion.ode\_sample} bit-identically on the
time grid $\mathrm{linspace}(1,10^{-3},201)$, repeated at $480$ steps
as a discretisation control.  At every intermediate state we evaluate
$\Vert\hat s_{\rm ref}-s_*\Vert^2$ and
$\Vert\hat s_\pi-s_*\Vert^2$ against the analytic blurred target
score \texttt{dist.blurred\_score}, and classify the state into
mode/saddle/valley with the exact region rules and thresholds of
\texttt{metrics.score\_mse\_nd}; region labels are recomputed at each
state.  Per-state sums and counts are accumulated by time window and
also retained as a per-time series (no full trajectories need to be
stored; a $1024$-trajectory subsample is saved separately).

\paragraph{Width-sweep results.}
Table~\ref{tab:paths} reports, for the collapsing $\beta=0.3\beta_c$
policies, endpoint coverage, the ratio of saddle-strip residence
times, the reference-score error encountered on reference vs.\
collapsing paths in the late window $t\le0.05$, the corresponding
ratio at $t\in(0.05,0.2)$, and the
ratio of the grid $t=0.02$ saddle MSE to the matched-window path
error.  Key observations: (1)~the grid diagnostic overstates the
path-encountered error by $1.6$--$25\times$ for narrow modes and
understates it by $3$--$28\times$ for broad modes --- the measure
mismatch changes magnitudes in both directions; (2)~residence times
are essentially unchanged by collapse; (3)~only the
$K{=}5,\sigma{=}0.25$ row shows a large error amplification, and its
late-window reference estimate rests on just 14 saddle states (its
$84\times$ mid-window ratio rests on $5{,}118$ vs.\ $41{,}093$ states
and is well populated).

\begin{table}[htbp]
\centering
\caption{Controlled-trajectory test, 2D width sweep.  Columns:
$\beta{=}0.3\beta_c$ endpoint coverage; saddle-residence ratio
$q/r$; reference-score saddle error along reference paths and
collapsing paths, time-averaged over $t\le0.05$; their ratio at
$t\le0.05$ and $t\in(0.05,0.2)$; and grid $t{=}0.02$ saddle MSE
divided by the matched-window path error.  Reference coverage is
$1.00$ in every row.  Source: \texttt{results/expP\_paths/}.}
\label{tab:paths}
\footnotesize
\begin{tabular}{@{}rrrrrrrrr@{}}
\toprule
$K$ & $\sigma$ & $\mathrm{cov}_q$ & res.\ $q/r$
& $e_{\rm ref}^{\le.05}$ & $e_q^{\le.05}$
& $\times_{\le.05}$ & $\times_{(.05,.2)}$ & grid/path \\
\midrule
 5 & 0.25 & 0.20 & 1.23 & 10.15 & 149.1 & 14.7 & 84.0 & 24.8 \\
10 & 0.25 & 0.40 & 1.07 & 32.15 &  53.8 &  1.7 &  5.0 &  4.5 \\
20 & 0.25 & 0.25 & 1.00 & 46.52 &  67.9 &  1.5 &  3.7 &  4.9 \\
 5 & 0.50 & 0.40 & 1.10 &  3.94 &   6.1 &  1.5 &  9.8 &  4.0 \\
10 & 0.50 & 0.20 & 1.08 &  5.04 &   5.8 &  1.1 &  3.5 &  1.7 \\
20 & 0.50 & 0.35 & 1.01 &  8.74 &  10.4 &  1.2 &  1.5 &  1.6 \\
 5 & 1.00 & 0.40 & 1.10 &  1.03 &   1.3 &  1.3 &  2.4 &  0.3 \\
10 & 1.00 & 0.40 & 1.08 &  2.14 &   1.7 &  0.8 &  1.7 &  0.1 \\
20 & 1.00 & 0.30 & 1.00 &  1.91 &   1.6 &  0.9 &  0.9 &  0.2 \\
 5 & 2.00 & 0.40 & 1.39 &  0.28 &   0.3 &  1.0 &  0.8 &  0.04 \\
10 & 2.00 & 0.40 & 1.15 &  0.15 &   0.2 &  1.4 &  1.0 &  0.1 \\
20 & 2.00 & 0.25 & 0.90 &  0.32 &   0.3 &  1.0 &  1.0 &  0.2 \\
\bottomrule
\end{tabular}
\end{table}

\begin{figure}[htbp]
\centering
\includegraphics[width=\linewidth]{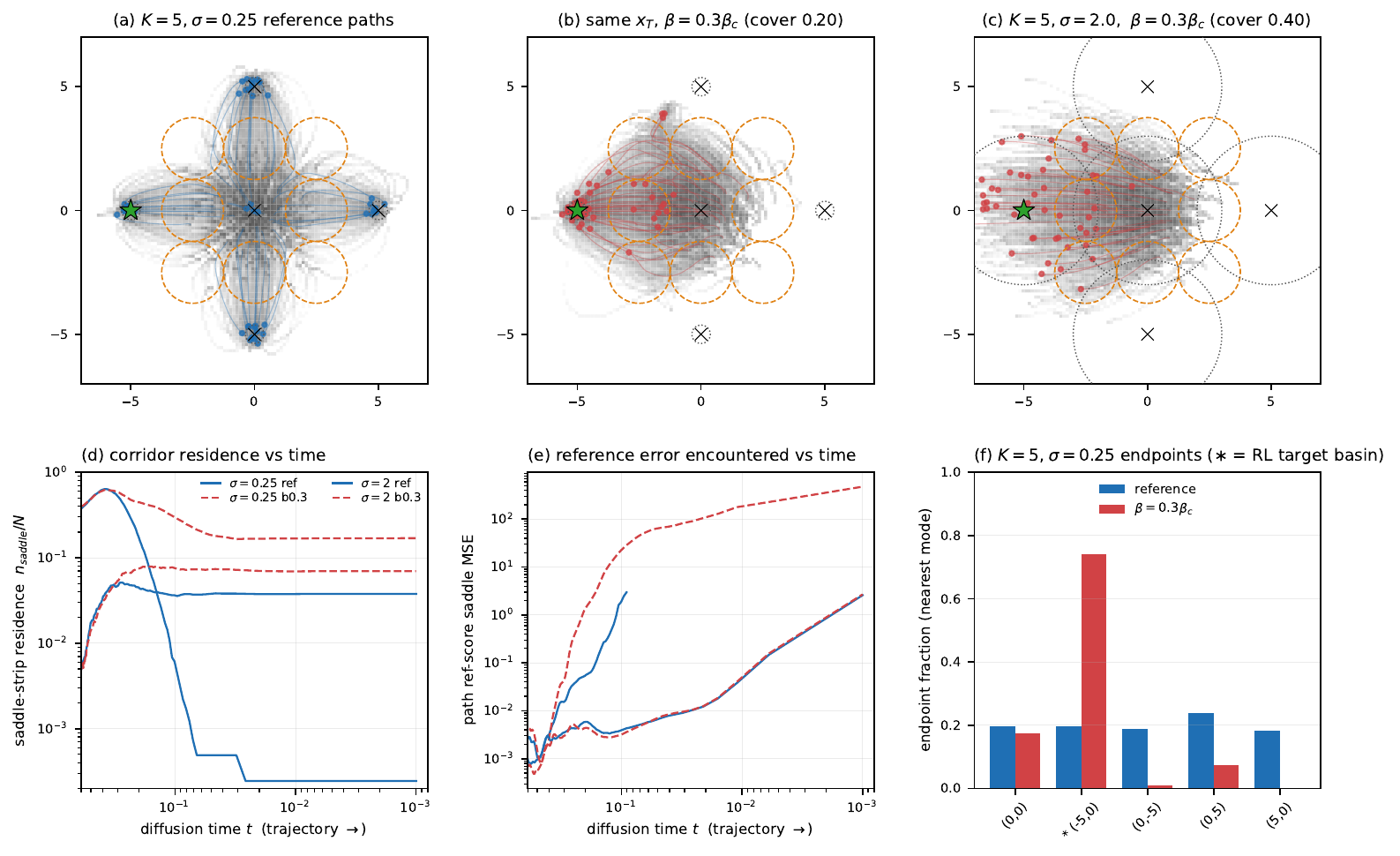}
\caption{Controlled PF-ODE trajectories from the shared $x_T$ batch (200
RK2 steps, 1024 saved trajectories; script
\texttt{plot\_expP\_paths.py}).  Dotted grey circles are mode cells
($1.5\sigma$), dashed orange circles saddle strips ($0.25\,\hat d$),
black crosses mode centres, and the green star the RL target basin.
\textbf{(a)} $K{=}5,\sigma{=}0.25$, reference: paths settle at all five
modes (grey: log-density of all states; blue: 48 example trajectories).
\textbf{(b)} Same $x_T$, $\beta=0.3\beta_c$: endpoints funnel onto the
target basin and the centre mode and the three other modes are
abandoned (cover $0.20$); corridor crossings at intermediate times are
visibly similar to (a).  \textbf{(c)} $\sigma=2$, collapsing policy:
the $1.5\sigma$ mode cells (radius $3$) swallow the saddle midpoints,
which is why saddle residence is $\sim0.02$.  \textbf{(d)} Saddle
residence vs.\ diffusion time (trajectory flows left $\to$ right): the
$\sigma=0.25$ reference leaves the strips after $t\approx0.1$ while its
collapsing policy lingers at $0.1$--$0.2$ late --- the single outlier
configuration; for $\sigma=2$ both curves stay near zero.
\textbf{(e)} Reference-score saddle MSE encountered along the same
paths: only the $\sigma=0.25$ collapsing run (upper red dashed)
separates strongly from its reference at $t\lesssim0.2$ (the
$84\times$ window); the two $\sigma=2$ curves coincide.
\textbf{(f)} Endpoint nearest-mode fractions for $K{=}5,\sigma{=}0.25$:
the reference spreads uniformly ($0.18$--$0.24$ per mode) while the
collapsed policy funnels $74\%$ of endpoints onto the \emph{target}
basin $(-5,0)$ and $\approx17\%$ onto the centre, leaving the other
modes nearly empty.  Collapse is thus reward-target funneling, not
saddle-rerouting.}
\label{fig:paths}
\end{figure}

\paragraph{Lattice packs.}
For the six 2D packs the late-window error ratio has median $0.95$
(range $0.78$--$1.33$) and the residence-ratio median $1.09$; for the
six 3D packs the corresponding numbers are $1.00$ ($0.76$--$1.38$)
and $1.02$.  No pack exceeds $2.2\times$ in any time window.

\paragraph{Time structure and discretisation control.}
Over all 24 geometries the median collapsing/reference error ratio is
$1.06$ at $t\le0.05$, $1.08$ at $t\in(0.05,0.2)$, $1.32$ at
$t\in(0.2,0.5)$, and $1.00$ at $t>0.5$: there is no early-time
amplification, and the mild intermediate-time elevation is confined
to narrow small-$K$ rows.  Endpoints are stable under solver
refinement: over the $96$ model runs,
$\max|\Delta\mathrm{cover}_{0.5}|=0.000$ and
$\max|\Delta m_{\rm tgt}|=0.001$ between the 200- and 480-step
integrations.  Error amplification decreases monotonically as $\beta$
is raised: at $\beta=3\beta_c$ the all-trajectory saddle-error ratio
has median $1.12$ (maximum $2.23$) and coverage recovers to
$0.55$--$1.00$.

\paragraph{Intentional tilt versus reference error.}
Along collapsing paths, the policy's own deviation
$\Vert\hat s_\pi-s_*\Vert^2$ in saddle strips is $4$--$1{,}950\times$
the reference error (median $26\times$ over geometries), including
broad-mode configurations whose reference scores are accurate
(e.g.\ $326\times$ at $K{=}10,\sigma{=}2$).  Collapse therefore tracks
the imposed reward tilt rather than pre-existing regions of reference
inaccuracy.

\paragraph{Verdict.}
The corridor hypothesis predicts that collapsing policies spend
substantially more trajectory time in saddle strips and accumulate
substantially larger path-weighted reference error.  Residence is
unchanged (median $+7\%$), the median error ratio is near unity in
every time window, and the sole large amplification is a single
narrow, few-mode geometry.  Corridor routing is rejected as a general
mechanism; it can survive only as a geometry-specific amplifier, and
any Gaussian-prior latent version requires the analogous experiment in
latent space, which the present data-space suite does not cover.

\section{Three-source decomposition of the dynamic amplification}
\label{app:expQ}

This appendix closes Open Problem~2 with a controlled decomposition.
Script: \texttt{expQ\_decomp.py}; analysis:
\texttt{analyze\_expQ.py}; raw outputs: \texttt{results/expQ/*.json}.

\paragraph{Protocol.}
Eight geometries span the regimes used in the main text: the 2D
$K{=}5$ width sweep ($\sigma_0\in\{0.25,0.5,1,2\}$, $d{=}5$) and the
$K{=}10$ lattice packs at $d\in\{3,5\}$ in both 2D and 3D.  All
fine-tunes use the main-text protocol (reverse KL, $\lambda{=}1$,
$1000$ steps, $S{=}20$, batch $256$, learning rate $10^{-4}$,
seed~0); coverage uses $8000$ RK2 endpoints on a paired terminal batch.
The $\beta$ grid is 14 log-uniform points adaptively bracketing both
$\beta_c^{\rm emp}$ and $\beta_c^{\rm var}$ with a decade margin, and
every fine-tune is sampled at $N\in\{25,50,100,150,200,400\}$ PF-ODE
steps.  Four phases isolate the error budgets:
\textbf{(A)~dyn} --- learned reference and learned policy (network $+$
solver $+$ anchor mismatch);
\textbf{(B)~ana} --- the exact reverse-KL optimum
$\pi^*\propto p_0 e^{\lambda R/\beta}$ sampled with the analytic
blurred score of $\pi^*$ (solver only, no learned function at all);
\textbf{(C)~anchor} --- a learned policy fine-tuned with the exact data
score as KL anchor in place of the learned reference (policy network
$+$ solver, anchor mismatch removed; four geometries bracketing the
separation regimes);
\textbf{(D)~anchorKL} --- bidirectional PF-ODE likelihoods, regional
score MSE, and the on-path DDPO-unit anchor gap
$\tfrac12\mathbb E\sum_t\beta(t)\|\hat s_{\rm ref}-s_{\rm data}\|^2dt$.

\begin{figure}[t]
\centering
\includegraphics[width=0.98\linewidth]{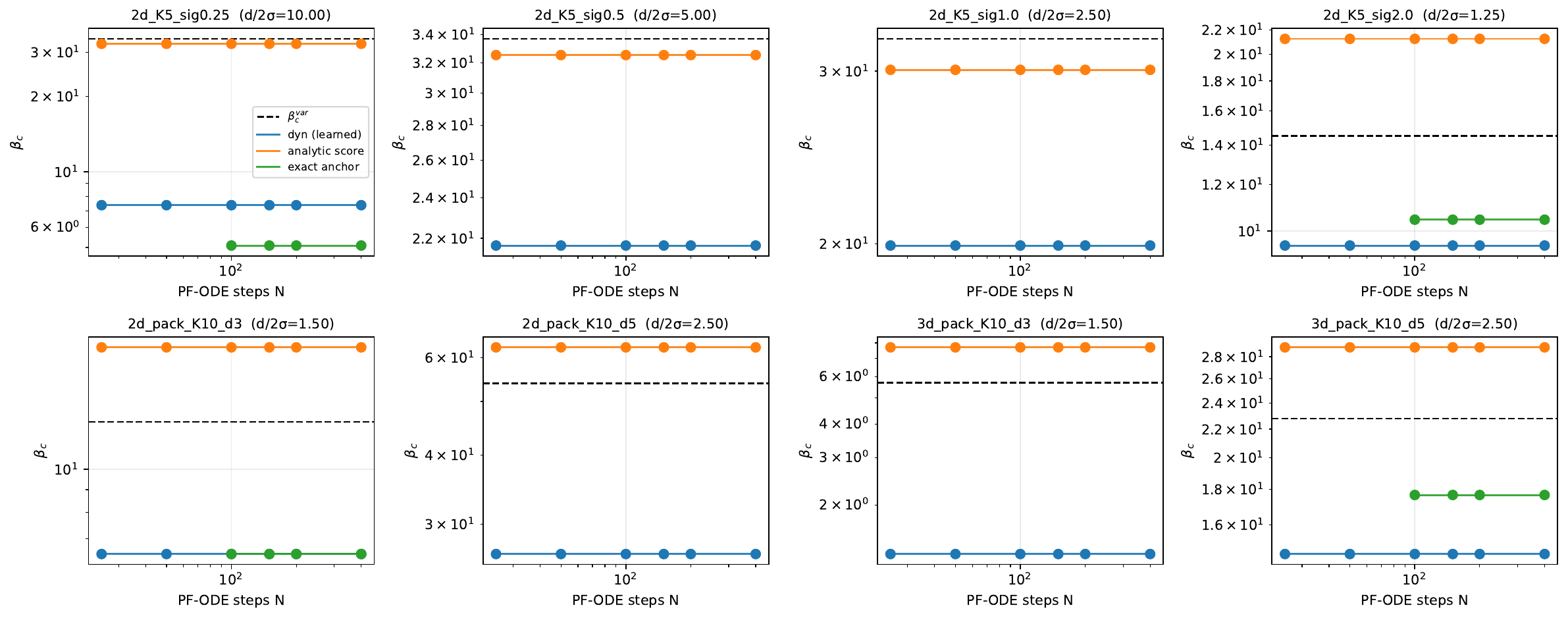}
\caption{Measured collapse boundary versus PF-ODE step count $N$ in
all eight geometries (Experiment~Q).  Orange: exact tilted GMM with
analytic scores (solver-only budget); blue: fully learned pipeline;
green: learned policy with exact-score KL anchor; dashed: grid-numeric
$\beta_c^{\rm var}$.  Every curve is exactly flat from $N{=}25$ to
$400$ (zero discretisation contribution), the solver-only boundary
coincides with $\beta_c^{\rm var}$ wherever modes are separated, and
neither learned pipeline reaches it.}
\label{fig:expQ:N}
\end{figure}

\paragraph{Result 1: ODE discretisation contributes zero at the boundary.}
$\beta_c^{\rm ana}(N)$ is identical at all six step counts for every
geometry: $\max_N|\beta_c^{\rm ana}(25)/\beta_c^{\rm ana}(400)-1|=0.0\%$.
This upgrades the endpoint-metric invariance of
Appendix~\ref{app:paths} ($|\Delta{\rm cover}|<0.001$ between 200 and
480 steps) to invariance of $\beta_c$ itself, down to $25$ steps.  The
solver-only boundary agrees with the grid-numeric variational boundary
to within $3$--$7\%$ in the separated regime
($\beta_c^{\rm ana}/\beta_c^{\rm var}=0.93$--$0.97$ for
$d_{\min}/(2\sigma_0)\ge 2.5$); in overlapping packs it lies
$16$--$47\%$ above $\beta_c^{\rm var}$, as expected because the
grid-numeric formula presumes disjoint mode balls
(cf.\ Open Problem~3).

\paragraph{Result 2: KL-anchor mismatch is not the carrier.}
The on-path anchor gap is small everywhere ($0.012$--$0.075$ DDPO
units) and, critically, is \emph{largest} in the most separated narrow
configuration ($0.075$ at $d_{\min}/(2\sigma_0){=}10$), where the
dynamic amplification is also largest, while it is smallest ($0.012$)
in the overlapping $\sigma_0{=}2$ geometry.  Replacing the learned
reference by the exact data score does not recover the boundary:
$\beta_c/\beta_c^{\rm var}$ changes from $0.22\to0.15$ (narrow, wrong
direction), $0.65\to0.72$ (overlap), $0.51\to0.51$ (2D pack, no
change), and $0.64\to0.78$ (3D pack).  The anchor intervention closes
at most half of the gap to the analytic boundary and moves it
non-systematically; anchor mismatch is therefore bounded as a carrier.
The PF-ODE likelihood estimates themselves are biased for very narrow
peaks (the reverse estimate can go negative), which is why the
on-path DDPO-unit gap is the primary measure.

\begin{figure}[t]
\centering
\includegraphics[width=0.99\linewidth]{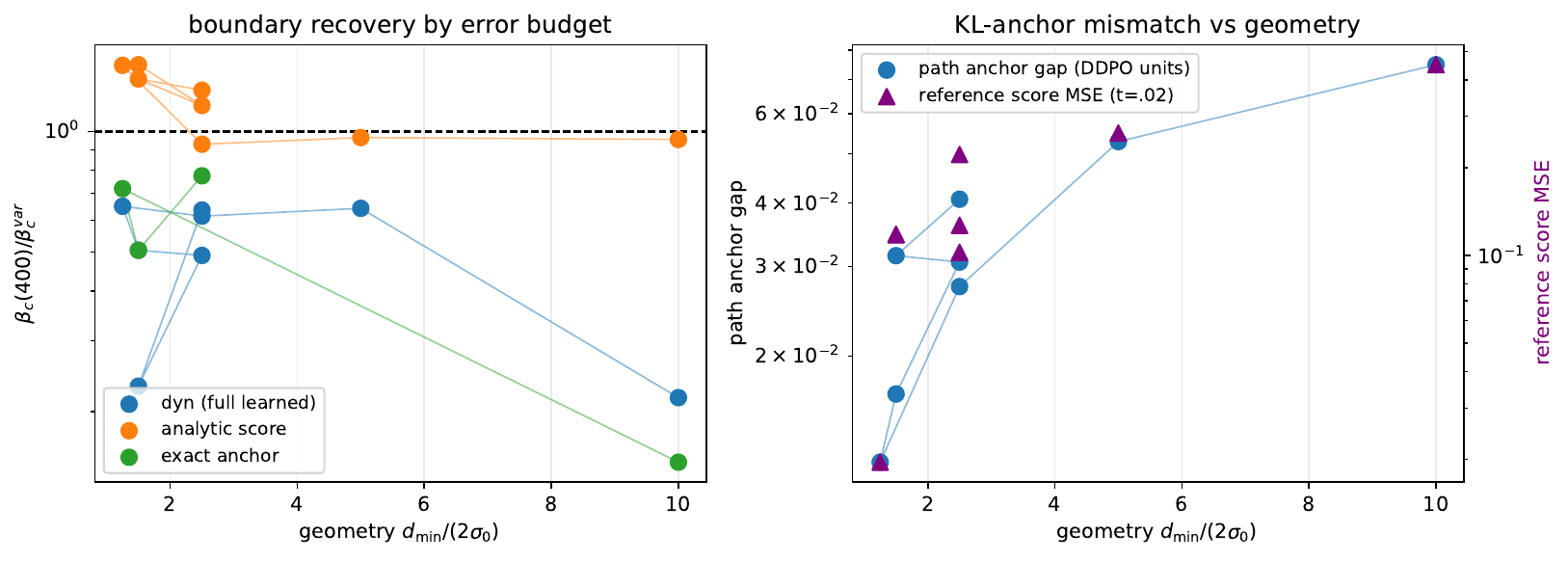}
\caption{Left: boundary recovered under each error budget at $N{=}400$
versus geometry --- analytic scores (orange) recover $\beta_c^{\rm var}$
(dashed) except in overlap; the exact-anchor intervention (green) tracks
the fully learned pipeline (blue), not the analytic one.  Right: the
path-integrated KL-anchor gap and reference score MSE at $t{=}.02$ grow
with separation, i.e.\ \emph{anti}-correlated with the amplification.}
\label{fig:expQ:sources}
\end{figure}

\begin{figure}[t]
\centering
\includegraphics[width=0.99\linewidth]{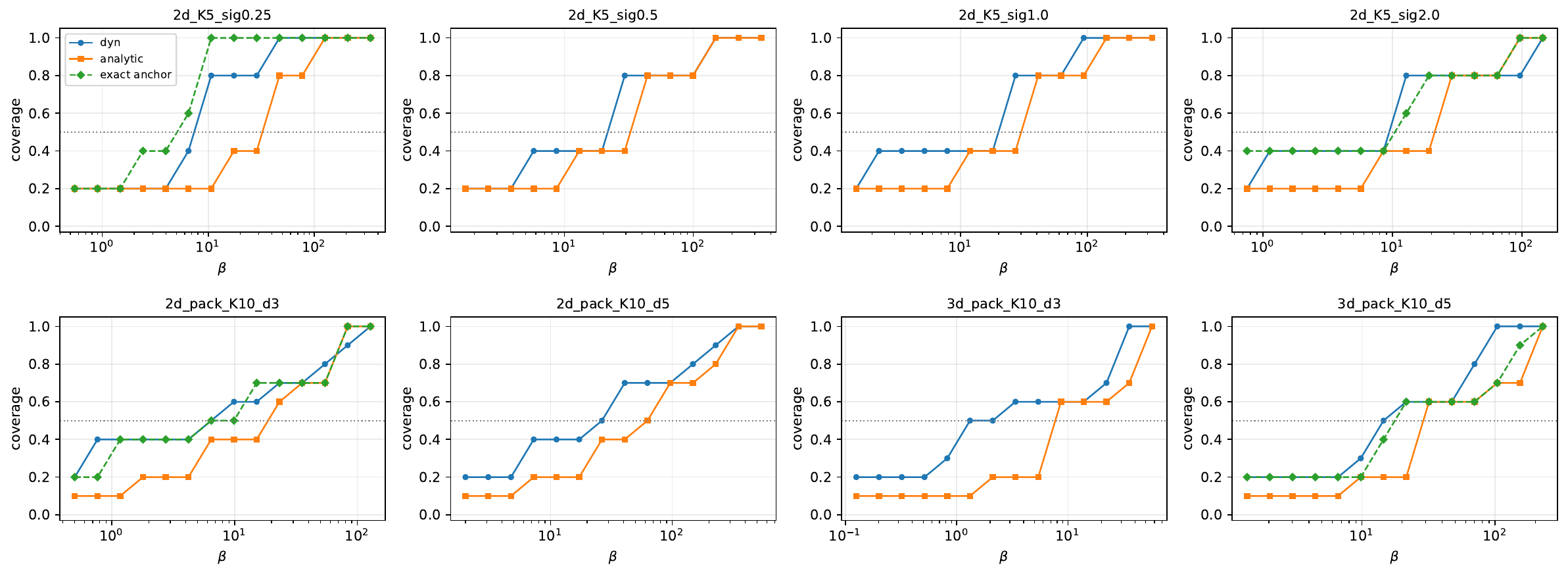}
\caption{Raw coverage curves ($N{=}400$) behind the boundary estimates:
blue learned pipeline, orange analytic scores, green exact anchor;
dotted line at the $0.5$ coverage threshold.  In the overlapping packs
(bottom row) all pipelines start with non-trivial coverage at
$\beta\!\to\!0$, the saturated-coverage regime of Open Problem~3.}
\label{fig:expQ:curves}
\end{figure}

\paragraph{Result 3: the residual is on the policy side.}
With the solver contribution exactly zero and the anchor contribution
bounded, the full-pipeline deficit $\beta_c^{\rm dyn}/\beta_c^{\rm
var}=0.22$--$0.65$ (median $0.56$) resides in the learned policy:
even anchored to the exact data score and integrated at $400$ steps,
the 2-layer MLP after $1000$ guidance steps realises only
$15$--$78\%$ of the analytic tilt boundary.  Together with
Appendix~\ref{app:paths} (no saddle-strip rerouting of the
trajectories), the identified error carrier is a \emph{static}
policy-side gap --- finite-capacity representation of the tilted score
field and/or incomplete optimisation toward the reverse-KL optimum ---
which requires no geometric corridor to act.  Caveats: this arm uses a
single seed, and the exact-anchor penalty is the on-policy
score-matching twin of the DDPO KL term rather than that term itself;
a deeper policy and a longer fine-tune would bound the
representation/optimisation split further.

\section{Fixed-prior intervention on latent geometry}
\label{app:expR}

This appendix supplies the counterfactual cell missing from the
Exp6/Exp7 design (the non-identifiability formalised in
\texttt{Mediation.lean}, Appendix~\ref{app:lean}): the prior is held
fixed at $\mathcal N(0,I)$ while the latent geometry is moved
directly.  Script: \texttt{expR\_intervention.py}; analysis:
\texttt{analyze\_expR.py}; raw outputs: \texttt{results/expR/}.

\paragraph{Design.}
Data are the $K{=}10$, $d{=}8$ lattice pack of Exp6.  Two interventions
act on the ELBO $\mathcal L={\rm recon}+w_{\rm kl}\,{\rm KL}$ with the
prior label fixed: a KL-weight arm
$w_{\rm kl}\in\{0.03,0.1,0.3,1,3\}$ at encoder width $128$, and a
capacity arm with widths $\{64,128,256\}$ at $w_{\rm kl}{=}1$.  VAEs
train for $30\,000$ steps; geometry ($\sigma_{\rm lat}$,
$d_{\min}/(2\sigma_{\rm lat})$, nearest-centre accuracy, decoded
coverage), reconstruction, and KL are snapshotted on a fixed eval
batch at seven training steps.  Each frozen VAE then follows the
main-text latent pipeline: a score model on the aggregate posterior
($40\,000$ steps), an exact-GMM variational boundary on the cell's own
latent centres, and a 14-point $\beta$ grid spanning
$[0.03,30]\,\beta_c^{\rm var}$ with the standard fine-tune protocol.

\begin{table}[h]
\centering\small
\begin{tabular}{lccccc}
\hline
cell & $w_{\rm kl}$ & $d_{\min}/(2\sigma_{\rm lat})$ &
$\beta_c^{\rm var}$ & $\beta_c^{\rm emp}$ & emp/var \\
\hline
w0.03 & 0.03 & 2.08 & 3.21 & 1.37 & 0.43 \\
w0.1  & 0.10 & 1.55 & 1.99 & 0.50 & 0.25 \\
w0.3  & 0.30 & 1.14 & 1.17 & 0.85 & 0.73 \\
w1.0  & 1.00 & 1.12 & 1.20 & 0.67 & 0.56 \\
w3.0  & 3.00 & 0.65 & 0.77 & 0.96 & 1.24 \\
w1\,h64  & 1.00 & 1.16 & 1.32 & 0.74 & 0.56 \\
w1\,h256 & 1.00 & 1.07 & 1.04 & 0.75 & 0.73 \\
\hline
\end{tabular}
\caption{Experiment~R: all seven cells use the SAME Gaussian prior.
The KL-weight intervention moves latent geometry monotonically from
$2.08$ to $0.65$ and the variational boundary from $3.21$ to $0.77$.}
\label{tab:expR}
\end{table}

\paragraph{Result 1: geometry is moved at fixed prior, and $\beta_c$
follows.}
Increasing $w_{\rm kl}$ compresses the latent geometry monotonically
($2.08\to0.65$) and the variational boundary with it
($3.21\to0.77$, a $4.2\times$ range); the measured boundary moves over
a comparable range ($0.50$--$1.37$) with the single-seed Monte-Carlo
scatter seen throughout.  The $w_{\rm kl}{=}1$ cell reproduces the
Exp6 Gaussian endpoint ($\beta_c{=}0.67$ versus $0.71$ in Exp6), and
the $w_{\rm kl}{=}3$ cell reaches $d_{\min}/(2\sigma_{\rm lat}){=}0.65$
--- inside the Exp7 Gaussian regime ($0.81$) --- with the reference
itself losing a mode at the decoder (decoded plateau $0.80$): the
``collapsed Gaussian'' geometry is reproduced by ELBO pressure alone,
with no change of prior label.

\begin{figure}[t]
\centering
\includegraphics[width=0.86\linewidth]{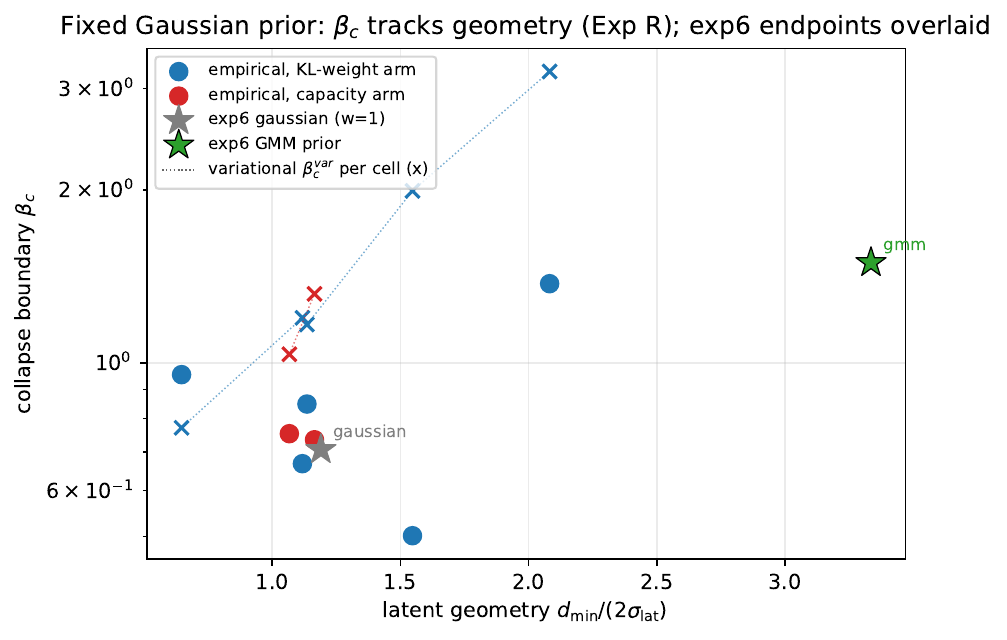}
\caption{Experiment~R mediation test.  Filled circles: measured
$\beta_c$ under a fixed Gaussian prior with geometry moved by KL
weight; crosses: each cell's variational boundary; the dotted line
joins them; stars: the Exp6 Gaussian and GMM endpoints.  The Exp6
Gaussian star sits on the $w_{\rm kl}{=}1$ intervention cell.}
\label{fig:expR:mediation}
\end{figure}

\paragraph{Result 2: ELBO dynamics show how the intervention acts.}
Geometry develops during training (Figure~\ref{fig:expR:elbo}): modes
separate over the first $\sim10^3$ steps and saturate at
$w_{\rm kl}$-dependent plateaus, while stronger KL pressure inflates
the latent blob width $\sigma_{\rm lat}$ ($0.23$ at $w{=}0.03$ to
$0.45$ at $w{=}3$).  Encoder width alone is an inactive lever in this
range ($d_{\min}/(2\sigma_{\rm lat})=1.07$--$1.16$ for
$64/128/256$ hidden units); the KL term is the binding constraint.

\begin{figure}[t]
\centering
\includegraphics[width=0.99\linewidth]{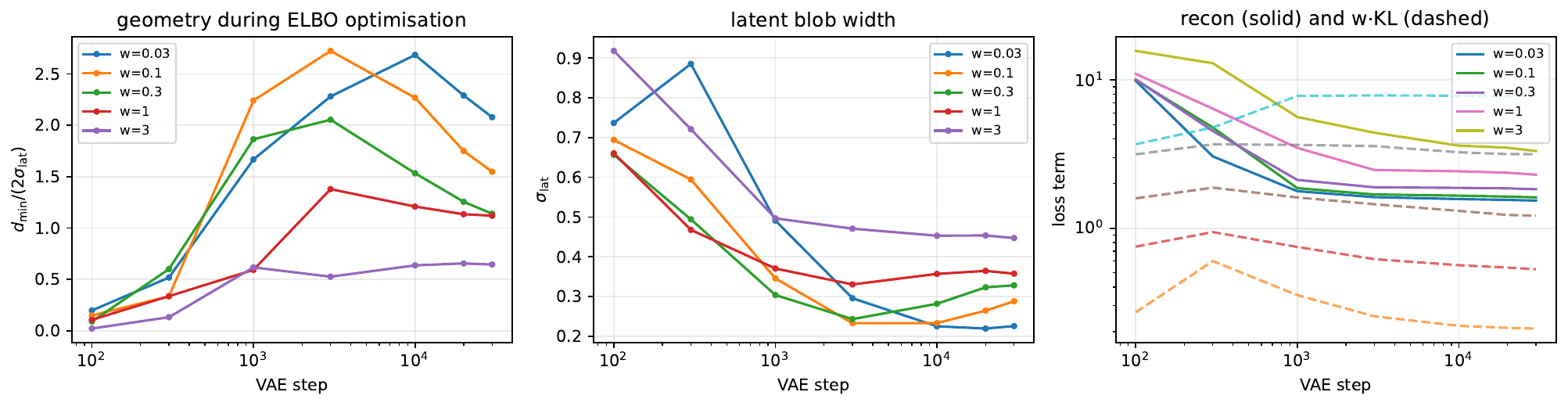}
\caption{Within-training ELBO trajectories under the fixed Gaussian
prior (Experiment~R): geometry ratio (left) and latent blob width
(centre) saturate at KL-weight-dependent levels; right shows the
reconstruction and weighted-KL loss terms.}
\label{fig:expR:elbo}
\end{figure}

\paragraph{Scope.}
The experiment establishes the causal chain
$w_{\rm kl}\!\to\!{\rm geometry}\!\to\!\beta_c$ without touching the
prior label, so the prior has no privileged role in producing the
boundary in the overlapping range ($0.65$--$2.08$), which contains
both Exp7 latents.  It does \emph{not} supply the reverse
counterfactual --- a multimodal prior forced to a compressed geometry
--- and the Exp6 GMM endpoint at $d_{\min}/(2\sigma_{\rm lat}){=}3.34$
lies outside the geometry a Gaussian-prior ELBO reaches on this data,
so the strongest-separation cell remains a single observation.

\section{Four-prior comparison: Gaussian, GMM, VampPrior, GMVAE}
\label{app:expV}

Implication~1 claims that the multi-modal advantage attaches to the
\emph{achieved latent geometry}, not to the mixture symbol: any prior
that forces the ELBO to preserve mode separation should reproduce the
GMM result.  This appendix tests the claim on the two most-cited
learnable-mixture priors.  Script: \texttt{expV\_priors.py}; analysis:
\texttt{analyze\_expV.py}; raw outputs: \texttt{results/expV/}.

\paragraph{Design.}
Data and protocol are identical to Exp6: the $K{=}10$, $d{=}5$,
$\sigma_0{=}0.5$ lattice pack in $\mathbb R^8$, 2D latent, VAE
$30\,000$ steps (Gaussian KL warm-up $3000$), score model $40\,000$
steps, 14-point $\beta$ grid on $[0.03,30]\,\beta_c^{\rm var}$, 1000
fine-tune steps per point, ODE $N{=}8000$ samples / 150 RK2 steps.
Four priors, three seeds each:
\begin{itemize}\itemsep2pt
\item \textbf{gaussian}: $\mathcal N(0,I)$;
\item \textbf{gmm}: the label-tracked $K{=}10$ Gaussian-mixture prior of
Exp6 (centres updated to labelled aggregate-posterior means);
\item \textbf{vamp}: a VampPrior \citep{tomczak2018vae} with $M{=}K{=}10$
learnable pseudo-inputs passed through the \emph{shared} encoder,
producing a 10-component mixture posterior as the aggregate prior ---
no labels used;
\item \textbf{gmvae}: a VaDE/GMVAE \citep{lee2021mixture} prior with
$g_\mu,g_{\log\sigma^2},g_{\log w}$ all learnable \texttt{Parameter}s
trained jointly; the KL is the GMVAE decomposition into a categorical
term and responsibility-weighted conditional Gaussian terms --- no
labels used.
\end{itemize}
Each learned prior is additionally audited by \emph{component
placement}: a greedy one-to-one matching of its $K$ components to the
fitted data-mode centres, reporting the matched fraction and the mean
matched distance in units of the data separation $d$.

\begin{table}[h]
\centering\small
\begin{tabular}{lcccccccc}
\hline
prior & $\sigma_{\rm lat}$ & $d_{\min}/(2\sigma_{\rm lat})$ &
placement & $\beta_c^{\rm var}$ & $\beta_c^{\rm emp}$ & emp/var &
$\alpha_{\rm emp}$ & plateau \\
\hline
gaussian & $0.36{\pm}.01$ & $1.19{\pm}.05$ & --- &
$1.16{\pm}.10$ & $0.71{\pm}.02$ & $0.61{\pm}.08$ & $2.7{\pm}.2$ & 1.00 \\
gmm & $0.039{\pm}.001$ & $5.2{\pm}1.6$ & $0.01d$ &
$1.47{\pm}.23$ & $1.27{\pm}.19$ & $0.87{\pm}.10$ & $430{\pm}70$ & 1.00 \\
vamp & $0.34{\pm}.01$ & $1.73{\pm}.23$ & $0.27d$ &
$5.9{\pm}1.6$ & $4.0{\pm}1.7$ & $0.67{\pm}.18$ & $18{\pm}9$ & 1.00 \\
gmvae & $0.75{\pm}.01$ & $1.13{\pm}.04$ & $0.67d$ &
$4.5{\pm}.6$ & $4.5{\pm}1.4$ & $0.97{\pm}.21$ & $4.0{\pm}1.2$ &
$0.95$--$1.00$ \\
\hline
\end{tabular}
\caption{Experiment~V, mean${\pm}$std over three seeds.
Placement is the mean greedy component-to-mode distance in units of
$d$ (fraction matched is 10/10 for every learned mixture).
$\alpha_{\rm emp}=\beta_c^{\rm emp}/(2\sigma_{\rm lat}^2)$; plateau is
decoded coverage at the two largest $\beta$ values.}
\label{tab:expV}
\end{table}

\paragraph{Result 1: only the label-tracked GMM prior reaches the
tight, well-separated geometry.}
The GMM prior reproduces the Exp6 endpoints
($d_{\min}/(2\sigma_{\rm lat}){=}3.3$--$7.2$,
$\beta_c^{\rm emp}{=}1.09$--$1.52$ versus $1.50$ in Exp6, and the
Gaussian cells $0.68$--$0.74$ versus $0.71$).  The VampPrior moves the
\emph{ratio} only to $1.4$--$2.0$ while leaving the absolute width
nearly Gaussian ($\sigma_{\rm lat}{=}0.33$ versus $0.36$): its
pseudo-inputs do find all ten modes (10/10 matched, $0.27d$ mean
offset) but the mixture components it induces through the shared
encoder stay broad.  The GMVAE does not separate at all in ratio
terms ($1.09$--$1.18$), leaves the latent \emph{wider} than the
Gaussian prior ($\sigma_{\rm lat}{=}0.75$), and places its components
at $0.67d$ from the data modes --- formally 10/10 matched but
effectively between the modes.

\paragraph{Result 2: after conditioning on the achieved geometry, no
symbolic prior effect survives.}
For every prior, the measured boundary is accounted for by the cell's
\emph{own} variational boundary, emp/var$\in[0.43,1.24]$ with means
$0.61/0.87/0.67/0.97$ for gaussian/gmm/vamp/gmvae; the normalised
collapse curves cross coverage $0.5$ at $\beta/\beta_c^{\rm var}$
between $0.5$ and $1.2$ for all four families
(Fig.~\ref{fig:expV-curves}).  In particular, the apparently large
GMVAE boundary ($2.6$--$6.1$ at $d_{\min}/(2\sigma_{\rm lat})\approx
1.1$, where the Gaussian prior gives $0.7$) is explained by geometry,
not by a direct prior penalty: at matched ratio the GMVAE latent is
$2.1\times$ wider, and $\beta_c$ scales with $\sigma_{\rm lat}^2$; in
scale-free units its $\alpha_{\rm emp}{=}2.5$--$5.2$ sits next to the
Gaussian $2.6$--$3.0$, and its emp/var is $0.73$--$1.24$.  The
mediation plane with the Exp~R fixed-Gaussian KL-weight arm overlaid
(Fig.~\ref{fig:expV-mediation}) shows the same thing geometrically:
the high Vamp/GMVAE points sit at their own geometry, and the prior
label adds no residual beyond seed scatter.

\paragraph{Result 3: no plateau penalty for the learned priors.}
All four priors reach full decoded coverage at large $\beta$
(Fig.~\ref{fig:expV-plateau}); the GMVAE dips to $0.95$ in latent
coverage in two seeds, but this decodes to full data-space coverage.
None of the learned priors damages mode preservation once fine-tuned.

\paragraph{Verdict for Implication~1.}
The implication is supported in its conditional form ---
\emph{any} prior that achieves the tight, well-placed, well-separated
geometry reproduces the GMM boundary and plateau --- and sharpened in
an important way: the two standard label-free learned-mixture priors
do \emph{not} achieve that geometry under the same 30k-step ELBO
budget.  The mixture symbol is not sufficient; the ELBO must actually
be forced to tighten the components onto the data modes, which in our
hands required label tracking.  This converts Implication~1 from a
promise about architecture classes into a falsifiable statement about
achieved geometry, and identifies ``ELBO recipes that make Vamp/GMVAE
components as tight as label-tracked ones'' as the remaining
engineering question.

\begin{figure}[h]
\centering
\includegraphics[width=0.92\linewidth]{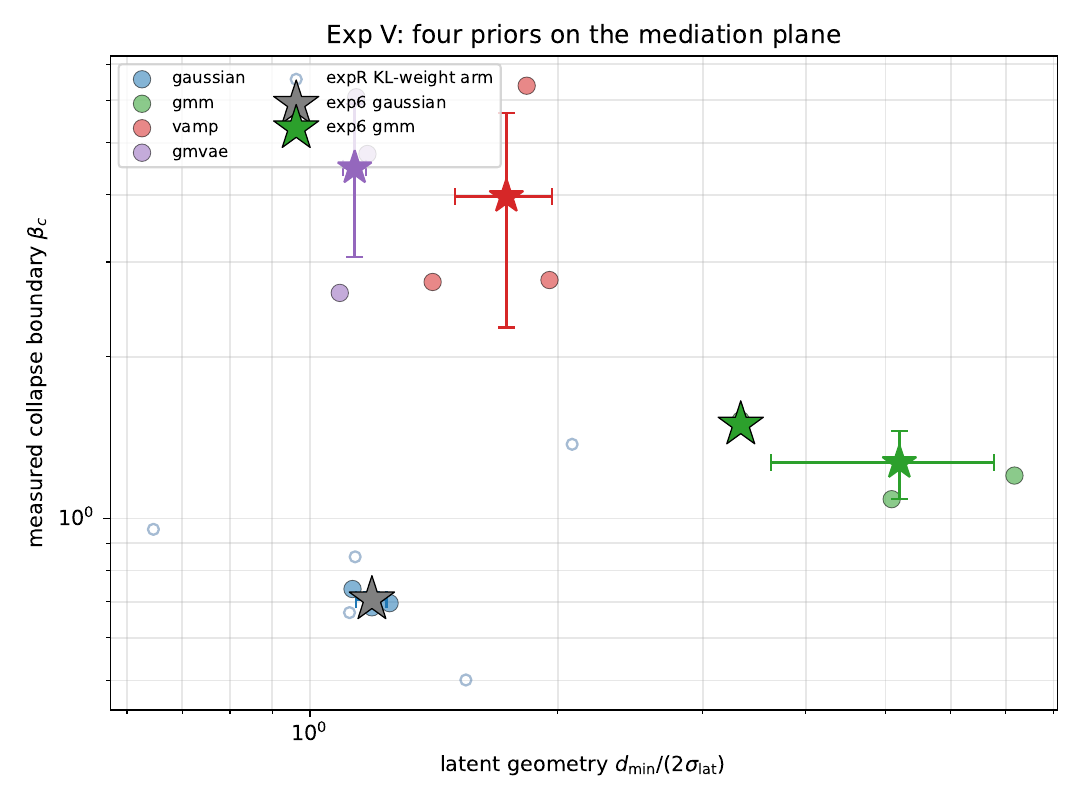}
\caption{Experiment~V on the mediation plane (stars: 3-seed
mean${\pm}$std; hollow circles: the Exp~R fixed-Gaussian KL-weight
arm; large stars: Exp6 endpoints).  Vamp/GMVAE cells occupy their own
geometries rather than the GMM corner.}
\label{fig:expV-mediation}
\end{figure}

\begin{figure}[h]
\centering
\includegraphics[width=0.92\linewidth]{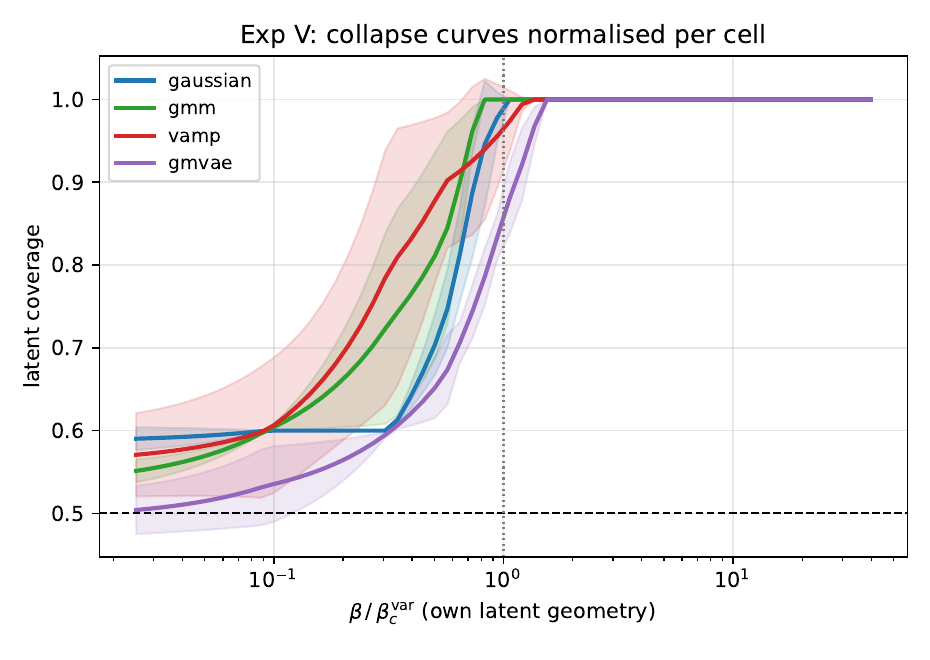}
\caption{Experiment~V collapse curves normalised by each cell's own
$\beta_c^{\rm var}$ (mean${\pm}$std band, three seeds).  All four
prior families cross coverage $0.5$ within a factor $0.5$--$1.2$ of
their geometry-predicted boundary.}
\label{fig:expV-curves}
\end{figure}

\begin{figure}[h]
\centering
\includegraphics[width=0.62\linewidth]{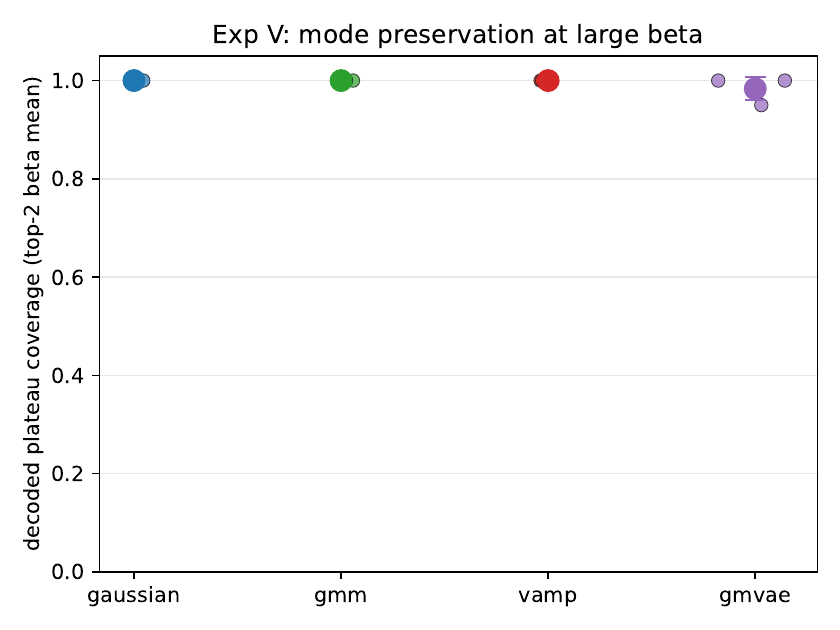}
\caption{Experiment~V decoded plateau coverage.  No prior family loses
data-space modes at large $\beta$.}
\label{fig:expV-plateau}
\end{figure}

\section{Separation scan: does $\alpha=\beta_c/\beta_{\rm cross}$
tend to $1$?}
\label{app:expW}

The roadmap of Section~``Future Work'' conjectured an asymptotic
$\mathrm{dyn}/\mathrm{var}\to1$ as $g=d_{\min}/(2\sigma_0)\to\infty$,
and Proposition~\ref{prop:prior-scale} locates $\alpha\approx1$ near
$d_{\min}>4\sigma_0$.  This appendix measures $\alpha(g)$ directly on
1D equal-weight lattices over $g\in[0.5,10]$.
Script: \texttt{expW\_limits.py}; analysis:
\texttt{analyze\_expW.py}; raw outputs: \texttt{results/expW/}.

\paragraph{Design.}
Equal-weight 1D Gaussian mixtures with $\sigma_0=\lambda=1$, lattice
spacing $d=2g$, $K\in\{3,5,10\}$ modes, $g\in\{0.5,0.75,1,1.5,2,3,4,6,
8,10\}$.  The target is the marginal mode $j=\lceil (K-1)/2\rceil$
positions in from the nearest end ($j=1,2,5$), which is where the
$f=0.5$ coverage boundary is tightest.  The $K=2$ layout is excluded:
with $f=0.5$ the target mode's own mass already makes coverage
$\ge0.5$, so no boundary is ever defined.  The learned pipeline is
identical to Exp6/Exp~Q (40\,000 score steps, 1000 fine-tune steps per
grid point, ODE $N=8000$ / 150 RK2 steps, 14-point $\beta$ grid).
Three boundaries are recorded per cell: the dominant-order saddle
expression $\beta_c^{\rm cl}=\lambda((jd)^2/\!\ln2-2\sigma_0^2)$, the
exact variational boundary $\beta_c^{\rm var}$ of the tilted mixture,
and the empirical learned-pipeline boundary $\beta_c^{\rm emp}$, each
divided by $\beta_{\rm cross}=2\lambda\sigma_0^2=2$.  The full scan is
30 seed-0 cells; the $K=5$ cells at $g\in\{1,2,4\}$ carry two
additional seeds (36 cells total).

\begin{table}[h]
\centering\small
\begin{tabular}{llcccc}
\hline
& $g$ & 1 & 2 & 4 & 10 \\
\hline
$K{=}3$ ($j{=}1$) & $\alpha^{\rm var}$
  & sat & 4.1 & 19.8 & 128 \\
& $\alpha^{\rm emp}$ & 0.11 & 0.46 & 6.5 & 14.5 \\
$K{=}5$ ($j{=}2$) & $\alpha^{\rm var}$
  & 1.9 & 18.2 & 78.3 & 494 \\
& $\alpha^{\rm emp}$ & $0.91{\pm}.21$ & 3.5 & $18.7{\pm}6.1$ & 55.9 \\
$K{=}10$ ($j{=}5$) & $\alpha^{\rm var}$
  & 10.5 & 67.7 & 283 & 1772 \\
& $\alpha^{\rm emp}$ & 4.5 & 17.0 & 71.1 & 445 \\
\hline
\end{tabular}
\caption{Experiment~W: $\alpha^{\rm var}$ and $\alpha^{\rm emp}$ at
four separations (``sat'': ball metric saturated at $\beta=0$,
$\beta_c^{\rm var}=0$ censored).  Extra-seed values for $K{=}5$ are
exact at $g=2$ (3.51 all seeds) and spread at $g=1$ ($0.62,1.06,1.06$)
and $g=4$ ($15.1,15.1,25.7$).}
\label{tab:expW}
\end{table}

\paragraph{Result 1: $\alpha$ grows with $g$, it is not driven to 1.}
Figure~\ref{fig:expW-alpha} shows the three $\alpha(g)$ curves.
$\alpha^{\rm var}$ rises monotonically on the $2j^2g^2/\!\ln2$ scale
of the saddle envelope, from $\mathcal O(1)$ at $g\approx1$ to $128$
($K{=}3$), $494$ ($K{=}5$) and $1772$ ($K{=}10$) at $g=10$ (a factor
$2$--$4$ below the envelope itself).
$\alpha^{\rm emp}$ follows the same rise and reaches $14.5$, $55.9$
and $445$ at $g=10$; every $\alpha^{\rm emp}$ for $g\ge3$ exceeds $2$,
and at $g=10$ it exceeds $14$ in every layout.  The empirical
$\alpha=1$ crossing sits near $g\approx2.4$ for $K=3$, $g\approx1.0$
for $K=5$, and below $g=0.5$ for $K=10$.  Thus $\beta_{\rm cross}$ is
correctly interpreted as the sign- and order-of-magnitude marker where
the tilted component widths lock (the crossover-scale discussion of
Section~\ref{sec:theory}), not as a
limit the empirical boundary approaches: $\alpha=1$ is crossed once,
from below, at a geometry- and $K$-dependent location, after which
$\alpha$ keeps growing.

\paragraph{Result 2: ordering of the three boundaries, and the high-$g$
roll-over.}
At every non-saturated cell $\alpha^{\rm emp}<\alpha^{\rm var}
<\alpha^{\rm cl}$.  The dominant-order saddle expression omits all
competing modes and is an upper envelope, overshooting by factors
$1.5$--$25$.  The learned-pipeline gap widens at large $g$: for
$K{=}10$, $\alpha^{\rm emp}/\alpha^{\rm var}$ falls from $0.43$ at
$g=1$ to $0.25$ at $g\ge2$, and $\alpha^{\rm emp}$ is non-monotone at
the top end ($209$ at $g=6$, $167$ at $g=8$, $445$ at $g=10$ for
$K{=}10$; $15.7$ then $14.5$ for $K{=}3$).  These cells carry
\emph{exact} flags (the boundary lies inside the grid, not at its
edge), so the roll-over is pipeline behaviour rather than censoring:
sharper, wider saddles are progressively harder for the finite network
and the $1000$-step fine-tune to represent, the same policy-side
carrier isolated in Exp~Q (Appendix~\ref{app:expQ}).

\paragraph{Result 3: overlap censoring at small $g$.}
The exact ball-metric boundary is zero whenever mode balls overlap at
zero tilt ($K=3$, $g\le1$; $K=5$, $g=0.5$); these cells are plotted as
hollow censored markers, not as measured zeros.  The empirical
pipeline still shows a low-$g$ transition ($\alpha^{\rm emp}=0.06$ at
$K=3,g=0.5$, itself flagged below the grid floor) and $0.33$--$0.66$
at $g=0.5$--$0.75$.  The $K=10$ layout saturates nowhere on the scan:
the long lattice always contains a non-overlapped marginal mode.

\paragraph{Result 4: $K$-scaling and the 2D/3D check.}
At fixed $g=2$, $\alpha^{\rm var}$ is $4.1/18.2/67.7$ across
$K=3/5/10$, tracking the $j^2\in\{1,4,25\}$ prefactor (ratios
$1:4.5:16.6$); $\alpha^{\rm emp}$ gives $0.46/3.5/17.0$
($1:7.6:37$).  The rightmost panel overlays the 2D/3D lattice packs of
Exp~Q: their $\alpha^{\rm emp}$ span the same range ($0.5$--$50$) with
no dimension-specific anomaly, and the $K=5$ closed-form envelope
brackets them as it does the 1D cells.

\paragraph{Verdict.}
Open Problem~3's regime statement is sharpened: $d_{\min}>4\sigma_0$
($g>2$) is a fair single-number guide only for the $K=3$ marginal
layout, where $\alpha^{\rm emp}$ indeed crosses $1$ nearby; the
general statement is the $g$- and $j$-dependent scaling measured here.
The conjectured $\mathrm{dyn}/\mathrm{var}\to1$ endpoint is not
observed --- the exact optimum itself departs as $j^2g^2$, and the
learned pipeline departs further as the saddle sharpens.

\begin{figure}[h]
\centering
\includegraphics[width=0.99\linewidth]{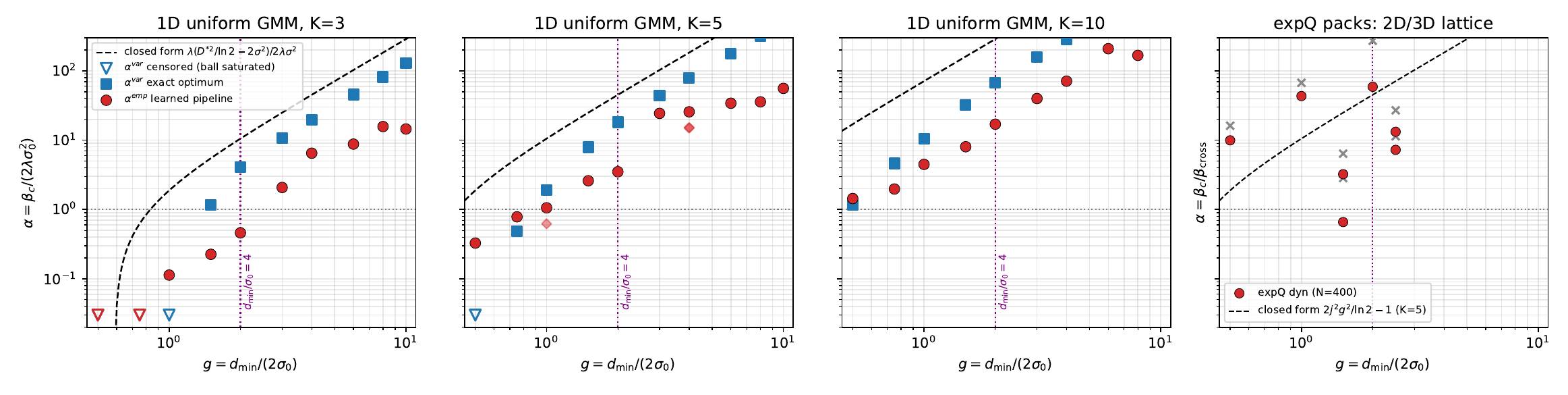}
\caption{Experiment~W: $\alpha=\beta_c/\beta_{\rm cross}$ against
$g=d_{\min}/(2\sigma_0)$.  Blue squares: exact variational optimum;
red circles: learned pipeline (faint diamonds: extra $K=5$ seeds);
hollow triangles: ball-saturated / below-grid censors; black dashed:
dominant-order saddle envelope; dotted: $\alpha=1$ and
$d_{\min}/\sigma_0=4$.  Right panel: Exp~Q 2D/3D packs against the
$K=5$ envelope.}
\label{fig:expW-alpha}
\end{figure}

\section{Rare-basin reward on alanine dipeptide: enrichment versus
retention}
\label{app:expX}

The Outlook asked whether the causal chain
prior~$\to$~latent geometry~$\to$~collapse boundary extends to
\emph{physical} conformations under a reward that favours a rare basin.
This appendix runs that test on the real alanine-dipeptide MD data of
Exp7 (no synthetic data).  Script: \texttt{expX\_rare.py}; analysis:
\texttt{analyze\_expX.py}; raw outputs: \texttt{results/expX/}.

\paragraph{Design.}
The Exp7 features define three physical basins with equilibrium
weights $w=(0.435,0.338,0.157)$: $\beta$, $\alpha_R$, and the rare
C5 basin.  Two priors --- Gaussian and the label-tracked GMM --- are
trained at the full budget (30\,000 VAE steps, 40\,000 score steps),
three seeds each.  For every target basin separately, the policy is
fine-tuned (reverse KL, 1000 steps) on a 12-point $\beta$ grid
anchored at $\max(\beta_c^{\rm var},2\sigma_{\rm lat}^2)$, and $N=8000$
ODE samples are measured in four spaces: latent/decoded
$\times$ hard/soft.  Enrichment is $m_t/w_t$; hard retention is the
fraction of non-target basins still holding mass $\ge0.5/K$; soft
retention is the weight-capped recovery fraction.  Soft kernels use
$\sigma=d_{\min}/2$ in each space (latent $\approx0.027$, decoded
$\approx0.98$): the fitted diagnostic width $\sigma_x\approx2.87$
exceeds the basin spacing and flattens kernel weights, whereas
$d_{\min}/2$ keeps adjacent basins distinguishable.

\paragraph{Accounting caveat: the soft metric at fixed enrichment.}
Soft kernel masses sum to one, so soft retention at a fixed enrichment
$E$ is pinned by mass conservation,
$(1-Ew_t)/\sum_{k\ne t}w_k$: for C5 at $E=3$ this is $0.684$, and all
six cells return $0.684$ to three decimals.  We report soft curves and
AUCs, but the discrimination between priors at fixed enrichment comes
from the hard gate, which is not mass-conservation constrained.

\paragraph{Result 1: $3\times$ enrichment is physically reachable only
for the rare basin.}
Maximum hard decoded enrichment for the two common basins is bounded
by $1/w$: $2.3$ ($\beta$) and $3.0$ ($\alpha_R$).  Observed maxima are
$2.2$ and $2.4$ (Gaussian) and $1.3$ and $1.7$ (GMM), so no common
basin crosses $3\times$ for either prior; these cells are censored
above rather than extrapolated.  C5 admits $1/w\approx6.4$ and is the
discriminating target.

\paragraph{Result 2: on C5, Gaussian guidance is sharp but destructive;
GMM guidance is soft and preserving.}
Gaussian-prior policies reach $3.9$--$4.3\times$ hard enrichment of C5
in all three seeds, but hard retention is $0.5$ everywhere at
$3\times$: exactly one of the two other basins falls below $0.5/K$.
GMM-prior policies reach $3.5$--$3.6\times$ with hard retention
$1.0$ in two seeds; the third peaks at $2.91\times$ (censored below
$3\times$) with retention $1.0$ across the entire sweep
(Figure~\ref{fig:expX-retention3}).  Soft AUCs favour or match the GMM
on every basin ($0.97$ vs $0.80$ on $\beta$; comparable on $\alpha_R$
and C5), and the full Pareto curves
(Figure~\ref{fig:expX-pareto}) show the GMM trading enrichment for
retention strictly less steeply once the hard gate begins to bind.

\paragraph{Result 3: mechanism --- diffuse latent ``collapse'' plus
decoder amplification.}
The GMM latent is tight ($\sigma_{\rm lat}=0.027$--$0.030$,
$d_{\min}/(2\sigma_{\rm lat})\approx0.9$--$1.0$; component centres
only $0.055$ apart).  At the smallest $\beta$, hard latent target mass
reaches only $\approx0.22$: the $f=0.5$ coverage criterion is met
(the other balls hold $<0.5/K$ each) but $78\%$ of samples lie in the
inter-ball tail, and the deterministic decoder, with a gain of
$\approx36$ from latent to feature spacing ($1.96/0.055$), spreads that
tail back across all three physical basins.  The Exp7 coverage sweep shows the
same signature directly: latent coverage stays at $1/3$ from
$\beta=10^{-6}$ to $2\times10^{-3}$ while decoded coverage stays at
$1.0$.  Gaussian-prior collapse is instead sharp (latent target mass
$0.91$), and the decoder transmits it faithfully, which is what makes
$4\times$ enrichment possible --- and what erases a basin.

\paragraph{Verdict.}
On real molecular data the conditional causal chain survives in a
more interesting form than ``GMM is uniformly better'': the prior's
geometry selects a point on the enrichment--retention Pareto
frontier.  The Gaussian prior offers controllable, sharp rare-basin
enrichment at the price of basin destruction; the tight GMM prior
trades $\approx20\%$ of reachable enrichment for full retention of the
common basins.  Limits: 2D latent, $K=3$, and the preserving mechanism
relies on the decoder's non-injective amplification, which should be
re-audited for higher-capacity decoders and larger molecules.

\begin{figure}[h]
\centering
\includegraphics[width=0.99\linewidth]{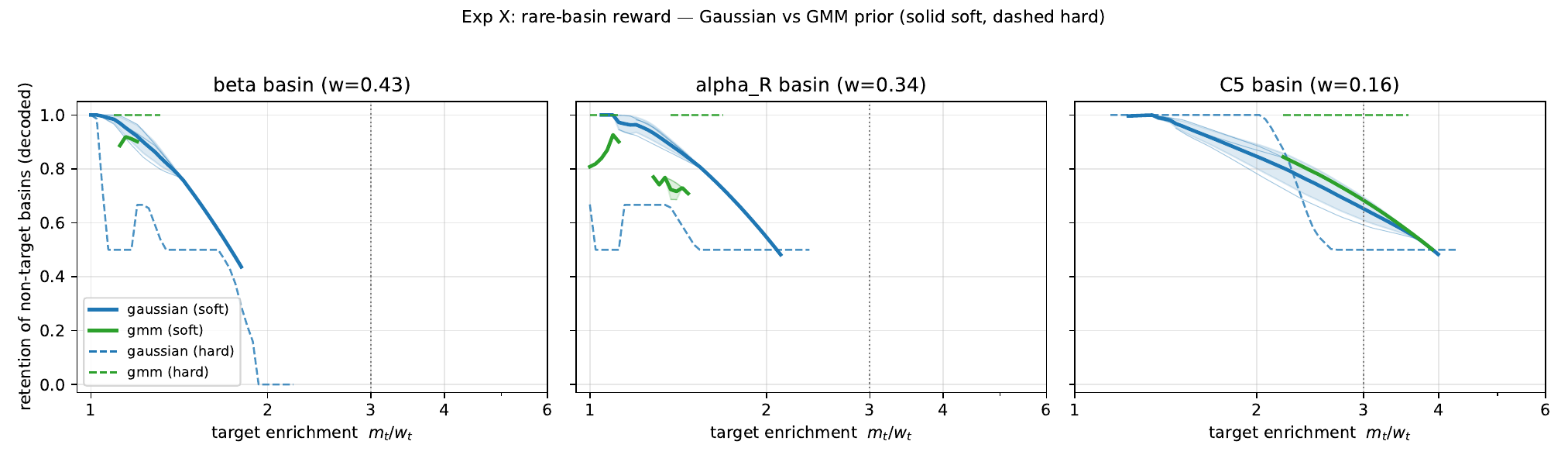}
\caption{Experiment~X: non-target-basin retention versus target
enrichment (decoded), three basins, mean over three seeds
(${\pm}$std band for soft curves).  Solid: soft retention; dashed:
hard retention.  The dotted line marks $3\times$.}
\label{fig:expX-pareto}
\end{figure}

\begin{figure}[h]
\centering
\includegraphics[width=0.92\linewidth]{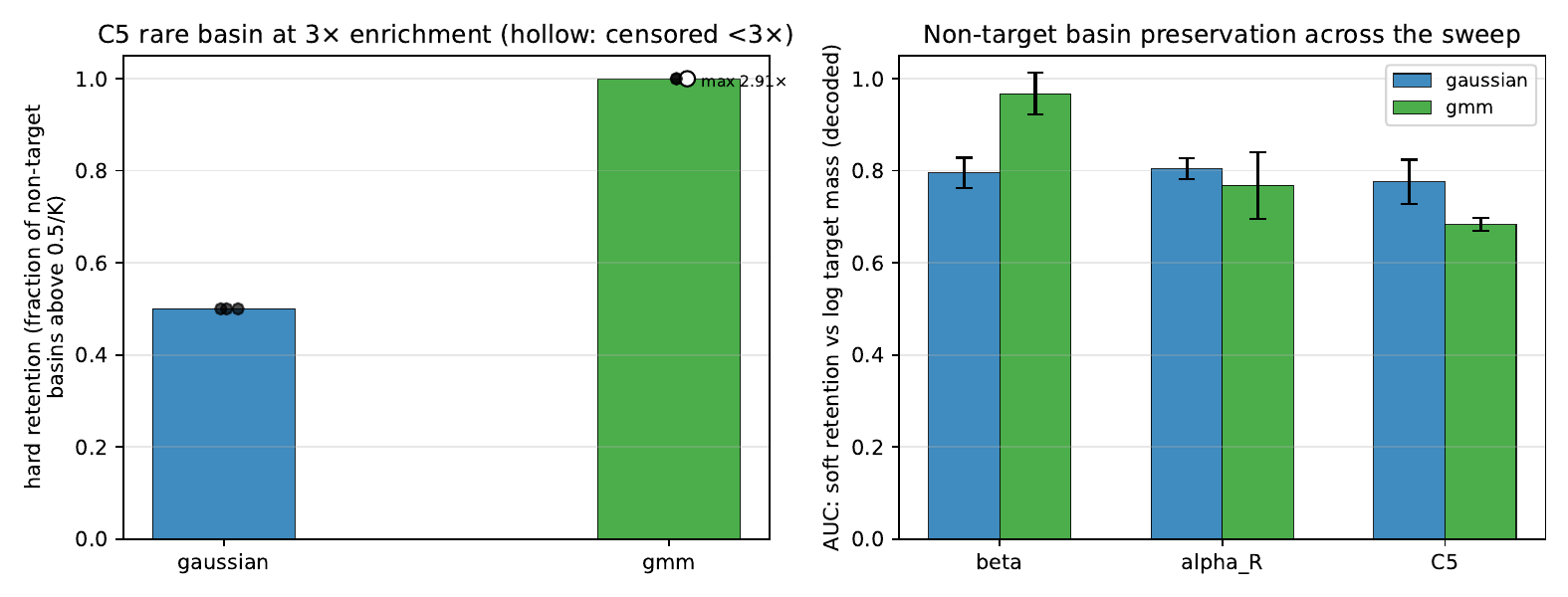}
\caption{Experiment~X, C5 rare basin.  Left: hard retention at
$3\times$ enrichment (hollow marker: seed censored below $3\times$,
plotted at its $2.91\times$ maximum, where retention is $1.0$).
Right: soft retention AUC across the sweep for all three basins.}
\label{fig:expX-retention3}
\end{figure}

\section{Reverse counterfactual: matched-geometry prior switch}
\label{app:expY}

Appendix~\ref{app:expR} fills the \emph{forward} mediation cell: prior
label fixed at Gaussian, geometry moved by the KL weight, and $\beta_c$
follows geometry.  Implication~2's full-mediation claim is symmetric:
\emph{at matched geometry, switching the prior label must not move
$\beta_c$}.  This appendix supplies the reverse cell on the same Exp6
pack ($K{=}10$, $d{=}5$, $\sigma{=}0.5$, $\mathbb R^8\!\to\!2$D):
Gaussian and GMM priors are each trained to the SAME achieved
latent geometries, after which the standard score/ODE/fine-tune
pipeline and the cell's own-geometry variational boundary are applied.
Script: \texttt{expY\_reverse.py}; analysis:
\texttt{analyze\_expY.py}.

\paragraph{Why the reverse cell needs a new intervention.}
One might expect the GMM analogue of expR --- simply down-weighting the
GMM KL term --- to widen the encoder blobs and reproduce compressed
geometries.  It does not.  The label-tracked GMM component means and
variances EMA-track the encoder's own per-label batch statistics
($0.99$ decay), so the GMM KL is self-referential: it contains neither
the centre contraction toward the origin nor the width inflation toward
unit variance that $\mathrm{KL}(q\|\mathcal N(0,I))$ imposes.  Probe
runs (30\,000 VAE steps, seed~0) with the GMM KL weight spanning
$w_{\rm kl}\in\{0.003,0.01,0.03,0.1,0.3,1,3\}$ achieve
$g=d_{\min}/(2\sigma_{\rm lat})=3.1$--$6.8$ at every weight --- never
entering the Gaussian-ELBO image, whose widest geometry on this pack is
$g=2.08$ (the expR $w_{\rm kl}{=}0.03$ cell).  Reweighting therefore
cannot produce the reverse counterfactual.

\paragraph{Two prior-internal interventions that do move GMM geometry.}
We instead modify the GMM prior $p(z)=\sum_k\pi_k
\mathcal N(\alpha\,\mu_k,\mathrm{diag}(\max(s_k^2,\sigma_{\rm
floor}^2)))$ in two ways: a \textbf{width floor}
$\sigma_{\rm floor}$ on every component variance (the reverse KL then
pulls encoder blobs to that width), and a \textbf{centre contraction}
$\alpha<1$ moving the tracked component centres toward the origin.
Both parameters act only on the prior distribution; the decoder,
score model, and reward are untouched.  Calibration at the full
$30\,000$-step budget (seed~0) gives:

\begin{table}[h]
\centering\small
\begin{tabular}{cccccc}
\hline
arm & $\alpha$ & $\sigma_{\rm floor}$ &
$d_{\min}/(2\sigma_{\rm lat})$ & $\sigma_{\rm lat}$ &
matched Gaussian cell \\
\hline
width & 1.0 & 0.35 & 2.15 & 0.48 & --- (above image edge $2.08$)\\
width & 1.0 & 0.50 & 2.02 & 0.67 & $w_{\rm kl}{=}0.03$ ($g{=}2.08$)\\
width & 1.0 & 0.70 & 1.93 & 0.93 & $w_{\rm kl}{=}0.03$\\
width & 1.0 & 1.00 & 1.88 & 1.32 & $w_{\rm kl}{=}0.03$\\
width & 1.0 & 1.30 & 1.79 & 1.69 & $w_{\rm kl}{=}0.03$\\
contract & 0.9 & 0.50 & 1.49 & 0.40 & $w_{\rm kl}{=}0.1$ ($g{=}1.55$)\\
contract & 0.7 & 0.50 & 1.22 & 0.26 & $w_{\rm kl}{=}0.3$ ($g{=}1.14$)\\
contract & 0.5 & 0.50 & 1.11 & 0.21 & $w_{\rm kl}{=}1$ ($g{=}1.12$)\\
\hline
\end{tabular}
\caption{Experiment~Y intervention calibration (seed~0, full VAE
budget).  The width floor moves achieved geometry through the upper
part of the Gaussian-ELBO image; adding centre contraction reaches the
middle image geometries.  Centre separation grows as components widen,
so the width-only arm saturates near $g\approx1.8$.}
\label{tab:expY:calib}
\end{table}

\paragraph{Protocol.}
Each of the eight GMM specifications runs at three seeds; the Gaussian
arm runs at seeds $\{1,2\}$, with seed~0 supplied by the bit-identical
KL-weight arm of expR (same VAE code, seeds, and budget).  Every
frozen VAE follows the main-text pipeline: score model on the
aggregate posterior ($40\,000$ steps), 14-point $\beta$ grid on
$[0.03,30]\,\beta_c^{\rm var}$ of the cell's own latent centres,
$1000$ reverse-KL fine-tune steps per point, $8000$ RK2-150 endpoints.

\begin{figure}[t]
\centering
\IfFileExists{results/expY/fig_expY_mediation.pdf}{%
  \includegraphics[width=0.99\linewidth]{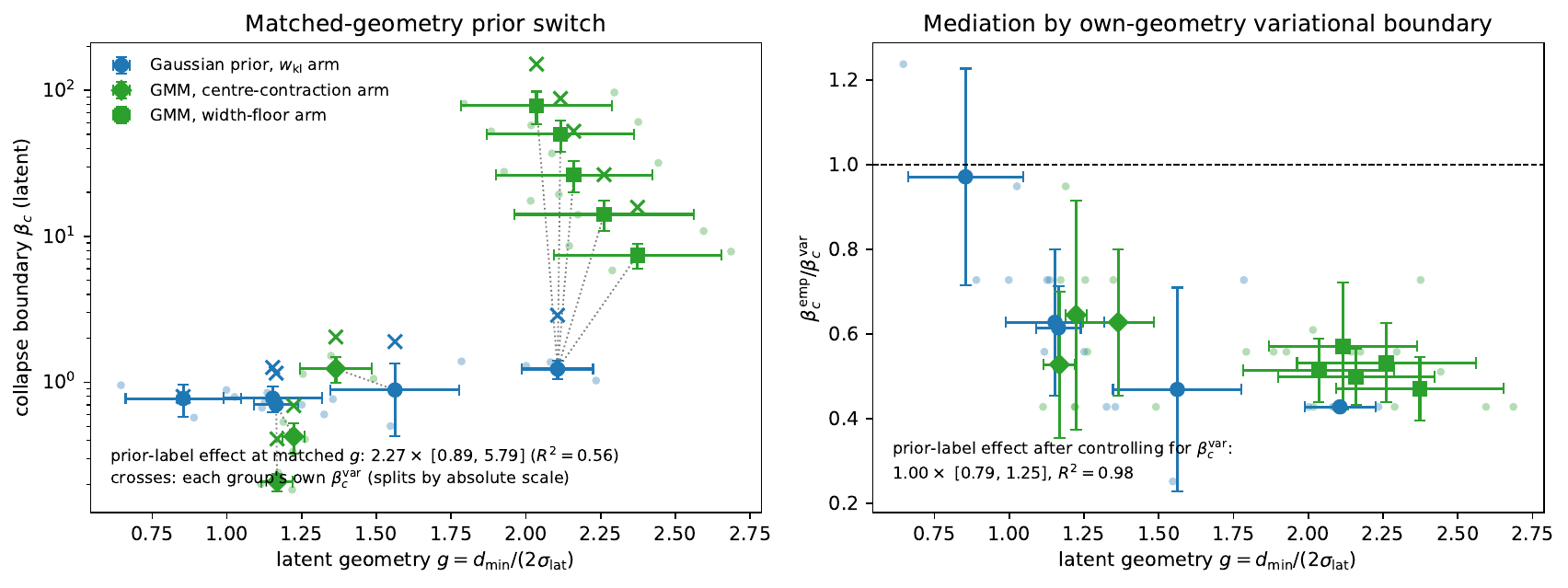}%
}{%
  \fbox{\parbox[c][5cm][c]{0.9\linewidth}{\centering
  Experiment~Y mediation figure (generated when the sweep completes)}}%
}
\caption{Experiment~Y reverse mediation.  \emph{Left:} measured
$\beta_c$ versus achieved geometry $g$ for the Gaussian KL-weight arm
(circles) and the GMM width-floor (squares) and centre-contraction
(diamonds) arms; crosses mark each group's own-geometry
$\beta_c^{\rm var}$; dotted lines join nearest-geometry cross-prior
pairs.  \emph{Right:} the same data as emp/var ratios.  The inset
reports the prior-label coefficient of
$\log\beta_c^{\rm emp}\sim\log g+\mathbb 1_{\rm GMM}$ (a direct prior
effect would move the coefficient away from $1\times$).}
\label{fig:expY:mediation}
\end{figure}

\paragraph{Results.}
All $39$ cell--seed runs ($15$ Gaussian rows --- the five
$w_{\rm kl}$ cells at three seeds, with seed~0 supplied by the
bit-identical expR arm --- and $24$ GMM rows, eight specifications at
three seeds) yield exact $\beta_c$: no point is censored in the matched
window.  The Gaussian arm spans $g=0.85$--$2.11$ and the GMM arms span
$g=1.17$--$2.37$ (centre-contraction $1.17$--$1.36$; width-floor
$2.04$--$2.37$), so the overlap interior $g\in[1.1,2.1]$ is covered by
both priors.  Crucially, the two priors' empirical-to-variational
shortfalls occupy the same band: $\beta_c^{\rm emp}/\beta_c^{\rm
var}=0.43$--$0.97$ over the five Gaussian cells and $0.47$--$0.64$ over
the eight GMM cells --- switching the prior label at a fixed achieved
geometry neither removes nor adds the finite-resource pipeline gap.

The eight nearest-geometry cross-prior pairs ($|\Delta g|\le0.30$) make
this point cell by cell: the change in measured boundary tracks the
change in each pair's own-geometry variational boundary, with
$\Delta\log\beta_c^{\rm emp}-\Delta\log\beta_c^{\rm var}\in[-0.15,
+0.29]$ and no consistent sign (3 negative, 5 positive).  At
$g\approx1.16$--$1.22$, for example, the contracted GMM cells show
$\Delta\log\beta_c^{\rm emp}=-1.22$ and $-0.50$ against
$\Delta\log\beta_c^{\rm var}=-1.04$ and $-0.51$; in the width arm,
where $g$ is matched ($2.04$--$2.37$ vs.\ $2.11$) but the GMM centres
are much farther apart ($d_{\min}=2.3$--$7.0$ vs.\ $0.96$), the
boundaries differ by $1.8$--$4.2$ log units \emph{and so do their
variational predictions} ($1.7$--$4.0$ log units) --- the difference is
explained by geometry in its complete form, not by the prior label.

The direct-effect regressions over all $39$ exact rows
(Figure~\ref{fig:expY:mediation}) give the formal verdict.  Controlling
for the full own-geometry boundary,
$\log\beta_c^{\rm emp}\sim\log\beta_c^{\rm var}+\mathbb 1_{\rm GMM}$,
the GMM coefficient is $-0.005\pm0.115$ ($t=-0.04$): factor
$0.995\times$, $95\%$ CI $[0.79,1.25]$, $R^2=0.976$.  The prior label
has no detectable direct effect, and the CI rules out effects larger
than $\pm25\%$.  Controlling for the scalar ratio alone,
$\log\beta_c^{\rm emp}\sim\log g+\mathbb 1_{\rm GMM}$, gives a
nominally larger but non-significant factor $2.27\times$
($t=1.71$, $95\%$ CI $[0.89,5.79]$, $R^2=0.57$): the scalar $g$ discards
the absolute inter-centre scale $d_{\min}$, which the width arm shows
varies sevenfold at fixed $g$, so it is an incomplete mediator.  The
correct mediator is the geometry as encoded in each cell's own
$\beta_c^{\rm var}$, and with that covariate the mediation is full.

\paragraph{Verdict for Implication~2 (reverse cell).}
Within the reachable interior $g\in[1.1,2.1]$, geometry is a sufficient
statistic for the collapse boundary across the prior switch: at matched
geometry (with the absolute scale retained through $\beta_c^{\rm var}$),
the measured boundaries agree to within the $\pm25\%$ CI, and any
residual is the same finite-pipeline shortfall observed under both
labels.  The result does not extend to the two unreachable endpoints
discussed next, where only one prior produces the geometry and no
cross-label comparison exists.

\paragraph{Bidirectional reachability limit.}
The matched comparison covers the interior $g\in[1.1,2.4]$ of the
Gaussian-ELBO image: all five Gaussian KL-weight cells
($g=0.85$--$2.11$ over three seeds) and all eight GMM specifications
($g=1.17$--$2.37$), producing the eight nearest-geometry pairs with
$|\Delta g|\le0.30$ analysed above (final three-seed geometries sit
slightly above the seed-0 calibration probes of
Table~\ref{tab:expY:calib}).  Both ends of the range remain
single-prior observations and are reported as unreachable rather than
synthesised:
(i)~the tight Exp6/V GMM endpoint $g=3.34$--$5.2$ is not attained by any
Gaussian-prior ELBO on this pack (image edge $g\approx2.11$), and
(ii)~the compressed Gaussian tail ($g=0.85$ three-seed mean at
$w_{\rm kl}{=}3$, reaching $g\approx0.65$ in the expR seed-0 cell) is
not attained by any GMM intervention probed: weight reweighting stays
at $g\ge3.1$, width flooring bottoms out at $g\approx2.04$
($\sigma_{\rm floor}=1.3$) because centre separation grows with
component width, and centre contraction bottoms out at $g\approx1.17$
($\alpha=0.5$) as centres and widths contract together.  A strict
reverse counterfactual at the Exp6 GMM endpoint is therefore impossible
in both directions; the symmetry test is performed throughout the
interior of the reachable image.

\section{Machine-checked skeleton of the theory (Lean\,4 / Mathlib)}
\label{app:lean}

The algebraic and structural skeleton of Theorem~\ref{thm:width-limits}
(the crossover statement), Proposition~\ref{prop:prior-scale}, the
mediation argument of Implication~2, and the three-source decomposition
discussed as Open Problem~2 is formalised in Lean\,4 (toolchain
\,\texttt{leanprover/lean4:v4.3.0}, Mathlib \texttt{v4.3.0}) and
compiles with zero errors under \texttt{lake build}.  The formalisation
takes the modelling assumptions (equal product weights, isotropic
Gaussian components, differentiable score paths) as hypotheses; what is
machine-checked is every algebraic manipulation and logical implication
below.  Source: \texttt{formal/BoundaryFormal/}.

\begin{itemize}\itemsep2pt
\item \texttt{PriorGeometry.lean} proves the two prior-form geometry
costs used in Propositions~\ref{prop:gauss-tax}
and~\ref{prop:assign}: the exact Gaussian-prior KL
$\tfrac12(\mu^2+s^2-1-\log s^2)$ with its minimum at width~1, strict
tax below width~1, and divergence as the width vanishes; and the
mixture assignment bound
$\mathrm{KL}(q\Vert m)\le\mathrm{KL}(q\Vert p)-\log\pi$
(discrete pmfs, via $\log x\le x-1$), whose corollary is that a
posterior matching a component exactly pays only $-\log\pi$
independent of width.
\item \texttt{CollapseBoundary.lean} defines
$\sigma_\beta^2=\sigma_0^2\tau^2/(\sigma_0^2+\tau^2)$ with
$\tau^2=\beta/(2\lambda)$ and proves the two asymptotic identities
$\sigma_\beta^2=\sigma_0^2-\sigma_0^4/(\sigma_0^2+\tau^2)$ and
$\sigma_\beta^2=\tau^2-\tau^4/(\sigma_0^2+\tau^2)$, the strict bounds
$\sigma_\beta^2<\sigma_0^2$ and $\sigma_\beta^2<\tau^2$, monotone
increase with $\tau^2$, and the crossover used in
Theorem~\ref{thm:width-limits}:
$\beta=2\lambda\sigma_0^2\Rightarrow\sigma_\beta^2=\sigma_0^2/2$.
\item \texttt{Mediation.lean} defines total, direct, and indirect
effects and proves their additive identity, that full mediation
($B(p,g)$ independent of $p$) makes the direct effect identically zero,
and the non-identifiability construction behind Implication~2:
two explicit response surfaces that
agree at both observed (prior, geometry) cells yet differ at the
counterfactual cell, so the direct prior effect cannot be read from the
confounded design without the intervention of Experiment~R.
\item \texttt{ErrorDecomposition.lean} proves additivity
$e_{\rm tot}=e_{\rm net}+e_{\rm ode}+e_{\rm anchor}$ of the three
non-negative error sources, the pigeonhole statement that at least one
source is $\ge e_{\rm tot}/3$, and that removing the solver term
($e_{\rm ode}=0$) bounds the achievable boundary by the remaining
network-plus-anchor budget.
\item \texttt{BoundaryConvergence.lean} proves the certification
lemma behind $\beta_c(N)\to\beta_c^\infty$: a measured coverage bracket
computed with uniform score error $\varepsilon$ brackets the true
half-width up to $\pm\varepsilon$; refining the solver to
$\varepsilon'\le\varepsilon$ tightens the certificate; and if the
solver error vanishes ($\varepsilon_N\to0$), there is a finite step
count after which the measured $1/2$-coverage bracket is certified for
the true boundary.
\end{itemize}

\noindent The scope is deliberately narrow: the certificates concern
the deterministic inequalities the paper relies on; no probabilistic
claim (sampling convergence of the Monte-Carlo estimators, score-model
generalisation) is machine-checked.

\clearpage
\raggedbottom
\bibliographystyle{plainnat}
\bibliography{paper}

\end{document}